%% file: main.tex
\documentclass{article}

\PassOptionsToPackage{sort, numbers, compress}{natbib}

\usepackage[preprint]{neurips_2026}

\usepackage[utf8]{inputenc} %
\usepackage[T1]{fontenc}    %
\usepackage{hyperref}       %
\usepackage{url}            %
\usepackage{booktabs}       %
\usepackage{amsfonts}       %
\usepackage{nicefrac}       %
\usepackage{microtype}      %
\usepackage{xcolor}         %

\usepackage{algorithm}
\usepackage{algpseudocode}
\usepackage{color}
\usepackage{amsmath} 
\usepackage{amsthm}
\usepackage{amssymb}
\usepackage{mathtools}
\usepackage{comment}
\usepackage{caption}
\usepackage{subcaption}

\usepackage{cleveref}

\usepackage[all]{foreign}   %
\usepackage{pgfplots}
\pgfplotsset{compat=1.18}
\usepackage{placeins}
\usepackage{longtable}
\usepackage{chngcntr}

\usepackage{tikz}
\usetikzlibrary{patterns,arrows,arrows.meta,calc,shapes,shadows,decorations.pathmorphing,decorations.pathreplacing,automata,shapes.multipart,positioning,shapes.geometric,fit,circuits,trees,shapes.gates.logic.US,fit,backgrounds,decorations.text,topaths}

\newcommand{\distr}[1]{\Delta(#1)}

\DeclareMathOperator*{\argmax}{arg\,max}

\newcommand{\last}[1]{\text{last}\left(#1\right)}

\newcommand{\infnorm}[1]{\left\lVert#1\right\rVert_{\infty}}

\newtheorem{theorem}{Theorem}
\newtheorem{lemma}{Lemma}
\newtheorem{definition}{Definition}

\newtheorem{remark}{Remark}
\newtheorem{example}{Example}

\input{include/macros}

\input{include/colors}
\input{include/crefnames}

\title{
Robust Shielding for Safe Reinforcement Learning
}

\author{%
Edwin Hamel-De le Court$^{1}$ \quad
Thom Badings$^{2,3}$ \quad
Alessandro Abate$^{3}$ \quad
\\
\textbf{Francesco Belardinelli}$^{4}$ \quad
\textbf{Francesco Fabiano}$^{3}$ 
\\ \, \\
$^1$ Department of Computer Science, University of Manchester, United Kingdom \\
\texttt{edwin.hamel-delecourt@manchester.ac.uk} \\
$^2$ Faculty of Computer Science \& DSME, RWTH Aachen University, Germany \\
$^3$ Department of Computer Science, University of Oxford, United Kingdom \\
$^4$ Department of Computing, Imperial College London, United Kingdom
}

\begin{document}

\maketitle

\begin{abstract}
Shielding is an effective approach to formally guarantee the safety of reinforcement learning agents in Markov decision processes (MDPs). However, existing shielding techniques typically assume knowledge of the safety-relevant transition dynamics -- a requirement that is seldom met in practice. To address this limitation, we introduce a novel shielding framework for robust MDPs (RMDPs), \ie MDPs with sets of transition probabilities. We define safety as the satisfaction of a linear temporal logic (LTL) formula with a certain threshold probability under the worst-case transition probabilities of the RMDP. We prove that our shielding framework is both sound and optimal for the RMDP: every policy admissible by the shield is safe, and conversely, every safe RMDP policy is admissible by the shield. We combine our approach with existing sampling methods for learning transition probabilities of MDPs with probably approximately correct (PAC) guarantees. This combination enables the construction of shields for MDPs that, with high confidence, guarantee safety while remaining minimally restrictive. Our experiments show that our shields for learned RMDPs guarantee safety in unknown MDPs while recovering strong expected return as the number of samples increases.
\end{abstract}

\section{Introduction}
\label{sec:introduction}
Reinforcement learning (RL)~\citep{10.5555/3312046} optimizes the behavior of an agent taking actions in an unknown environment, generally modeled as a \emph{Markov decision process}~(MDP).
While success stories of RL have reached into key areas, such as robotics~\citep{SurveyRobotics}, game playing~\citep{DBLP:journals/corr/MnihKSGAWR13}, and autonomous driving~\citep{DBLP:journals/tits/KiranSTMSYP22}, \emph{safety} in RL remains a major concern.
In particular, state-of-the-art RL algorithms let  agents explore actions unrestrictedly and will thus inevitably select potentially harmful actions~\citep{DBLP:journals/jmlr/GarciaF15}.

This \emph{unsafe exploration} problem has triggered research on \emph{shielded RL}~\citep{ABENTShielding}.
A shield is a mechanism that prevents (\ie \emph{``shields''}) RL agents from executing unsafe actions.
Shields have been developed for MDPs~\citep{DBLP:conf/aaai/CourtBG25,DBLP:conf/concur/0001KJSB20} as well as partially observable MDPs (POMDPs)~\citep{DBLP:conf/aaai/Carr0JT23}.
To guarantee safety, these approaches require knowing the MDP's transition dynamics relevant to safety. 
However, in typical RL problems, these dynamics are unknown.
Instead, our starting assumption is to have access to some prior data on the dynamics, whether from previous executions or a simulator.
One solution is to use this data to learn \emph{point estimates} of the dynamics, but doing so introduces statistical errors~\citep{DBLP:journals/mansci/MannorSST07} that compromise the shield's safety guarantees.
An alternative is to learn an \emph{uncertainty set} around each unknown transition probability that is \emph{probably approximately correct}~(PAC), \ie contains the true probability with high confidence~\citep{DBLP:journals/ior/NilimG05,DBLP:conf/birthday/SuilenBB0025}.
Unfortunately, existing shielding approaches are limited to models with exact probabilities and thus incompatible with such uncertainty sets and PAC~bounds.

\begin{figure}[t!]
    \centering
    \scalebox{0.85}{%
      \input{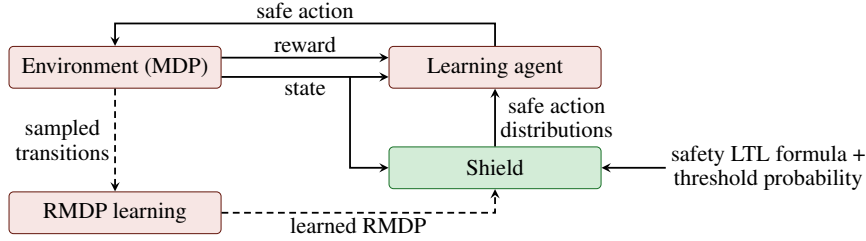}
    }
    \caption{Our shielding framework. Offline (shown by dashed arrows), we use sampled transitions on the MDP to learn the safety-relevant transition probabilities in the form of an RMDP with PAC guarantees. We use the learned RMDP to compute a shield that acts on the (unknown) MDP, by constraining the learning agent to the subset of stochastic policies that guarantee safety.}
    \label{fig:overview}
\end{figure}

\paragraph{Contributions: shielding robust MDPs.}
We address this gap by developing a novel shielding framework for MDPs with sets of transition probabilities, known as \emph{robust MDPs} (RMDPs)~\citep{DBLP:journals/mor/WiesemannKR13}.
An RMDP can be seen as a game between the \emph{agent}, which chooses the actions, and an \emph{adversary}, which chooses transition probabilities from the uncertainty sets.
Intuitively, our shield prevents the agent from deploying \emph{any} policy that leads to unsafe outcomes under \emph{any} choice by the adversary.
As is common in shielding~\cite{ABENTShielding,DBLP:conf/aaai/CourtBG25}, we adopt \emph{linear temporal logic} (LTL) as a specification language, focusing on the \emph{safety} fragment for algorithmic purposes~\citep{DBLP:journals/fmsd/KupfermanV01}.
When used to constrain the agent's policy, our shield guarantees that, regardless of the choices by the adversary, the probability of satisfying a given LTL formula meets a prescribed threshold.
As depicted in \cref{fig:overview}, we can instantiate our framework to shield unknown MDPs in two steps:
\begin{enumerate}
    \item We use sampled transitions in the environment to learn the transition probabilities of the (unknown) MDP as an RMDP with PAC guarantees, \ie the RMDP \emph{``contains''} the true MDP with a prescribed (high) confidence probability $\delta$.
    As common in shielding~\citep{ABENTShielding}, we do not need to learn \emph{all} probabilities, but instead only those relevant to the safety specification.
    \item Based on the learned RMDP, we construct a shield.
    The shield restricts the agent's choices to a subset of probability distributions over actions.
    When used to constrain the actions of the agent on the true MDP, the shield guarantees that, with the same confidence probability~$\delta$, the safety LTL formula is satisfied with at least the desired threshold probability.
\end{enumerate}
As a key contribution, we prove, under mild assumptions detailed in \cref{thm:RMDP_soundness,thm:RMDP_optimality}, that our shield is both \emph{sound and optimal} on the RMDP: every policy admissible under the shield is safe and, conversely, every safe RMDP policy is admissible under the shield. Since the RMDP converges to the true MDP as the number of sampled transitions increases, this optimality carries over to optimality for the MDP in the limit.
Thus, our shield guarantees safety on the true MDP at any point with high confidence -- both during learning and execution -- and is \emph{minimally restrictive} to the learning agent, enabling it to converge to optimal rewards as the quality of the learned RMDP increases.
We demonstrate these results empirically and, by contrast, show that state-of-the-art shielding for learned MDPs with point estimates of probabilities can violate safety.

\paragraph{Related work.}
Safety in RL is an active area of research~\cite{DBLP:journals/jmlr/GarciaF15,DBLP:journals/arcras/BrunkeGHYZPS22,DBLP:journals/pami/GuYDCWWK24}.
Shielding has proved to be a successful approach to block unsafe actions during training and execution~\cite {ABENTShielding,DBLP:conf/concur/0001KJSB20,DBLP:conf/aaai/Carr0JT23,DBLP:conf/sii/OdriozolaOlaldeZA23,DBLP:conf/ijcai/YangMRR23,DBLP:conf/atal/Elsayed-AlyBAET21,DBLP:journals/isse/KonighoferRPTB23,DBLP:journals/cacm/KonighoferBJJP25}.
Closest to our approach is the %
probabilistic shielding for MDPs proposed in~\cite{DBLP:conf/aaai/CourtBG25}, which, however, assumes knowledge of the underlying MDP.
Here, we consider the more challenging case of RMDPs and thereby shielding unknown MDPs learned as an RMDP.
Existing shields on learned dynamics generally lack finite-sample guarantees by relying on, \eg Monte Carlo sampling~\cite{DBLP:conf/atal/GoodallB24} or learned latent representations~\cite{DBLP:conf/ecai/BethellGCI25}.
To our knowledge, the only existing RMDP shielding approach is~\cite{DBLP:conf/amcc/ReedL25}, which considers RMDPs with intervals of probabilities (\emph{interval MDPs}) and iteratively expands the allowed actions per state, yielding sound but overly conservative shields as a \emph{permissive policy}~\cite{huynh2026permissive}.

Lagrangian approaches~\citep{DBLP:journals/jmlr/ChowGJP17,DBLP:conf/icml/StookeAA20} convert safety-constrained MDPs into unconstrained problems with dual variables to penalize cost violations in the reward objective.
Recent works extend these ideas to robust constrained MDPs~\cite{Ganguly2025,DBLP:journals/corr/abs-2010-04870}.
While often empirically effective, these methods cannot provably guarantee safety at any time (as opposed to shields), thus limiting their use in safety-critical scenarios.

Lyapunov-based methods enforce safety by satisfying a Lyapunov decrease condition defined w.r.t. a safe baseline policy~\cite{DBLP:conf/nips/ChowNDG18,DBLP:journals/corr/abs-1901-10031}, where guarantees depend on accurate safety critics and linearizations.
In~\cite{DBLP:conf/nips/BerkenkampTS017}, Gaussian processes (GP) are used for model-based RL, but guarantees rely on the GP prior and a known Lyapunov function, which is hard to construct in practice.
Finally, Dyna-style model-based RL uses rollouts on learned models to generate data for model-free agents~\cite{DBLP:journals/sigart/Sutton91,DBLP:conf/nips/JannerFZL19,DBLP:conf/icml/FrauenknechtESS24}, which reduces required system interactions but cannot give formal safety guarantees as shields do.

\paragraph{Overview.}
After the preliminaries in \cref{sec:preliminaries}, we present our shields for RMDPs in \cref{sec:shields_RMDPs}.
We discuss the setting of shielding unknown MDPs in \cref{sec:unknown_MDPs} and empirically evaluate our approach in \cref{sec:experiments}.
Throughout the paper, we focus on intuition for our results and instead present all rigorous details about, \eg measure-theoretic constructions, and proofs for our results in the appendix.

\section{Preliminaries}
\label{sec:preliminaries}

A \textit{Markov decision process} (MDP) is a tuple $\mathcal
M=\tuple{S,A,P,\sinit,AP,L}$, where 
    $S$ is an (in)finite set of \textit{states} with \textit{initial state} $\sinit\in S$; 
    $A$ is an (in)finite set of \emph{actions};\footnote{For notational simplicity, we assume every action is enabled in every state (see App.~\ref{app:preliminaries} for the case without this restriction).} 
    $P \colon S\times A \to \distr{\States}$ is a transition function;  
    $AP$ is a finite set of \emph{atomic propositions}; 
    and $L \colon S\to2^{AP}$ is the \emph{labeling function}. 
We say that $\mdp$ is an MDP \emph{with rewards} if it is additionally equipped with a \emph{reward function} $R \colon S\to \mathbb R$. 
For simplicity, we may write $P(s,a,s')$ instead of $P(s,a)(\{s'\})$.
A \emph{finite history} in $\mdp$ of length $n$ is a finite word $h=s_0 a_0 \cdots s_{n-1}a_{n-1}s_n \in S( A S)^n$ for some $n\in\mathbb N$, and we denote its last state $s_n$ by $\last{h}$. 
An \emph{infinite history} $h=s_0 a_0\cdots s_n a_n\cdots$ is defined analogously.
A \emph{policy} $\pi \colon H \to \distr{A}$ of \(\mathcal M\) is a mapping from the set $H$  of finite histories to distributions over actions.
Fixing a policy $\pi$ and state $s \in S$ induces a probability measure $\text{prob}^{s}_{\mdp,\pi}$ over histories. 
We write $\text{prob}_{\mdp,\pi}$ for $\text{prob}^{\sinit}_{\mdp,\pi}$.
For details~on policies and induced probability measures, see~\cite{baier2008principles,BSSstochastic}.

\paragraph{Safe linear temporal logic.}
We express objectives for MDPs in \emph{linear temporal logic} (LTL)~\cite{DBLP:conf/focs/Pnueli77}.
Following~\cite{baier2008principles}, an LTL formula $\formula$ over the atomic propositions $AP$ is generated by the \emph{grammar}
\[
\formula \Coloneqq \top \mid p \mid \neg \formula \mid \formula \land \formula
\mid \mathbf{X}\formula \mid \formula \,\mathbf{U}\, \formula,
\qquad p \in AP.
\]
As standard, we define
$\formula_1 \lor \formula_2 \coloneqq \neg(\neg \formula_1 \land \neg \formula_2)$,
$\mathbf{F}\formula \coloneqq \top \,\mathbf{U}\, \formula$,
and $\mathbf{G}\formula \coloneqq \neg \mathbf{F}\neg \formula$.
A \emph{finite (resp. infinite) trace} is a sequence $\tau = b_0b_1\cdots \in \Sigma^*$ (resp. $\Sigma^\omega$), where the \emph{alphabet}~$\Sigma = 2^{AP}$ specifies which propositions hold at each timestep.
Given an (in)finite history $h=s_0 a_0 s_1 a_1 \cdots$, we denote by $L(h)=L(s_0)L(s_1)L(s_2)\cdots$ %
the (in)finite trace induced by the labeling function. 
The satisfaction relation $\tau,i \models \formula$ for a formula $\formula$ is defined inductively over infinite traces $\tau = b_0 b_1\cdots$ of $\mdp$ as %
\begin{alignat*}{2}
\tau,i &\models p &&\iff p \in b_i,\\
\tau,i &\models \neg \formula &&\iff \tau,i \not\models \formula,\\
\tau,i &\models \formula_1 \land \formula_2
&&\iff \tau,i \models \formula_1 \text{ and } \tau,i \models \formula_2,\\
\tau,i &\models \mathbf{X}\formula &&\iff \tau,i+1 \models \formula,\\
\tau,i &\models \formula_1 \,\mathbf{U}\, \formula_2
&&\iff \exists j \geq i \text{ such that } \tau,j \models \formula_2
\text{ and } \tau,k \models \formula_1 \text{ for all } i \leq k < j.
\end{alignat*}
We write $\tau \models \formula$ as shorthand for $\tau,0 \models \formula$. 
An LTL formula $\formula$ is a \emph{safety formula}~\citep{DBLP:journals/fmsd/KupfermanV01} if every violating trace has a finite witness of violation, \ie for every $\tau \in \Sigma^\omega$ such that $\tau \not\models \formula$, there exists a finite prefix $u \prec \tau$ (called a \emph{bad prefix}) such that for every $\rho \in \Sigma^\omega$, $u\rho \not\models \formula$.
For every safety formula $\formula$, there exists a DFA that accepts exactly all bad prefixes of $\formula$~\cite{baier2008principles}:

\begin{definition}[DFA]
    \label{def:DFA}
    A \emph{deterministic finite automaton} (DFA) over the alphabet $\Sigma$ is a tuple $
    \mathcal{A}=\tuple{Q,\Sigma,\qinit,\delta,F}$,
    where $Q$ is a finite set of states with initial state $\qinit \in Q$,
    $\delta \colon Q\times\Sigma\to Q$ is the transition function, and
    $F \subseteq Q$ is the set of accepting states.
\end{definition}

The transition function extends to finite words as standard:
$\delta^*(q,ua)=\delta(\delta^*(q,u),a)$ for $u\in\Sigma^*$, $a\in\Sigma$. 
A finite word $u \in \Sigma^*$ is accepted by DFA $\mathcal{A}$ iff $\delta^*(\qinit,u)\in F$. 
A \emph{specification}~$\spec=\mathbb P_{\geq \lambda}(\formula)$ combines an LTL formula $\formula$ and a threshold $\lambda \in [0,1]$.
For a policy $\pi$, we write $(\mdp,\pi)\models \mathbb P_{\geq \lambda}(\formula)$
iff $
\text{prob}_{\mdp,\pi}
\bigl(
\{h \mid L(h)\models \formula\}
\bigr)
\geq \lambda$,
\ie the probability for generating a satisfying trace is at least~$\lambda$. 

\paragraph{Safe policy optimization.}
Our goal is to compute a policy $\pi$ that (1)~maximizes the $\gamma$-discounted expected return
$
J_{\mdp}^{\gamma}(\pi)
\coloneqq
\mathbb E_{\mdp,\pi}^{\sinit}
\left[
\sum_{t=0}^{\infty} \gamma^t R(s_t)
\right],
$
with the expectation taken w.r.t.~the probability measure $\text{prob}^{\sinit}_{\mdp,\pi}$, and (2)~ensures the probability of violating a safety LTL formula $\formula$ is below a prescribed threshold $p \in [0,1]$.
Formally, this amounts to solving the \emph{safe policy optimization~problem}:
\[
\pi^\star \in \argmax\nolimits_{\pi}\; J_{\mdp}^{\gamma}(\pi)
\quad
\text{subject to}
\quad
(\mdp,\pi)\models \mathbb P_{\geq {1-p}}(\formula).
\]
If the MDP $\mdp$ is known, $\pi^\star$ can be computed using the probabilistic shielding techniques for MDPs from~\cite{DBLP:conf/aaai/CourtBG25}.
Here, however, we assume that the transition function $P$ is \emph{unknown} and we only have access to \emph{samples} from $P$, rendering existing techniques inapplicable.
Instead, we develop a shielding framework for RMDPs (\cref{sec:shields_RMDPs}) and use it to shield a learned RMDP of the unknown MDP (\cref{sec:unknown_MDPs}).

\begin{definition}[RMDP]
\label{def:RMDP}
A \emph{robust MDP} (RMDP) is a tuple $\rmdp = \RMDP$, where $\States$, $\Actions$, $\sinit$, $\labels$, and $\labelfunc$ are defined as in an MDP, and $\cP \colon S\times A \to 2^{\distr{\States}}$ is an \emph{uncertain transition function} that maps every state-action pair to a set of probability measures over $S$.
\end{definition}

An RMDP \emph{with rewards} is additionally equipped with a reward function $R \colon S\to \mathbb R$. 
An RMDP defines a game between the agent that chooses actions via the policy $\pi$, and an \emph{adversary}~$\nature$ that chooses probability distributions in $\cP$.
Formally, the adversary $\nature$ is a mapping that associates every finite history $h \in H$ to a distribution in $\cP(\last{h}, a)$ for every action $a \in A$. 
Fixing an adversary $\nature$ for $\rmdp$ induces a standard MDP~\cite{DBLP:conf/birthday/SuilenBB0025}, denoted by $\rmdp[\nature]$.
We say that $\rmdp$ satisfies the specification $\mathbb P_{\geq 1-p}(\formula)$, denoted as $(\rmdp,\pi)\models \mathbb P_{\geq 1-p}(\varphi)$, iff $(\rmdp[\nature],\pi)\models \mathbb P_{\geq 1-p}(\varphi)$ for every adversary $\nature$.

\section{Probabilistic Shields for Robust MDPs}
\label{sec:shields_RMDPs}

In this section, we present our shielding approach for RMDPs.
We first define our general notion of shield in \cref{sec:shield_general}, instantiate it to finite RMDPs in \cref{sec:shield_finite}, and show its correctness in~\cref{sec:shield_correctness}.

\subsection{Shields for RMDPs}
\label{sec:shield_general}
A shield may need internal memory to remember enough about the history to determine which actions can still be taken while satisfying the specification.
We encode this memory in a \emph{controllable monitor}, which extends standard monitors (used in, \eg runtime verification~\cite{DBLP:conf/rv/HavelundR18,DBLP:series/lncs/BartocciDDFMNS18}) with auxiliary actions.

\begin{definition}[Controllable monitor]
    \label{def:monitor}
    A \emph{controllable monitor} over an \emph{observation space} \(\Xi\)
    is a tuple $\mathcal D=\tuple{M, U,\minit,\zeta}$,
    where \(M\) is a space of \emph{monitor states},
    \(U\) is a space of \emph{auxiliary actions},
    \(\minit\in M\) is the \emph{initial state}, and $\zeta \colon M\times \Xi\times U\to M$ is a \emph{transition function}.
\end{definition}

We will use the auxiliary action to allocate a remaining \emph{violation budget}, which records how much probability of violating the specification is still allowed over the remaining execution.

\begin{example}
    \label{ex:monitor}
    A DFA (without its accepting states) can be seen as a controllable monitor with singleton auxiliary action $U=\{\ast\}$ and observation space $\Xi = \Sigma$ as the alphabet of the DFA. An NFA, on the other hand, additionally requires auxiliary actions that resolve nondeterminism at each step.
\end{example}

To combine an RMDP with a controllable monitor, we construct a product model based on an \emph{observation map} $\alpha \colon S \to \Xi$ from the RMDP states $S$ into the observation space $\Xi$ of the monitor.
The product state $(s,m)$ contains both the RMDP and monitor state, and the product action $(a,u)$ consists of both an RMDP action $a \in A$ and an auxiliary action $u \in U$.

\begin{definition}[Product RMDP]
    \label{def:product-rmdp-monitor}
    Let $\rmdp=\tuple{S,A,\mathcal P,\sinit,AP,\labelfunc}$
    be an RMDP, let $\alpha \colon S \to \Xi$ be an observation map, and let $\mathcal D=\tuple{M, U,\minit,\zeta}$
    be a controllable monitor over $\Xi$.
    The \emph{product RMDP} of $\rmdp$ and $\mathcal D$ (relative to
    $\alpha$) is the RMDP $\rmdp\otimes_{\alpha} \mathcal D = \tuple{\overline S,\overline A,\overline{\mathcal P},\overline{\sinit},
    AP,\overline{\labelfunc}}$, where $\overline S=S\times M$, $\overline{A}=A\times U$, $\overline{\sinit}=(\sinit,\minit)$, and for every $(s,m)\in \overline S$, 
    \begin{enumerate}
        \item $\overline{\labelfunc}(s,m)=\labelfunc(s)$, and
        \item for every action $(a,u)\in\overline A((s,m))$,
        letting $
        T(s')
        =
        \bigl(s',\zeta(m,\alpha(s'),u)\bigr)
        $, the uncertainty set
        $\overline{\mathcal P}((s,m),(a,u))$
        consists of all pushforward measures
        $\mu\circ T^{-1}$ with $\mu\in\mathcal P(s,a)$.
    \end{enumerate}
\end{definition}

Intuitively, choosing action $(a,u)$ in state $(s,m)$ of the product RMDP means that the next RMDP state $s' \in S$ is given by $\cP(s,a)$, and the next monitor state is $m' = \zeta(m,\alpha(s'),u)$.
\cref{def:product-rmdp-monitor} extends the usual product between an MDP $\mdp$ and a DFA $\mathcal A$, used in model checking~\cite{baier2008principles} and shielding~\cite{ABENTShielding}, with auxiliary actions. 
Indeed, this usual product can be recovered with $\mathcal M\otimes_{\labelfunc} \mathcal A$, where $\labelfunc$ is the labeling function of the MDP, and the DFA is identified with its associated monitor as in \cref{ex:monitor}.

We can now define our notion of a shield.
In a nutshell, a shield specifies, for every observation and monitor state, which randomized choices over (RMDP and auxiliary) actions are allowed.

\begin{definition}[Shield]
\label{def:shield}
Let $\Xi$ be a set of observation states and let $A_\Xi$ be a set of actions. A \emph{shield} over $(\Xi,A_\Xi)$ is a pair $
\mathfrak S=(\mathcal D,\Gamma)$, where $\mathcal D=\tuple{ M,U,\minit,\zeta}$
is a controllable monitor over $\Xi$, and $\Gamma$ is a map from $\Xi\times M$ to  $2^{\distr{A_\Xi \times U}}$ that models the allowed distributions over actions.
\end{definition}

The shield in \cref{def:shield} is defined over an \emph{observation space} $\Xi$, rather than directly over the full state space $S$.
This feature, in particular, allows us to construct shields on a \emph{safety abstraction} of the RMDP~\cite{ABENTShielding}.
Let $\rmdp=\tuple{S,A,\mathcal P,\sinit,AP,\labelfunc}$ be an RMDP, let $\mathfrak S=(\mathcal D,\Gamma)$ be a shield over $(\Xi,A_\Xi)$, and let $\alpha \colon S \to \Xi$ be an observation map.
We say that the shield $\mathfrak S$ \emph{acts on} $(\rmdp,\alpha)$ if $A_{\Xi}=A$.
As visualized in \cref{fig:shield}, at runtime, the shield observes only the abstract state $\xi = \alpha(s)$ and monitor state~$m$, and enforces the agent to select a randomized choice over actions contained in $\Gamma(\xi,m)$.

\begin{figure}[t!]
    \centering
    \scalebox{0.85}{%
      \input{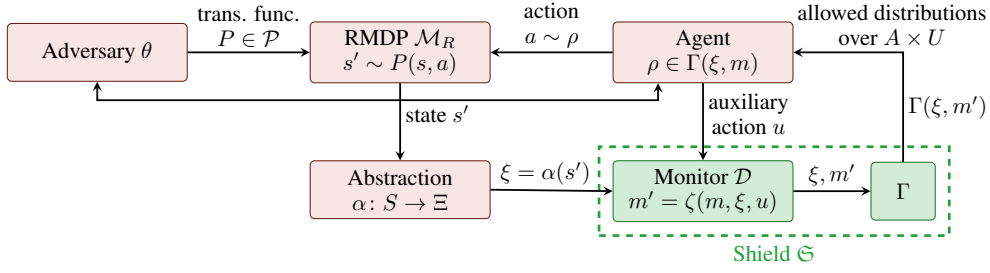}
    }
    \caption{The shield $\mathfrak S = (\mathcal D, \Gamma)$ deployed on an RMDP $\rmdp$ with an abstraction map $\alpha$.}
    \label{fig:shield}
\end{figure}

\subsection{A Shield for Finite RMDPs}
\label{sec:shield_finite}
We now instantiate the general shield from \cref{def:shield} to \emph{finite} RMDPs.
Throughout this section, we assume we are given an RMDP $\rmdp$ with \textbf{finite} sets $S$ and $A$.
Inspired by~\cite{DBLP:conf/aaai/CourtBG25}, our construction uses a \emph{robust upper bound} on the probability of reaching a bad state, formalized as follows.
\begin{definition}
    \label{def:robust_Bellman}
    An $(\rmdp,U)$-\emph{inductive value function} is a function $\beta:S\to [0;1]$ such that
    \(
    \left(\mathcal B_{\rmdp}^{U}(\beta)\right)(s) \leq \beta(s)
    \)
    for every $s \in S$, where \(\mathcal B_{\rmdp}^{U}\) is the robust Bellman operator defined as
    \[
    \left(\mathcal B_{\rmdp}^{U}(\beta)\right)(s)
    =
    \begin{cases}
    1, & \text{ if } s\in U,\\[1mm]
    \displaystyle \inf_{a\in A}\ \sup_{\mu\in\mathcal P(s,a)} \mathbb E_{s'\sim \mu}
    \left[\beta(s')\right],
    & \text{otherwise.}
    \end{cases}
    \]
\end{definition}
In our case, the states $U$ correspond to violating a given safety LTL formula $\varphi$.
This~violating set is defined as $S \times F$ on the product $\rmdp \otimes_L \mathcal A$ between the RMDP $\rmdp$ and DFA $\mathcal A=\tuple{Q,2^{AP},\qinit,\delta,F}$ that recognizes the bad prefixes of $\formula$.
Thus, we use an $(\rmdp\otimes_L\mathcal A, S \times F)$-inductive value function $\beta \colon S \times Q \to [0,1]$.
Intuitively, $\beta(s,q)$ is an upper bound on the probability of violating $\varphi$ in state $(s,q)$ under an optimal policy and worst-case adversary.
To enforce the probability threshold $p \in [0,1]$, we thus require that $\beta(\sinit,\qinit) \leq p$.
Furthermore, we define
\(
    \mathcal V_\beta=\prod\nolimits_{(s,q)\in S\times Q} \left[ \beta(s,q),1 \right].
\)
Each $v \in \mathcal V_\beta$ assigns an admissible \emph{violation budget} $v(s,q) \in [0,1]$ to every product state $(s,q)$.
\begin{definition}[Shield for finite RMDP]
\label{def:shield_finite_rmdp}
Let $\rmdp$, $\mathcal A$, and $\beta$ be defined as above.
The shield $\mathfrak S(\rmdp,\mathcal A,\beta) = (\mathcal D,\Gamma)$ is defined as an instantiation of \cref{def:shield} as follows:
\begin{enumerate}
    \item $\mathcal D = \tuple{M, \mathcal V_\beta,\minit,\zeta}$ is a controllable monitor over $S$ with 
    states $M \coloneqq Q \times [0,1]$ and $\minit \coloneqq (\qinit,p)$,
    auxiliary actions $\mathcal V_\beta$,
    and
    transition function $\zeta \colon M \times S \times \mathcal V_\beta \to M$ defined for all $(q,y) \in M$, $s \in S$, and $v \in \mathcal V_\beta$ as
    \(
    \zeta((q,y),s,v)
    = \bigl(\delta(q,\labelfunc(s)),
    v(s,\delta(q,\labelfunc(s)))\bigr);
    \)
    \item for all states $(s,q,y)$ of $\rmdp\otimes\mathcal D$, the set $\Gamma(s,q,y)$ contains exactly those distributions $\overline\rho \in \distr{A \times \mathcal V_\beta}$ such that
    \(
    \mathbb E_{(a,v)\sim\overline\rho}
    \Big[
    \sup_{\mu\in\mathcal P(s,a)}
    \mathbb E_{s'\sim\mu}
    \left[
    v\bigl(s',\delta(q,\labelfunc(s'))\bigr)
    \right]
    \Big]
    \le y.
    \)
\end{enumerate}
\end{definition}

Intuitively, $\Gamma$ contains those distributions $\overline\rho$ such that the \emph{worst-case expected next budget} remains below the \emph{current budget $y$}.
Since the initial budget is at most $p$, every policy allowed by the shield keeps the worst-case probability of reaching $S \times F$ (and thus of a bad prefix) at most~$p$.

\subsection{Sound and Optimal Shields on Safety Abstractions}
\label{sec:shield_correctness}
We now analyze the guarantees of the shield from \cref{def:shield_finite_rmdp}.
In this section, we return to the setup in \cref{fig:shield}, where $\rmdp$ may be an \emph{infinite} RMDP but the observation space $\Xi$ is \emph{finite}.
Thus, we in fact analyze the shield $\mathfrak S(\rmdp/\alpha,\mathcal A,\beta)$ for the finite safety abstraction $\rmdp/\alpha$, which is defined as the standard quotient (see~\cref{def:quotient_RMDP} in App.~\ref{app:preliminaries}).
Yet, the shield is deployed on the (non-abstracted) RMDP $\rmdp$, so we need to reason about policies on $\rmdp[\overline\nature] \otimes_\alpha \mathcal D$ for different adversaries $\overline\nature$, \ie on the product of the induced MDP $\rmdp[\overline\nature]$ and the monitor $\mathcal D$, relative to the abstraction map $\alpha$.

\begin{definition}[Compliant Policies]
    A policy $\overline\pi$ on $\rmdp \otimes_\alpha \mathcal D$ is \emph{compliant} with the shield $\mathfrak S$ acting on $(\rmdp, \alpha)$ up to time $T$ under the adversary $\overline\nature$, if $\overline\pi$ only outputs distributions consistent with $\Gamma$,~\ie
    \[
    \prob_{\rmdp[\overline\nature]\otimes_\alpha\mathcal D,\overline\pi}
    \left[
    \forall t < T,\;
    \overline\pi(\cdot\mid \overline h_t)
    \in
    \Gamma(\alpha(s_t),m_t)
    \right]
    =1,
    \]
    where $\overline h_t$ denotes the product history up to time $t$, and $(s_t, m_t)$ is its last product state.
\end{definition}

We say that a shield is \emph{realizable} if, after following any compliant policy up to time $t$, there always exists a continuation policy that is compliant under every adversary.

\begin{definition}[Shield realizability]
    \label{def:shield_reliability}
    A shield \(\mathfrak S=(\mathcal D,\Gamma)\) acting on \((\rmdp,\alpha)\) is
    \emph{realizable} over \((\rmdp,\alpha)\) if, for
    all \(t\in\Nat\), adversaries \(\overline\nature\) of
    \(\mathcal M\otimes_\alpha\mathcal D\), and policies
    \(\overline\pi\) compliant with \(\mathfrak S\) up to time~\(t\) under \(\overline\nature\), there exists a policy $\overline\pi'$ that 
        (1)~coincides with $\overline\pi$ on all histories of length up to $t$, and
        (2)~is compliant with $\mathfrak S$ under every adversary $\overline\nature'$ that coincides with $\overline\nature$ on all histories of length up to~$t$.
\end{definition}

As a final ingredient for defining soundness and optimality, we need to project policies $\overline\pi$ on $\rmdp \otimes_\alpha \mathcal D$ onto the original RMDP $\rmdp$.
Intuitively, the projected policy $\overline\pi^\downarrow$ is obtained from $\overline\pi$ by ignoring the monitor component, \ie at each history, it selects actions according to the marginal distribution induced by $\pi$ over product histories consistent with that history (see~\cref{def:projected_policy} for a formal definition).

\begin{definition}[Shield soundness]
    \label{def:shield_soundness}
    Let \(\spec = \mathbb P_{\geq {1-p}}(\formula)\) be a specification, and let \(\rmdp\) be an RMDP equipped with
    an observation map \(\alpha \colon S \to \Xi\). A shield
    \(\mathfrak S=(\mathcal D,\Gamma)\) acting on \((\rmdp,\alpha)\) is
    \emph{sound} %
    for \(\spec\) if, for
    every policy $
    \overline\pi$
    on \(\rmdp\otimes_\alpha\mathcal D\) compliant with \(\mathfrak S\), we have $
    (\rmdp,\,\overline\pi^{\downarrow}) \models \Phi$.
\end{definition}

\begin{definition}[Shield approximate optimality]
    \label{def:shield_optimality}
    Let $\spec = \mathbb P_{\geq {1-p}}(\formula)$ be a specification, let $\rmdp$ be an RMDP equipped with observation map $\alpha \colon S \to \Xi$, let $\epsilon>0$, and let $\gamma\in(0,1]$.
    A shield $\mathfrak S=(\mathcal D,\Gamma)$ acting on $(\rmdp,\alpha)$ is
    $(\epsilon,\gamma)$-\emph{optimal} for $\spec$ if, for every
    policy $\pi$ of $\rmdp$ such that $(\rmdp,\pi)\models\spec$, there
    exists a policy $\overline\pi$ of $\rmdp\otimes_\alpha\mathcal D$
    compliant with $\mathfrak S$ such that, for every adversary~$\nature$,
    \[
    J^\gamma_{\rmdp[\nature]}(\overline\pi^{\downarrow})
    \ge
    J^\gamma_{\rmdp[\nature]}(\pi)-\epsilon.
    \]
\end{definition}

We now state the main results of this section.
Recall $\rmdp$ is an RMDP equipped with a \textbf{finite} safety abstraction $\alpha:S\to\Xi$, defining the quotient RMDP $\rmdp/\alpha = \tuple{\Xi,A_\alpha,\mathcal P_\alpha,\xiinit,AP,\labelfunc_\alpha}$.
Also, let \(\spec=\mathbb P_{\ge 1-p}(\formula)\), let
\(\mathcal A=\tuple{Q,2^{AP},\qinit,\delta,F}\) recognize the bad
prefixes of \(\formula\), and let $\beta:\Xi\times Q\to[0,1]$ be an $((\rmdp/\alpha)\otimes_{\labelfunc_\alpha}\mathcal A,\Xi\times F)$-
inductive value function
such that $
\beta(\xiinit,\qinit)\le p$.

\begin{theorem}[Realizability and soundness]
    \label{thm:RMDP_soundness}
    The shield $\mathfrak S(\rmdp/\alpha,\mathcal A,\beta)$ acts on
    $(\rmdp,\alpha)$, is realizable over
    $(\rmdp,\alpha)$, and is sound over $(\rmdp,\alpha)$
    for $\spec$.
\end{theorem}

We say an RMDP $\rmdp$ is \emph{conditionally deterministic} w.r.t.~abstraction $\alpha$ if every reachable abstract history determines a unique concrete history, \ie for every abstract history $\widehat h_t=\xi_0a_0\xi_1a_1\cdots a_{t-1}\xi_t$
with \(\xi_0=\alpha(\sinit)\), there is at most one concrete history $h_t=s_0a_0s_1a_1\cdots a_{t-1}s_t$
such that \(s_0=\sinit\), \(\alpha(s_i)=\xi_i\) for all \(i\le t\), and
\(h_t\) has positive probability under some adversary of \(\rmdp\). 
We remark that this assumption is trivially satisfied if the safety abstraction $\alpha$ is the identity. 
As another example, this assumption is satisfied by the Pacman environment used for our experimental evaluation in Section \ref{sec:experiments}. 
This assumption is needed for \cref{thm:RMDP_optimality} because an optimal safe policy might otherwise induce different safety-relevant behaviors on different concrete states associated with the same abstract history, in which case it might become impossible to certify the safety of that safe optimal policy without some further knowledge of the dynamics of the concrete MDP.

\begin{theorem}[Approximate optimality]
    \label{thm:RMDP_optimality}
    Let
    \(
    \epsilon=
    \frac{\|\beta-\beta^{\infty}\|_\infty}
    {\|\beta-\beta^{\infty}\|_\infty+p-\beta(\xiinit,\qinit)},
    \)
    with \(\beta^\infty\) the least fixed point of 
    \(
    B^{\Xi \times F}_{(M_R/\alpha)\otimes_{L_\alpha} A}
    \),
    suppose \(\beta(\xiinit,\qinit)<p\), and that $\rmdp$ is conditionally deterministic w.r.t. $\alpha$.
    Then, for the shield \(\mathfrak S(\rmdp/\alpha,\mathcal A,\beta)\), we have:
    \begin{enumerate}
        \item if the total undiscounted return is defined on all infinite paths, and if its absolute value is
        uniformly bounded by \(B\), then
        \(\mathfrak S(\rmdp/\alpha,\mathcal A,\beta)\) is
        \(
        (2B\epsilon,1)\text{-optimal}
        \)
        over \((\rmdp,\alpha)\) for \(\spec\);
        \item if \(0<\gamma< 1\) and $
        \sup_{(s,a)\in S\times A}|R(s,a)|\le Z$,
        then \(\mathfrak S(\rmdp/\alpha,\mathcal A,\beta)\) is
        \(
        (\frac{2Z\epsilon}{1-\gamma},\gamma)\text{-optimal}.
        \)
    \end{enumerate}
\end{theorem}

\cref{thm:RMDP_soundness,thm:RMDP_optimality} show that the shield in \cref{def:shield_finite_rmdp} is sound and approximately optimal, respectively.
Furthermore, if the value function $\beta$ used to define the shield is equal to the exact least fixed point~$\beta^\infty$, then our shield is (exactly) optimal, as long as the corresponding return is well-defined.

\section{Shielding Unknown MDPs via Learned RMDPs}
\label{sec:unknown_MDPs}

We use the shielding approach from~\cref{sec:shields_RMDPs} to shield unknown finite MDPs learned as an RMDP.
A common approach to learning an unknown (finite) MDP is to estimate transition probabilities from samples and represent the resulting uncertainty via intervals~\cite{DBLP:conf/qestformats/MeggendorferWW25,DBLP:conf/cdc/NazeriBSSA25,DBLP:journals/jair/BadingsRAPPSJ23,DBLP:conf/nips/SuilenS0022,DBLP:conf/tacas/SchnitzerAP25}. 
For each state-action $(s,a)$ and next state $s'$, we construct an interval $[\TrL(s,a,s'), \TrU(s,a,s')] \subseteq [0,1]$ and define the uncertainty set
\begin{equation*}
    \cP(s,a) = \left\{ \mu \in \distr{\States} \;:\; \forall s' \in \States,\; \mu(s') \in [\TrL(s,a,s'), \TrU(s,a,s')] \right\}.
\end{equation*}
Suppose that, for a fixed pair $(s,a)$, we obtain $N_{s,a}$ samples from $P(s,a)$, out of which $N_{s,a,s'}$ led to a transition to $s'$.
For any confidence budget $\conf \in (0,1)$, the Clopper-Pearson confidence interval~\cite{clopper1934use} provides a \emph{probably approximately correct} (PAC) estimator of the unknown probability $P(s,a,s')$:
\begin{equation*}
    \Prob\big\{ P(s,a,s') \in [\TrL(s,a,s'), \TrU(s,a,s')] \big\} \ge 1 - \conf,
\end{equation*}
which is defined via the inverse CDF of the beta distribution $\mathrm{Beta}(\alpha, \nu)$, denoted by $\mathrm{Beta}^{-1}(\cdot; \alpha, \nu)$:
\begin{align*}
\TrL(s,a,s') &=
\begin{cases}
0, & N_{s,a,s'} = 0,\\
\mathrm{Beta}^{-1}(\tfrac{\conf}{2};\,\, N_{s,a,s'},\, N_{s,a}-N_{s,a,s'}+1), \,\,\,\,\,\,\,\,\,\, & \text{otherwise},
\end{cases}
\\[0.0em]
\TrU(s,a,s') &=
\begin{cases}
1, & N_{s,a,s'} = N_{s,a},\\
\mathrm{Beta}^{-1}(1-\tfrac{\conf}{2};\,\, N_{s,a,s'}+1,\, N_{s,a}-N_{s,a,s'}), & \text{otherwise}.
\end{cases}
\end{align*}
By repeating this procedure for every state-action pair, we obtain the transition function $\cP$ for an~RMDP such that
$
\Prob\left\{ P \in \cP \right\} \geq 1-\numEdges \cdot \conf,
$
with $\numEdges \leq |S|^2 \cdot |A|$ the number of transitions in the~MDP.

\begin{remark}
    \label{remark:minimum_probability}
    Without further assumptions on the MDP, the Clopper-Pearson confidence interval cannot prove the absence of a transition $(s,a,s')$, even if $N_{s,a,s'} = 0$ and $N_{s,a}$ is high.
    A common workaround is to assume a minimum probability $p_\text{min}$ for each transition in the MDP~\cite{DBLP:conf/qestformats/MeggendorferWW25,DBLP:journals/tocl/DacaHKP17}.
    If the sum of the lower bounds $\check{P}(s,a,s')$ for all states $s' \in \States$ with $N_{s,a,s'} > 0$ exceeds $1-p_\text{min}$, then the probability of having overlooked another transition falls within the confidence budget.
\end{remark}
In our experiments, we either assume that (a)~the graph of the MDP (and thus $\numEdges$) is known, or (b)~the graph of the MDP is unknown, but the minimum probability $p_\text{min}$ is known (and $\numEdges = |S|^2 \cdot |A|$).

\paragraph{Shielding unknown MDPs.}
Finally, we use our approach from \cref{sec:shields_RMDPs} to shield a learned RMDP $\widehat{\mdp}$, representing an \textbf{unknown} MDP $\mdp = \tuple{S,A,P,\sinit,AP,\labelfunc,\rewardfunc}$~equipped with a safety abstraction \(\alpha \colon S\to\Xi\) into a \textbf{finite} set $\Xi$.
We write $\mdp/\alpha = \tuple{\Xi,A_\alpha,P_\alpha,\xiinit,AP,\labelfunc_\alpha}$
for the quotient MDP (cf.~\cref{def:quotient_RMDP}), $\spec=\mathbb P_{\ge 1-p}(\formula)$ is a specification, and  $\mathcal A$ is a DFA for the bad prefixes of $\formula$.
In the following, we assume we have learned a finite RMDP 
such that the \emph{true $P_\alpha$} is contained in the learned set $\widehat\cP$,
\ie $P_\alpha(\xi,a)\in\widehat{\mathcal P}(\xi,a)$ 
$\,\forall (\xi,a) \in \Xi \times A_\alpha$.
In practice, this holds with probability $\geq 1-\numEdges \cdot \conf$, so this confidence level carries over to \cref{thm:unknown_MDP_soundness,thm:unknown_MDP_optimality} below.

Let
\(\widehat\beta:\Xi\times Q\to[0,1]\) be an $
(\widehat{\mdp}\otimes_{\labelfunc_\alpha}\mathcal A,\Xi\times F)$-inductive value function
such that $
\widehat\beta(\xiinit,\qinit)\le p$.
Then, the shield $\mathfrak S(\widehat{\mdp},\mathcal A,\widehat\beta)$ computed on the learned RMDP $\widehat\mdp$ is sound for the unknown MDP~$\mdp$:

\begin{theorem}[Soundness for the unknown MDP]
\label{thm:unknown_MDP_soundness}
The shield $
\mathfrak S(\widehat{\mdp},\mathcal A,\widehat\beta)$
acts on \((\mdp,\alpha)\), and every compliant policy $\overline\pi$ satisfies the specification $\Phi$ on the true MDP $\mdp$, \ie $(\mdp,\overline\pi^\downarrow) \models \mathbb P_{\ge 1-p}(\formula)$.

\end{theorem}

We now present our optimality results, which bound the gap in expected return between the optimal policy under the shield $\mathfrak S(\widehat{\mdp},\mathcal A,\widehat\beta)$, and the optimal safe policy on the true MDP $\mdp$.
These bounds depend on the total variation abstraction error
\(
\eta
\coloneqq
\sup_{(\xi,a) \in \Xi \times A_\alpha}
\sup_{\widehat P\in\widehat{\mathcal P}(\xi,a)}
\mathrm{TV}(P_\alpha(\xi,a),\widehat P)
\),
the least fixed points \(\widehat\beta^\infty\) and \(\beta^\infty\) of the Bellman operators $\mathcal B_{\widehat{\mdp}\otimes_{\labelfunc_\alpha}\mathcal A}^{\Xi\times F}$ and~$\mathcal B_{(\mdp/\alpha)\otimes_{\labelfunc_\alpha}\mathcal A}^{\Xi\times F}$ as in \cref{def:robust_Bellman}, and the parameter $q_\text{min} = \frac{p_\text{min}\left(\mathcal M\right) \cdot p_\text{min}(\widehat\mdp)}{p_\text{min}(\mathcal M/\alpha)}$, where $p_\text{min}(\cdot)$ denotes the minimum transition probability of the respective model.
Finally, let $H_{\max}$ be a uniform upper bound on the expected hitting time to either a bad automaton state or a safe absorbing component, taken over all deterministic \(\mathcal A\)-memory policies and all transitions in $\widehat\cP$.
As a conservative bound, we may take $H_{\max}= \frac{|S| \cdot |Q|}{ (q_\text{min})^{|S| \cdot |Q|}}$.

In \cref{thm:unknown_MDP_optimality}, we again use optimality as per \cref{def:shield_optimality} and define 
\(
\widehat\epsilon_\beta
\coloneqq
\frac{
\|\widehat\beta-\widehat\beta^\infty\|_\infty
}{
\|\widehat\beta-\widehat\beta^\infty\|_\infty
+
p-\widehat\beta(\xiinit,\qinit)
}
\) for brevity. 
To ensure the denominator of $\epsilon_{\widehat \beta}$ is nonzero, we assume $\widehat\beta(\xiinit,\qinit) < p$, which in turn implies there exists a policy
\(\pi_{\mathrm{sl}}\) and a constant \(\kappa>0\) such that
\(
\inf_{\widehat\nature}
\prob_{\widehat{\mdp}[\widehat\nature],\pi_{\mathrm{sl}}}(\formula)
\ge
1-p+\kappa.
\)

\begin{theorem}[Near-optimality for the unknown MDP]
    \label{thm:unknown_MDP_optimality}
    If $\mathcal M$ is conditionally deterministic w.r.t. $\alpha$, it holds that:
    \begin{enumerate}
        \item if $\mdp$ is finite, the learned $\widehat\mdp$ is graph-preserving,\footnote{An RMDP $\rmdp$ is graph-preserving if each induced MDP $\rmdp[\theta]$ has the same graph. For $\widehat\mdp$, we can ensure graph-preservation (with the specified confidence budget) by using sufficiently many samples to satisfy \cref{remark:minimum_probability}.} and the absolute value of the total undiscounted return is uniformly bounded by $B$, then the shield $\mathfrak S(\widehat{\mdp},\mathcal A,\widehat\beta)$ is
        $(2B\epsilon_{\mathrm{TV}},1)$-optimal
        over $(\mdp,\alpha)$ for $\Phi$, where $\epsilon_{\mathrm{TV}}
        \coloneqq
        \epsilon_{\widehat\beta}
        +
        \frac{\eta H_{\max}}{\kappa+\eta H_{\max}}$.
        \item if $0 < \gamma < 1$ and $\sup_{(s,a)\in S\times A}|R(s,a)|\le Z$, 
        then the shield $\mathfrak S(\widehat{\mdp},\mathcal A,\widehat\beta)$ is
        $\left(
        \frac{2Z\epsilon_T}{1-\gamma},
        \gamma
        \right)$-optimal
        over $(\mdp,\alpha)$ for $\Phi$, where
        \(
        \epsilon_T
        \coloneqq \min_{T\in\Nat}\left(
        \epsilon_{\widehat\beta}
        +
        \gamma^T
        +
        \frac{T\eta + \|\widehat\beta^\infty-\beta^\infty\|_\infty}{\kappa+T\eta + \|\widehat\beta^\infty-\beta^\infty\|_\infty}
        \right).
        \)
    \end{enumerate}
\end{theorem}

The bounds in \cref{thm:unknown_MDP_optimality} separate three sources of suboptimality. 
First, \(\epsilon_{\widehat\beta}\) comes from using the inductive certificate
\(\widehat\beta\) instead of the exact robust fixed point
\(\widehat\beta^\infty\). 
Second, the terms with~\(\eta\) quantify the mismatch between the learned RMDP and the true abstraction $\mdp/\alpha$. 
Third, the additional term \(\gamma^T\) in the discounted case is the price of switching after \(T\) steps to a robustly safe Slater policy; increasing \(T\) reduces this switching cost but increases the finite-horizon model error
term \(T\eta\).

\section{Experimental Evaluation}
\label{sec:experiments}

We implement our shielding approach and perform numerical experiments to demonstrate our approach for safe reinforcement learning on unknown MDPs via shielding learned RMDPs.

\paragraph{Implementation.}
Our shielding pipeline for unknown MDPs consists of three steps: (1)~Learn an RMDP from sampled transitions; (2)~Construct a shield on the learned RMDP; and (3)~Run policy optimization with the shield active.
We assume access to a simulator of the underlying MDP $\mdp$, but crucially do \emph{not} know its transition function~$P$.
For step (1), we use this simulator (in fact, a simulator of the safety abstraction only suffices) to sample transitions for every state-action pair and learn an RMDP~$\widehat\mdp$ as described in~\cref{sec:unknown_MDPs} that is correct wrt a specified confidence $\delta \in (0,1)$.
For~(2), we run robust value iteration~\cite{DBLP:journals/ior/NilimG05,DBLP:journals/mor/Iyengar05} on the %
product RMDP $\widehat\mdp \otimes_L \mathcal A$ to compute an inductive value function~$\beta$ (cf.~\cref{def:robust_Bellman}) to construct the shield as per \cref{def:shield_finite_rmdp}.
For~(3), we use reinforcement learning on the MDP simulator with the action space restricted to $A^{\mathfrak S}(s,m)=\Gamma(\alpha(s),m)$.
This space $A^{\mathfrak S}(s,m)$ is convex and can thus be given by its extremal points.
In fact, the set of violation budgets $\mathcal V_\beta$ that satisfy condition~2 in \cref{def:shield_finite_rmdp} has finitely many extremal points, allowing us to represent and optimize over the shielded action space (see~\cref{appendix:implementation} for details).
In practice, we optimize the policy with PPO~\cite{DBLP:journals/corr/SchulmanWDRK17} and implement a heuristic for choosing the auxiliary actions, described in \cref{appendix:implementation}.

\paragraph{Experimental setup.}
We conduct experiments comparing our contributions (cases 3 and 4 below) against two baselines that use state-of-the-art probabilistic shielding for MDPs (cases 1 and 2):
\begin{enumerate}
     \item \textbf{Known MDP} (\texttt{MDP}). This baseline assumes full knowledge of the true MDP transition probabilities and uses the probabilistic shielding from~\citep{DBLP:conf/aaai/CourtBG25} directly on the exact model. 
     This baseline provides a reference for safety and performance under full knowledge of the MDP.
    \item \textbf{MLE-based Shielding} (\texttt{MLE}). This baseline applies the probabilistic shielding from~\citep{DBLP:conf/aaai/CourtBG25} on a learned MDP with maximum-likelihood estimates (MLEs) of the transition probabilities.
    While this approach often achieves strong reward performance, the lack of uncertainty modeling prevents theoretical guarantees and can lead to unsafe behavior.
    \item \textbf{RMDP Shield with Known Support} (\texttt{Kno}).
    The MDP is unknown and is learned as an RMDP with PAC guarantees, but the graph of the true MDP is known in advance.
    \item \textbf{RMDP Shield with Unknown Support} (\texttt{Unk}).
    Same as setting 3, but the graph of the true MDP is unknown and must instead be learned based on a given minimum transition probability $p_\text{min}$ as part of the RMDP learning, as described in \cref{remark:minimum_probability}.
\end{enumerate}
We compare these four cases on two variants of three environments: \textbf{Media Streaming}, \textbf{Color Bomb Gridworld}, and \textbf{Pacman} (see~\cref{app:additional_results} for detailed descriptions).
We learn each RMDP~with an overall confidence level of $0.95$, taking between $10$--$20\,000$ samples per distribution $P(s,a)$ in the MDP simulator (except for the \texttt{MDP} baseline, which does not use sampling).
For each configuration, we report two evaluation metrics: (1)~expected discounted reward (\texttt{rew}) is the average episodic return, %
and (2)~specification satisfaction probability (\texttt{sat}) is the empirical probability that an evaluation rollout satisfies the safety specification on 15 episodes.
The experiments are run on a server running Ubuntu 24.04.4 LTS, with an Intel Core i7-10700K CPU, 32 GB of RAM, and NVIDIA GeForce RTX 3080 GPU with 10 GB VRAM.
The code to reproduce our results is provided in the supplementary material.
Our implementation builds upon the MASA library \cite{Goodall2025MASASafeRL}, \eg for the PPO algorithm.

\paragraph{Discussion.}
\Cref{fig:tot-4-size-main} summarizes our experimental results, reporting the averaged metrics (\texttt{rew} and \texttt{sat}) during training as a function of the sample size for each environment.
Learning curves showing the evolution of these metrics are deferred to \cref{app:additional_results} to avoid clutter.
Finally, all experiments are averaged over three random seeds, and shaded regions show the standard deviation.
The \texttt{MDP} baseline corresponds to a perfect-information setting and therefore remains constant across all sample sizes.

First, we observe that \texttt{Kno} and \texttt{Unk} cannot always be evaluated for small sample sizes, as illustrated by \textbf{Pacman Slippery}.
In particular, for too low sample sizes, the learned uncertainty sets $\cP(s,a)$ of the RMDP are too large to construct a shield that satisfies the required satisfaction probability of $1-p$.
Increasing the sample size yields tighter sets $\cP(s,a)$, thereby mitigating this issue.
For \texttt{Unk}, we additionally need at least a certain minimum number of samples from each distribution $P(s,a)$ of the MDP to learn the graph structure, as described in \cref{remark:minimum_probability}.

Despite these sampling requirements,~\cref{fig:tot-4-size-main} shows that the claimed safety probability of $1-p$ is indeed satisfied empirically by our robust shields (\texttt{Kno} and \texttt{Unk}).
As the number of samples increases, the uncertainty sets become progressively tighter, reducing the conservativeness of the robust shield and allowing the methods to achieve increasingly higher expected rewards while maintaining safety.
By contrast, the \texttt{MLE} shield frequently fails to satisfy the required safety threshold for small to moderate sample sizes, despite often achieving strong reward performance, \eg in the \textbf{Pacman} environment.
The reason is that \texttt{MLE} does not explicitly model uncertainty in the learned probabilities and thus cannot provide safety guarantees.

Overall, our results highlight a trade-off between the data used to learn the RMDP and the shield's conservatism.
For small sample sizes, our robust methods lead to suboptimal rewards, while the specification satisfaction probability remains well above the required threshold.
This conservatism is a typical limitation of robust methods, which rely on worst-case reasoning leading to loose bounds for small sample budgets~\cite{DBLP:journals/sttt/BadingsSSJ23,bertsimas2022robust}.
Nevertheless, as the sample size grows, the learned RMDP becomes a tighter approximation to the true MDP, yielding results close to the baseline with a fully known~\texttt{MDP}.

\begin{figure}[tb]
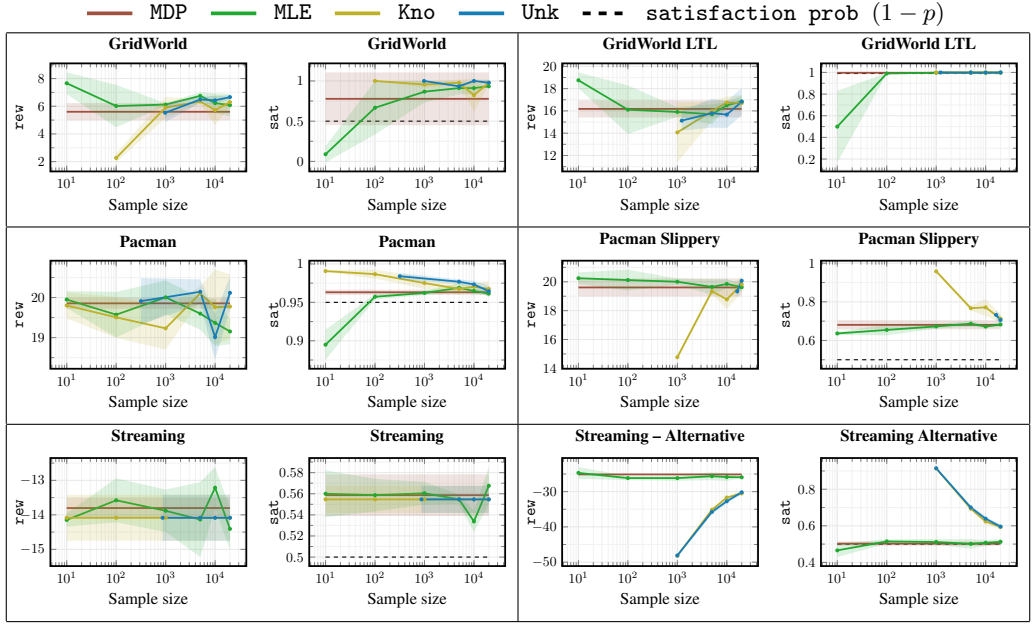

\centering  
    \setlength{\tabcolsep}{2pt}
    \begin{tabular}{|cc|cc|}
        \multicolumn{4}{c}{\input{figures/results/legend}}\\
        \hline
        \input{figures/results/main/GW_finalavg_rew} &
        \input{figures/results/main/GW_finalavg_sat} &
        \input{figures/results/main/GWltl_finalavg_rew} &
        \input{figures/results/main/GWltl_finalavg_sat} \\
        \hline
        \input{figures/results/main/Mini_finalavg_rew} &
        \input{figures/results/main/Mini_finalavg_sat} &
        \input{figures/results/main/MiniSlip_finalavg_rew} &
        \input{figures/results/main/MiniSlip_finalavg_sat} \\
        \hline
        \input{figures/results/main/Stream_finalavg_rew} &
        \input{figures/results/main/Stream_finalavg_sat} &
        \input{figures/results/main/StreamV2_finalavg_rew} &
        \input{figures/results/main/StreamV2_finalavg_sat} \\
        \hline
    \end{tabular}
    \caption{Avg. \texttt{rew} and \texttt{sat} as a function of the number of samples for each state-action pair. The dashed line shows the specification threshold \(1-p\). Shaded regions indicate one standard deviation.}
    \label{fig:tot-4-size-main}
\end{figure}

\section{Conclusion}
We presented a novel shielding approach for robust MDPs (RMDPs).
Our shielding approach is sound and (approximately) optimal and can be used to shield unknown MDPs learned as an RMDP.
By using existing sampling methods for learning transition probabilities with PAC guarantees, we construct shields that, with high confidence, ensure safety while remaining minimally restrictive.
Our experiments confirmed that, while conservative for small sample sizes, our robust shields guarantee probabilistic safety, whereas shields constructed for MDPs with maximum likelihood estimates do not.
In the future, we plan to extend our approach to robust shielding of stochastic games and investigate more sophisticated RMDP learning techniques to reduce the required sample sizes.

\begin{ack}
This research is supported by the EPSRC grant EP/Y028872/1, Mathematical Foundations of Intelligence: An ``Erlangen Programme'' for AI.
Edwin Hamel-De le Court is funded by the ARIA opportunity seed grant titled ``Hardware-Level AI Safety Verification.''
Thom Badings was supported by the RWTH ``Port to Europe'' Postdoc Programme, as part of the Excellence Initiative of the German Research Foundation DFG.
\end{ack}

\bibliographystyle{unsrt}
\bibliography{references}

\newpage

\appendix

\renewcommand{\thesection}{\Alph{section}}

\counterwithin{theorem}{section}
\counterwithin{lemma}{section}
\counterwithin{definition}{section}
\counterwithin{assumption}{section}
\counterwithin{corollary}{section}
\counterwithin{remark}{section}
\counterwithin{example}{section}
\counterwithin{problem}{section}

\setcounter{theorem}{0}
\renewcommand{\thetheorem}{\thesection.\arabic{theorem}}

\setcounter{lemma}{0}
\renewcommand{\thelemma}{\thesection.\arabic{lemma}}

\setcounter{definition}{0}
\renewcommand{\thedefinition}{\thesection.\arabic{definition}}

\setcounter{assumption}{0}
\renewcommand{\theassumption}{\thesection.\arabic{assumption}}

\setcounter{corollary}{0}
\renewcommand{\thecorollary}{\thesection.\arabic{corollary}}

\setcounter{remark}{0}
\renewcommand{\theremark}{\thesection.\arabic{remark}}

\setcounter{example}{0}
\renewcommand{\theexample}{\thesection.\arabic{example}}

\setcounter{problem}{0}
\renewcommand{\theproblem}{\thesection.\arabic{problem}}

\include{appendix_v2}

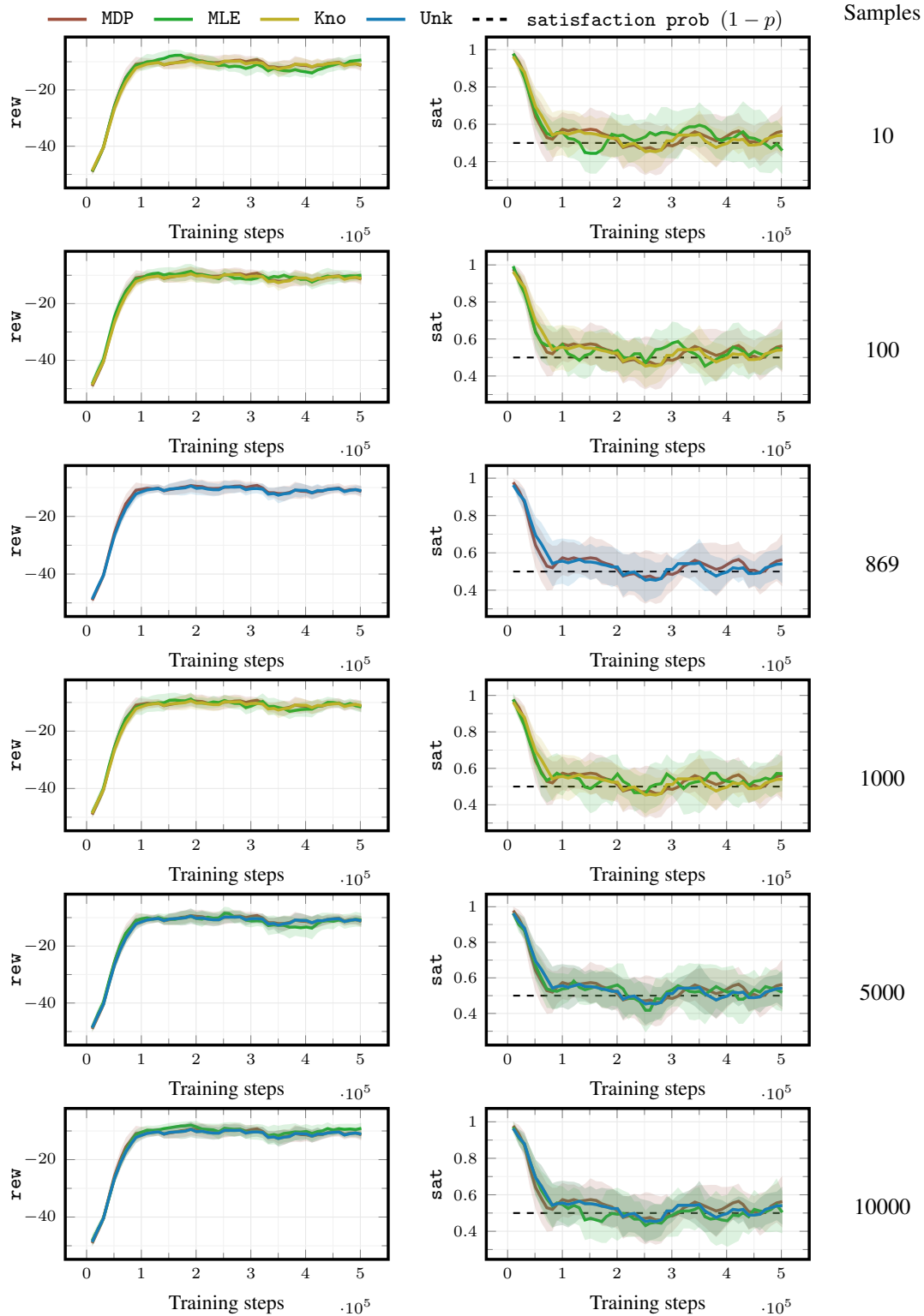
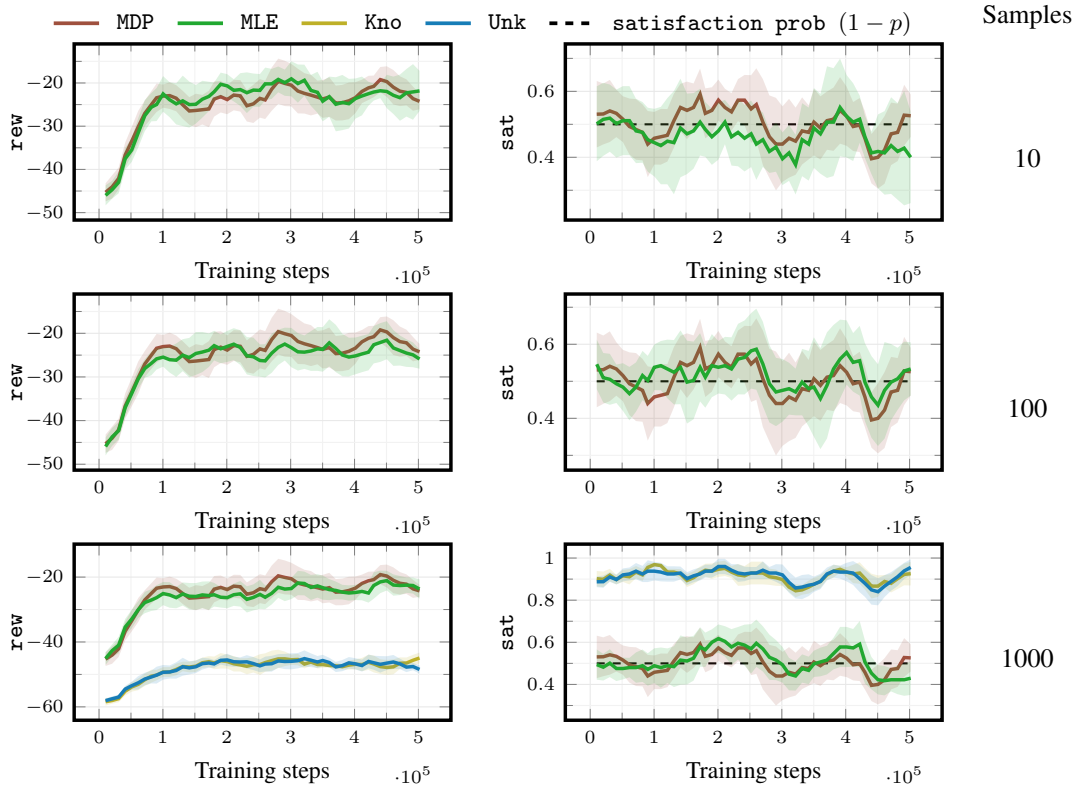
\end{document}

%% file: include/macros.tex
\newcommand*{\spec}{\Phi}
\newcommand*{\formula}{\varphi}

\newcommand*{\nature}{\theta}

\newcommand*{\conf}{\ensuremath{\tau}}
\newcommand*{\numEdges}{\ensuremath{\psi}}

\newcommand*{\Nat}{\mathbb{N}}

\newcommand{\cP}{\mathcal{P}}

\newcommand*{\States}{\ensuremath{S}}

\newcommand*{\Actions}{\ensuremath{A}}
\newcommand*{\sinit}{\ensuremath{s_\text{init}}}
\newcommand*{\xiinit}{\ensuremath{\xi_\text{init}}}
\newcommand*{\qinit}{\ensuremath{q_\text{init}}}
\newcommand*{\minit}{\ensuremath{m_\text{init}}}

\newcommand*{\TrL}{\ensuremath{\check{P}}}
\newcommand*{\TrU}{\ensuremath{\hat{P}}}

\newcommand*{\labelfunc}{\ensuremath{L}}
\newcommand*{\rewardfunc}{\ensuremath{R}}
\newcommand*{\labels}{{\ensuremath{AP}}}

\newcommand*{\mdp}{\ensuremath{\mathcal{M}}}

\newcommand*{\rmdp}{\ensuremath{\mathcal{M}_R}}
\newcommand*{\RMDP}{\ensuremath{\tuple{\States,\Actions,\cP,\sinit,\labels,\labelfunc}}}

\newcommand*{\tuple}[1]{\langle #1 \rangle}

\newcommand*{\Dirac}[1]{\mathrm{Dirac}\left(#1\right)}

\newcommand*{\Prob}{\mathbb{P}}
\newcommand*{\prob}{\textup{prob}}

\newcommand*{\safe}{\text{safe}}

%% file: include/colors.tex
\definecolor{red}{rgb}{0.745,0.192,0.102}
\definecolor{darkgreen}{RGB}{34,161,55}
\definecolor{lightblue}{RGB}{220, 238, 255}
\definecolor{ruhuisstijlrood}{rgb}{0.745,0.192,0.102}
\definecolor{ruhuisstijlzwart}{rgb}{0,0,0}
\definecolor{ruhuisstijlwit}{rgb}{0.98,0.98,0.98}

\definecolor{darkred}{rgb}{0.560784314,0.125490196,0.0666666667}

\definecolor{color1}{RGB}{55,126,184} %
\definecolor{color2}{RGB}{228,26,28} %
\definecolor{color3}{RGB}{77,175,74} %
\definecolor{color4}{RGB}{152,78,163} %
\definecolor{color5}{RGB}{255,127,0} %
\definecolor{color6}{rgb}{0.5, 1.0, 0.83} %
\definecolor{color7}{rgb}{1.0, 0.0, 1.0} %
\definecolor{color8}{rgb}{0.66, 0.66, 0.66} %

\colorlet{plot1}{color1}
\colorlet{plot2}{color2}
\colorlet{plot3}{color3}
\colorlet{plot4}{color4}
\colorlet{plot5}{color5}

%% file: include/crefnames.tex
\crefname{equation}{Eq.}{Eqs.}
\crefname{pluralequation}{Eqs.}{Eqs.}

\crefname{algorithm}{Algorithm}{Algorithm}

\crefname{figure}{Fig.}{Figs.}
\crefname{pluralfigure}{Figs.}{Figs.}

\crefname{section}{Sect.}{Sects.}
\crefname{pluralsection}{Sects.}{Sects.}

\crefname{table}{Table}{Table}
\crefname{pluraltable}{Tables}{Tables}

\crefname{definition}{Def.}{Def.}
\crefname{pluraldefinition}{Defs.}{Defs.}

\crefname{theorem}{Theorem}{Theorems}
\crefname{pluraltheorem}{Theorems}{Theorems}

\crefname{corollary}{Corollary}{Corollaries}
\crefname{pluralcorollary}{Corollaries}{Corollaries}

\crefname{lemma}{Lemma}{Lemmas}
\crefname{plurallemma}{Lemmas}{Lemmas}

\crefname{example}{Example}{Example}
\crefname{pluralexample}{Examples}{Examples}

\crefname{problem}{Problem}{Problem}
\crefname{pluralproblem}{Problems}{Problems}

\crefname{assumption}{Assumption}{Assumption}
\crefname{pluralassumption}{Assumptions}{Assumptions}

\crefname{remark}{Remark}{Remark}
\crefname{pluralremark}{Remarks}{Remarks}

\crefname{appendix}{Appendix}{Appendices}
\crefname{pluralappendix}{Appendices}{Appendices}

\crefname{line}{Line}{Line}
\crefname{pluralline}{Lines}{Lines}

%% file: appendix_v2.tex
\section{Extended Preliminaries}
\label{app:preliminaries}

\begin{remark}
    Throughout the appendices, we provide formalizations of all definitions and results from the paper, including their measure-theoretic constructions and, for example, (R)MDPs with subsets of enabled actions.
    For completeness, we thus repeat all definitions and results from the paper, often extending them with the appropriate notations.
\end{remark}

For a standard Borel space \(X\), we write \(\distr{X}\) for the set of
probability measures on \(X\), equipped with the evaluation
\(\sigma\)-algebra. When \(X\) is endowed with a Polish topology
generating its Borel \(\sigma\)-algebra, we equip \(\distr{X}\) with the
corresponding weak topology; its Borel \(\sigma\)-algebra coincides with
the evaluation \(\sigma\)-algebra. For any set-valued map $f\colon E\to 2^F$, we let $\operatorname{Gr}(f)$ denote the \emph{graph} of f, \ie the set $\{(x,y)\mid x\in E, y\in f(x)\}$.

\subsection{Markov Decision Processes}
A \textit{Markov Decision Process} (MDP) is a tuple $\mathcal
M=\tuple{S,A,P,\sinit,AP,L}$, where $S$ is a standard Borel space of \textit{states}; $A$ is a mapping \(A\colon S\to 2^{\mathsf A}\) where $\mathsf A$ is a standard Borel space of \emph{actions}, that defines for each $s\in S$ the actions $A(s)$ \emph{enabled} in $s$, and such that $\operatorname{Gr}(A)$ is Borel; the \emph{transition function} $P \colon \operatorname{Gr}(A) \to \distr{\States}$ is a measurable map; $\sinit\in S$ is the \textit{initial state}; $AP$ is a finite set of \emph{atomic propositions} equipped with the discrete $\sigma$-algebra; and $L\colon S\to2^{AP}$ is a measurable
\emph{labeling function}. We say that $\mdp$ is an MDP \emph{with rewards} if it is additionally equipped with a measurable \emph{reward function} $\rewardfunc\colon \operatorname{Gr}(A)\to \mathbb R$. 
For simplicity, we may write $P(s,a,s')$ instead of $P(s,a)(\{s'\})$ when $\{s'\}$ is measurable. A finite \emph{path} or \emph{history} in $\mdp$ is a finite word $h=s_0 a_0 \cdots s_{n-1}a_{n-1}s_n$ in $S(\mathsf A S)^n$ for some $n\in\mathbb N$ such that $a_i\in A(s_i)$ for $i\in \{0,\ldots,n-1\}$. The \emph{length} of a history $h=s_0 a_0 \cdots s_{n-1}a_{n-1}s_n$ is $n$, and we denote its last state $s_n$ by $\last{h}$. An \emph{infinite path} or \emph{infinite history} $h=s_0 a_0\cdots s_n a_n\cdots$ is defined analogously.

For \(n\in\mathbb N\), let $
H_n$
be the space of histories of length \(n\), equipped with the trace $\sigma$-algebra induced by the product \(\sigma\)-algebra  of $S\times (\mathsf A \times S)^n$ on $H_n$. A \emph{policy} of \(\mdp\) is a sequence $
\pi=(\pi_n)_{n\ge 0}$,
where each \(\pi_n\) is a measurable map from \(H_n\) to \(\distr{\mathsf A}\) such that $\pi_n(h)(A(\last h))=1$ for any $h\in H_n$. We denote $\pi_n(h)(E)$ by $\pi(E\mid h)$. We write \(\mathrm{HR}\) for the general class of history-dependent randomized
policies. A policy \(\pi\in\mathrm{HR}\) is \emph{memoryless deterministic} if there
exists a measurable map $
d\colon S\to\mathsf A$
such that \(d(s)\in A(s)\) for every \(s\in S\), and for every finite
history \(h\), $
\pi(\cdot\mid h)=\Dirac{d(\last h)}$.
We write \(\mathrm{MD}\) for the class of memoryless deterministic
policies. In the rest of the paper, the policies considered are in $\mathrm{HR}$ unless stated otherwise.

Fixing a policy $\pi$ and an initial state $s \in S$ induces a probability measure $\text{prob}^{s}_{\mdp,\pi}$ over paths of the MDP $\tuple{S,A,P,s,AP,L}$. We write $\text{prob}_{\mdp,\pi}$ for $\text{prob}^{\sinit}_{\mdp,\pi}$.
For details~on policies and induced probability measures, see~\cite{baier2008principles,BSSstochastic}.

\subsection{Safe Linear Temporal Logic}
We consider objectives for MDPs expressed in \emph{linear temporal logic} (LTL).
Following~\cite{baier2008principles}, an LTL formula $\formula$ over the atomic propositions $AP$ is generated by the \emph{grammar}
\[
\formula ::= \top \mid p \mid \neg \formula \mid \formula \land \formula
\mid \mathbf{X}\formula \mid \formula \,\mathbf{U}\, \formula,
\qquad p \in AP.
\]
As usual, we define
$\formula_1 \lor \formula_2 \coloneqq \neg(\neg \formula_1 \land \neg \formula_2)$,
$\mathbf{F}\formula \coloneqq \top \,\mathbf{U}\, \formula$,
and $\mathbf{G}\formula \coloneqq \neg \mathbf{F}\neg \formula$.
A \emph{finite (resp. infinite) trace} is a sequence $\tau = b_0 b_1\cdots \in \Sigma^*$ (resp. $\Sigma^\omega$), where the \emph{alphabet}~$\Sigma = 2^{AP}$ specifies which propositions hold at one instant.
Given an (in)finite path $h=s_0 a_0 s_1 a_1 \cdots$, we denote by $L(h)=L(s_0)L(s_1)L(s_2)\cdots$ %
the (in)finite trace induced by the labeling function. 
The satisfaction relation $\tau,i \models \formula$ for a formula $\formula$ is defined inductively over infinite traces $\tau = b_0 b_1\cdots$ of $\mdp$ as
\begin{alignat*}{2}
\tau,i &\models p &&\iff p \in b_i,\\
\tau,i &\models \neg \formula &&\iff \tau,i \not\models \formula,\\
\tau,i &\models \formula_1 \land \formula_2
&&\iff \tau,i \models \formula_1 \text{ and } \tau,i \models \formula_2,\\
\tau,i &\models \mathbf{X}\formula &&\iff \tau,i+1 \models \formula,\\
\tau,i &\models \formula_1 \,\mathbf{U}\, \formula_2
&&\iff \exists j \geq i \text{ such that } \tau,j \models \formula_2
\text{ and } \tau,k \models \formula_1 \text{ for all } i \leq k < j.
\end{alignat*}
We write $\tau \models \formula$ as shorthand for $\tau,0 \models \formula$. 
An LTL formula $\formula$ is a \emph{safety formula}~\citep{DBLP:journals/fmsd/KupfermanV01} if every violating trace has a finite witness of violation, \ie for every $\tau \in \Sigma^\omega$ such that $\tau \not\models \formula$, there exists a finite prefix $u \prec \tau$ (called a \emph{bad prefix}) such that for every $\rho \in \Sigma^\omega$, $u\rho \not\models \formula$.
For every safety formula $\formula$, %
there exists a DFA that accepts exactly all bad prefixes of $\formula$:

\begin{definition}[DFA, restatement of~\cref{def:DFA}]
    A \emph{deterministic finite automaton} (DFA) over the alphabet $\Sigma$ is a tuple $
    \mathcal{A}=(Q,\Sigma,\qinit,\delta,F)$,
    where $Q$ is a finite set of states, $\qinit \in Q$ is the initial state,
    $\delta\colon Q\times\Sigma\to Q$ is the transition function, and
    $F \subseteq Q$ is the set of accepting states. 
\end{definition}

The transition function extends to finite words in the standard way:
    $\delta^*(q,ua)=\delta(\delta^*(q,u),a)$ for $u\in\Sigma^*$ and $a\in\Sigma$.    
    A finite word $u \in \Sigma^*$ is accepted by the DFA $\mathcal{A}$ iff $\delta^*(\qinit,u)\in F$.

A \emph{specification} $\spec=\mathbb P_{\geq p}(\formula)$ combines an LTL formula $\formula$ and a threshold $p \in [0,1]$.
For a state $s\in S$ and policy $\pi$, we write $(\mdp,\pi,s)\models \mathbb P_{\geq p}(\formula)$
iff $
\text{prob}^{s}_{\mdp,\pi}
\bigl(
\{h \mid L(h)\models \formula\}
\bigr)
\geq p$,
\ie the probability for generating a trace satisfying $\formula$
is at least $p$. 
When $s=\sinit$, we simply write~$(\mdp,\pi)\models \mathbb P_{\geq p}(\formula)$.

\subsection{Robust MDPs and Safety Abstractions}
\emph{Robust MDPs} (RMDPs) extend MDPs with \emph{sets of transition probabilities}~\citep{DBLP:journals/mor/WiesemannKR13}.

\begin{definition}[RMDP, restatement of~\cref{def:RMDP}]
\label{def:RMDP_appendix}
A \emph{robust MDP} (RMDP) is a tuple $\rmdp = \RMDP$, where $\States$, $\Actions$, $\sinit$, $\labels$, and $\labelfunc$ are defined as in an MDP, and $\cP \colon \operatorname{Gr}(A) \to 2^{\distr{\States}}$ is a mapping with measurable graph called an \emph{uncertain transition function}.
\end{definition}

We say that $\rmdp$ is an RMDP \emph{with rewards} if, additionally, it is equipped with a measurable reward function. An RMDP defines a game between the agent that chooses actions via the policy $\pi$, and an \emph{adversary}~$\nature$ that chooses probability distributions in $\cP$.
Formally, an adversary for an RMDP \(\rmdp\) is a sequence $
\nature=(\nature_n)_{n\in\mathbb N}$,
where each \(\nature_n\) assigns to every history \(h\in H_n\) and every
action \(a\in A(\last h)\) a measure
$\nature_n(h,a)\in\mathcal P(\last h,a)$,
and is measurable as a map \((h,a)\mapsto\nature_n(h,a)\). Fixing an adversary \(\nature\) induces an MDP on the history space $
H=\bigsqcup_{n\in\mathbb N}H_n$, whose enabled actions at \(h\in H\) are \(A(\last h)\), and whose transition kernel is given by
\[
P_\nature(h,a)(E)
=
\nature_n(h,a)
\bigl(\{s'\in S: h a s'\in E\}\bigr),
\qquad h\in H_n.
\]
We denote this history MDP by \(\rmdp[\nature]\).
For an RMDP \(\rmdp\) and a policy $\pi$, we write $
\rmdp,\pi\models\spec$
if, for every adversary \(\nature\) of \(\mdp\), $
\rmdp[\nature],\pi\models\spec$.

\begin{remark}
\Cref{def:RMDP,def:RMDP_appendix} describes an \emph{$(s,a)$-rectangular RMDP}, meaning the sets of probabilities between states and actions are independent, and optimal policies can be computed efficiently by robust dynamic programming~\citep{DBLP:journals/mor/WiesemannKR13}.
More generally, RMDPs can be defined with dependent probabilities in different actions in the same state ($s$-rectangularity) or even in different states (non-rectangularity).
However, computing optimal robust policies for non-rectangular RMDPs is NP-hard in general~\citep{DBLP:journals/mor/WiesemannKR13}.
\end{remark}

We could directly define a shield on the RMDP $\rmdp$, but often, parts of the state space are unimportant for the satisfaction of the safety formula $\formula$.
Thus, we follow the common approach of defining a shield on a safety-relevant abstraction of the model~\citep{ABENTShielding}, defined as follows.

\begin{definition}[RMDP abstraction]
    \label{def:abstraction}
    Let $\rmdp = \RMDP$ be an RMDP, and let $\Xi$ be a standard Borel space.
    A surjective measurable map $\alpha \colon S\to \Xi$ is an \emph{abstraction} of $\rmdp$ if the following conditions hold for all $s,s'\in S$:
    \begin{enumerate}
        \item if $\alpha(s)=\alpha(s')$, then $\labelfunc(s)=\labelfunc(s')$ and $A(s)=A(s')$;
        \item for every action $a\in A(s)$, we have
    \[
    \left\{\mu\circ \alpha^{-1}\mid \mu\in\mathcal P(s,a)\right\}
    =
    \left\{\mu\circ \alpha^{-1}\mid \mu\in\mathcal P(s',a)\right\};
    \]
        \item the induced correspondence \(\mathcal P_\alpha\) defined by $
    \mathcal P_\alpha(\xi,a)
    =
    \left\{\mu\circ\alpha^{-1}:\mu\in\mathcal P(s,a)\right\}$ for any
     $s\in\alpha^{-1}(\xi)$ has measurable graph.
    \end{enumerate}
\end{definition}

Intuitively, condition~1 requires that all states $s,s' \in S$ mapping to the same abstract state have the same label and enabled actions, and condition~2 that these states induce the same sets of transition probabilities under the abstraction map.
Since $\alpha$ is surjective and refines the labeling, the induced labeling map $
\labelfunc_\alpha\colon \Xi\to 2^{AP}$
is well-defined by $
\labelfunc_\alpha(\xi)=\labelfunc(s)$ for any $s\in\alpha^{-1}(\{\xi\})$.

\begin{definition}[Quotient RMDP]
    \label{def:quotient_RMDP}
    Let $\rmdp$ be an RMDP and let $\alpha\colon S\to\Xi$ be an RMDP abstraction.
    The \emph{quotient RMDP} of $\rmdp$ by $\alpha$ is $
    \rmdp/\alpha
    =
    \tuple{\Xi,A_\alpha,\mathcal P_\alpha,\xiinit,AP,\labelfunc_\alpha}$,
    where $\xiinit=\alpha(\sinit)$,  $A_\alpha(\xi)=A(s)$ for any $s\in\alpha^{-1}(\xi)$, $
        \mathcal P_\alpha(\xi,a)
        =
        \left\{\mu\circ \alpha^{-1}\mid \mu\in\mathcal P(s,a)\right\}$
        for any $s\in\alpha^{-1}(\xi)$, $\labelfunc_\alpha$ is given by the abstraction property above.
\end{definition}

\section{Probabilistic Shields for Robust MDPs}

In this appendix, we present all extended definitions, results, and proofs for our RMDP shielding approach from~\cref{sec:shields_RMDPs}.

\subsection{Shields for RMDPs}
In this section, we formally define the notion of shield for an RMDP and how it constrains the action choices when deployed on an RMDP.
We first define the notion of \emph{controllable monitor}.

\begin{definition}[Controllable monitor, restatement of~\cref{def:monitor}]
    A controllable monitor over a measurable observation space \(\Xi\)
    is a tuple $\mathcal D=\tuple{M, U,\minit,\zeta}$,
    where \(M\) is a measurable space of \emph{monitor states},
    \(U\) is a measurable space of \emph{auxiliary actions},
    \(\minit\in M\), and $\zeta\colon M\times \Xi\times U\to M$ is a measurable \emph{transition function}.
\end{definition}

For simplicity, we identify DFAs and their associated controllable monitor.
Recall that, for any measurable map $T\colon X\to Y$ and any measure $\mu$ on $X$,
the pushforward of $\mu$ by $T$ is the measure $\mu\circ T^{-1}$ on $Y$
defined by $(\mu\circ T^{-1})(E)=\mu(T^{-1}(E))$
for every measurable set $E\subseteq Y$.

An \emph{observation map} for a state space $S$ is a measurable map $\alpha\colon S\to\Xi$
into the observation space $\Xi$ of a controllable monitor.

\begin{lemma}[Measurability of the product uncertainty graph]
Let \(S,\mathsf A,M,\Xi,U\) be standard Borel spaces. Let
\(A\colon S\to 2^\mathsf A\) have measurable graph, and let $
\mathcal P\colon \operatorname{Gr}(A)\to 2^{\distr{S}}$
have Borel measurable graph. Let \(\alpha\colon S\to\Xi\) be measurable and let $
\zeta\colon M\times\Xi\times U\to M$
be measurable. For \(m\in M\) and \(u\in U\), define
\[
T_{m,u}\colon S\to S\times M,
\qquad
T_{m,u}(s')
=
\bigl(s',\zeta(m,\alpha(s'),u)\bigr).
\]
Define the product uncertainty correspondence by
\[
\overline{\mathcal P}((s,m),(a,u))
=
\left\{
\mu\circ T_{m,u}^{-1}
:
\mu\in\mathcal P(s,a)
\right\}.
\]
Then \(\operatorname{Gr}(\overline{\mathcal P})\) is Borel measurable.
\end{lemma}
\begin{proof}
First, since \(\alpha\) and \(\zeta\) are measurable, the map
\(
K\colon M\times U\times S\to S\times M
\)
defined by
\[
K(m,u,s')
=
\bigl(s',\zeta(m,\alpha(s'),u)\bigr)
\]
is measurable.
Hence, the induced pushforward map $
F\colon M\times U\times\distr{S}\to\distr{S\times M}$ such that $F(m,u,\mu)=\mu\circ T_{m,u}^{-1}$,
is measurable, since for every measurable \(E\subseteq S\times M\), we have $$
F(m,u,\mu)(E)
=
\mu\bigl(\{s'\in S:K(m,u,s')\in E\}\bigr).$$
Now define
\(
D
=
\{(s,m,a,u,\mu):
(s,a,\mu)\in\operatorname{Gr}(\mathcal P)\}.
\)
Since \(\operatorname{Gr}(\mathcal P)\) is Borel, \(D\) is Borel.
Let
\(
H\colon D\to
(S\times M)\times(\mathsf A\times U)\times\distr{S\times M}
\)
be the map such that
\[
H(s,m,a,u,\mu)
=
\bigl((s,m),(a,u),F(m,u,\mu)\bigr).
\]
The map \(H\) is Borel and injective. Indeed, from
\(
\nu=F(m,u,\mu)=\mu\circ T_{m,u}^{-1},
\)
we recover
\(
\mu=\nu \circ \mathrm{pr}_S^{-1},
\)
with the projection \(\mathrm{pr}_S\colon S\times M\to S\). Thus the value
of \(\mu\) is uniquely determined by \(\nu\), \(m\), and~\(u\).

By the Lusin--Souslin theorem, the image of a Borel set under an injective Borel map between standard Borel spaces is Borel. Hence \(H(D)\) is Borel.

Finally,
\(
H(D)
=
\operatorname{Gr}(\overline{\mathcal P}).
\)
Therefore \(\operatorname{Gr}(\overline{\mathcal P})\) is Borel measurable.
\end{proof}
The above lemma allows us to define the product of an RMDP and a controllable monitor.
\begin{definition}[Product RMDP, extension of~\cref{def:product-rmdp-monitor}]
    \label{def:product-rmdp-monitor_appendix}
    Let $\rmdp=\tuple{S,A,\mathcal P,\sinit,AP,\labelfunc}$
    be an RMDP, let $\alpha\colon S\to\Xi$ be an observation map, and let $\mathcal D=\tuple{M, U,\minit,\zeta}$
    be a controllable monitor over~$\Xi$.
    The \emph{product RMDP} of $\rmdp$ and $\mathcal D$ (relative to
    $\alpha$) is the RMDP $
    \rmdp\otimes_{\alpha} \mathcal D=\tuple{\overline S,\overline A,\overline{\mathcal P},\overline{\sinit},
    AP,\overline{\labelfunc}}$, such that $\overline S=S\times M$, $\overline{\sinit}=(\sinit,\minit)$, and for any $(s,m)\in \overline S$, 
    \begin{itemize}
        \item $\overline{\labelfunc}(s,m)=\labelfunc(s)$,
        \item $\overline A$ is a mapping from $S$ to $2^{\overline{\mathsf A}}$, where $\overline{\mathsf A}=\mathsf A\times U$, such that $\overline A(s,m)=A(s)\times U$;
    
        \item for every $(a,u)\in\overline A((s,m))$,
        letting $
        T(s')
        =
        \bigl(s',\zeta(m,\alpha(s'),u)\bigr)
        $, the uncertainty set
        $\overline{\mathcal P}((s,m),(a,u))$
        consists of all pushforward measures
        $\mu\circ T^{-1}$ with $\mu\in\mathcal P(s,a)$.
    \end{itemize}
\end{definition}

We now state our general definition of a shield, which extends \cref{def:shield} from the main paper with the necessary measurability requirements.

\begin{definition}[Shield, extension of~\cref{def:shield}]
    Let $\Xi$ be a standard Borel space, and let $\mathsf A_{\Xi}$ be a standard Borel space of actions, let $A_\Xi$ assign to each
    $\xi\in\Xi$ a subset $A_\Xi(\xi)\subseteq \mathsf A_{\Xi}$ so that the graph $\{(\xi,a)\mid \xi\in \Xi,a\in A_\Xi(\xi)\}$ is measurable. A \emph{shield} over $(\Xi,A_\Xi)$ is a pair $
    \mathfrak S=(\mathcal D,\Gamma)$, where $\mathcal D=\tuple{M,U,\minit,\zeta}$
    is a controllable monitor over $\Xi$,
    and where $\Gamma$ is a mapping from $\Xi\times M$ to $2^{\distr{\mathsf A_\Xi\times U}}$ such that
    \begin{itemize}
        \item for every
    \((\xi,m)\in\Xi\times M\), \(\Gamma(\xi,m)\) is a closed subset of
    \(\distr{\mathsf A_\Xi\times U}\);
        \item for every open set
    \(O\subseteq\distr{\mathsf A_\Xi\times U}\), the set
    \(
    \{(\xi,m)\in\Xi\times M:\Gamma(\xi,m)\cap O\neq\emptyset\}
    \)
    is measurable;
        \item for every \((\xi,m)\in\Xi\times M\) and every
    \(\rho\in\Gamma(\xi,m)\), we have $
    \rho(A_\Xi(\xi)\times U)=1$.
    \end{itemize}
\end{definition}

\subsection{A Shield for Finite RMDPs}
The robust Bellman operator
\(\mathcal B_{\rmdp}^{U}\) acts on functions
\(f\in [0,1]^S\) as
\[
\left(\mathcal B_{\rmdp}^{U}(f)\right)(s)
=
\begin{cases}
1, & \text{ if } s\in U,\\[1mm]
\displaystyle \inf_{a\in A(s)}\ \sup_{\mu\in\mathcal P(s,a)} \mathbb E_{s'\sim \mu}
\left[f(s')\right],
& \text{otherwise.}
\end{cases}
\]

\begin{definition}[Restatement of~\cref{def:robust_Bellman}]
    An $(\rmdp,U)$-\emph{inductive value function} is a function $\beta\colon S\to [0;1]$ such that, for all $s\in S$, $\left(\mathcal B_{\rmdp}^{U}(\beta)\right)(s)-\beta(s)\leq 0$.
\end{definition}

In the following, we let $\rmdp=\tuple{S,A,\mathcal P,\sinit,AP,\labelfunc,\rewardfunc}$ with $A\colon S\to 2^{\mathsf A}$ be a \textbf{finite} RMDP, we let $p\in [0;1]$ be a safety threshold, and we let $\spec=\mathbb P_{\geq 1-p}(\formula)$ be a probabilistic safety LTL formula. We let $\mathcal A=\tuple{Q,2^{AP},\qinit,\delta,F}$ be a DFA recognizing the bad prefixes of $\formula$. We additionally let $\beta$ be an $(\rmdp\otimes_{\labelfunc}\mathcal A,S\times F)$-inductive value function such that $\beta(\sinit,\qinit) \leq p$.
Finally, we let \(\mathcal V_\beta=\prod_{(s,q)\in S\times Q}[\beta(s,q),1]\).

\begin{definition}[Shield for finite RMDP, extension of~\cref{def:shield_finite_rmdp}]
    \label{def-abstract-shield}
    Let $\rmdp$, $\mathcal A$, and $\beta$ be defined as above.
    The shield $\mathfrak S(\rmdp,\mathcal A,\beta) = (\mathcal D,\Gamma)$ is defined as an instantiation of \cref{def:shield} as follows:
    \begin{enumerate}
        \item $\mathcal D = \tuple{M, \mathcal V_\beta,\minit,\zeta}$ is a controllable monitor over $S$ with 
        states $M \coloneqq Q \times [0,1]$,
        auxiliary actions $\mathcal V_\beta$,
        initial state $\minit = (\qinit,p)$, and
        transition function $\zeta \colon M \times S \times \mathcal V_\beta \to M$ defined for all $(q,y) \in M$, $s \in S$, and $v \in \mathcal V_\beta$ as
        \[
        \zeta((q,y),s,v)
        = \bigl(\delta(q,\labelfunc(s)),
        v(s,\delta(q,\labelfunc(s)))\bigr);
        \]
        \item for all states $(s,q,y)$ of $\rmdp\otimes\mathcal D$, the set $\Gamma(s,q,y)$ contains exactly those distributions $\overline\rho \in \distr{\mathsf A\times\mathcal V_\beta}$ such that $\overline\rho(A(s)\times\mathcal V_\beta)=1$, and such that
        \[
        \mathbb E_{(a,v)\sim\overline\rho}
        \Big[
        \sup_{\mu\in\mathcal P(s,a)}
        \mathbb E_{s'\sim\mu}
        \left[
        v\bigl(s',\delta(q,\labelfunc(s'))\bigr)
        \right]
        \Big]
        \le y.
        \]
    \end{enumerate}

\end{definition}

\begin{lemma}
The pair \(\mathfrak S(\rmdp,\mathcal A,\beta)=(\mathcal D,\Gamma)\) is a shield over \((S,A)\).
\end{lemma}

\begin{proof}
Since \(S\), \(Q\), and \(\mathsf A\) are finite, we equip them with the
discrete topology. The space
\[
\mathcal V_\beta
=
\prod_{(s,q)\in S\times Q}[\beta(s,q),1]
\]
is a compact metric space, as a finite product of compact intervals.
Thus, $
M=Q\times[0,1]$
and $
\mathsf A\times\mathcal V_\beta$
are compact metric spaces. We equip
\(\distr{\mathsf A\times\mathcal V_\beta}\) with the weak topology.

The monitor transition
\(
\zeta((q,y),s,v)
=
\bigl(\delta(q,L(s)),v(s,\delta(q,L(s)))\bigr)
\)
is measurable, since \(S\) and \(Q\) are finite and
\(v\mapsto v(s,q')\) is a coordinate projection on
\(\mathcal V_\beta\).

It remains to show that \(\Gamma\) is a closed-valued measurable
correspondence. For fixed \(s\in S\) and \(q\in Q\), define
\[
G_{s,q}(a,v)
\coloneqq
\begin{cases}
\displaystyle
\sup_{\mu\in\mathcal P(s,a)}
\mathbb E_{s'\sim\mu}
\left[
v\bigl(s',\delta(q,L(s'))\bigr)
\right],
& a\in A(s),\\[3mm]
0, & a\notin A(s).
\end{cases}
\]
For \(a\in A(s)\), the map \(v\mapsto G_{s,q}(a,v)\) is continuous.
Indeed, for any \(v,w\in\mathcal V_\beta\),
\[
|G_{s,q}(a,v)-G_{s,q}(a,w)|
\le
\|v-w\|_\infty.
\]
Since \(\mathsf A\) is finite and discrete, \(G_{s,q}\) is a bounded
continuous function on \(\mathsf A\times\mathcal V_\beta\).
Now, for \((s,q,y)\in S\times Q\times[0,1]\),
\[
\Gamma(s,q,y)
=
\left\{
\overline\mu\in\distr{\mathsf A\times\mathcal V_\beta}:
\overline\mu(A(s)\times\mathcal V_\beta)=1
\text{ and }
\int G_{s,q}(a,v)\,\overline\mu(d(a,v))\le y
\right\}.
\]
The set \(A(s)\times\mathcal V_\beta\) is closed in
\(\mathsf A\times\mathcal V_\beta\). Therefore the constraint
\(
\overline\mu(A(s)\times\mathcal V_\beta)=1
\)
defines a closed subset of
\(\distr{\mathsf A\times\mathcal V_\beta}\) by the Portmanteau theorem.
Moreover, because \(G_{s,q}\) is bounded and continuous, the map
\(
\overline\mu\mapsto
\int G_{s,q}(a,v)\,\overline\mu(d(a,v))
\)
is continuous under weak convergence. Hence each value
\(\Gamma(s,q,y)\) is closed.

Finally, since \(S\) and \(Q\) are finite and the defining constraints are
Borel in \(y\) and \(\overline\mu\), the correspondence \(\Gamma\) is
measurable. Therefore \(\mathfrak S(\mdp,\mathcal A,\beta)\) is a shield
over \((S,A)\).
\end{proof}

\subsection{Realizability, Soundness and Optimality}

For any RMDP $\rmdp=\tuple{S,A,\mathcal P,\sinit,AP,\labelfunc}$ with $A\colon S\to 2^{\mathsf A}$ and monitor observation map $\alpha$, we say that a shield
$\mathfrak S=(\mathcal D,\Gamma)$ over $(\Xi,A_\Xi)$ with $A_{\Xi}\colon  \Xi \to 2^{\mathsf A_\Xi}$ \emph{acts on} $(\mdp,\alpha)$ if $\mathsf A=\mathsf A_{\Xi}$, and if
$A(s)=A_\Xi(\alpha(s))$ for all $s\in S$. 

If \(\nature\) is an adversary of the RMDP \(\rmdp\), we denote by
\(\nature^\otimes\) the canonical adversary of the product RMDP
\(\rmdp\otimes_\alpha\mathcal D\) induced by \(\nature\), which is defined as
follows. 
For a product history \(\overline h_t\) ending in \((s,m)\), and
for a product action \((a,u)\in A(s)\times U\), let
\(T_{m,u}\colon S\to S\times M\) be defined by
\[
T_{m,u}(s')
=
\bigl(s',\zeta(m,\alpha(s'),u)\bigr).
\]
Then
\[
\nature^\otimes_t(\overline h_t,(a,u))
=
\nature_t(\overline h_t^\downarrow,a)\circ T_{m,u}^{-1}.
\]
A distribution \(\rho\in\distr{\mathsf A\times U}\) is
\emph{compliant with \(\mathfrak S\) at a product history
\(\overline h_t\)} ending in \((s,m)\) if
\[
\rho\in\Gamma(\alpha(s),m).
\]

Let \(\overline\pi\) be an HR policy on
\(\rmdp\otimes_\alpha\mathcal D\). We say that \(\overline\pi\) is
\emph{compliant with \(\mathfrak S\) up to time \(t\) under the 
adversary \(\nature\)} of $\rmdp$ if
\[
\prob_{(\rmdp\otimes_\alpha\mathcal D)[\nature^\otimes],\overline\pi}
\left[
\forall k<t,\;
\overline\pi(\cdot\mid\overline H_k)
\in
\Gamma(\alpha(S_k),M_k)
\right]
=
1,
\]
where \(\overline H_k\) is the product history up to time \(k\), and
\((S_k,M_k)\) is its last product state.

We say that \(\overline\pi\) is \emph{compliant with \(\mathfrak S\) up
to time \(t\)} if it is compliant up to time \(t\) under every adversary
\(\nature\) of the base RMDP \(\rmdp\). Finally, \(\overline\pi\) is
\emph{compliant with \(\mathfrak S\)} if it is compliant up to time \(t\)
for every \(t\in\Nat\).

\begin{definition}[Non-blocking Shield]
    Let $(\mathcal D,\Gamma)$ be a shield acting on $(\rmdp,\alpha)$, and let let $D_\Gamma$ denote the set $\{(\xi,m)\in\Xi\times M:\Gamma(\xi,m)\neq\emptyset\}$. The shield is $(\mathcal D,\Gamma)$ \emph{non-blocking} over $(\rmdp,\alpha)$ if $(\xiinit,\minit)\in D_\Gamma$ and if, for every \((\xi,m)\in D_\Gamma\) and every
\(\rho\in\Gamma(\xi,m)\), we have
\[
\mathbb E_{(a,u)\sim\rho}
\left[
\sup_{\mu\in\mathcal P_\alpha(\xi,a)}
\mu\!\left(
\left\{
\xi'\in\Xi:
(\xi',\zeta(m,\xi',u))\notin D_\Gamma
\right\}
\right)
\right]
=0.
\]
\end{definition}

We use the following concatenation notation for history-dependent
randomized policies. Let \(t\in\Nat\), and let \(\overline\pi^-\) and
\(\overline\pi^+\) be policies on
\(\rmdp\otimes_\alpha\mathcal D\). We define $
\overline\pi^-\star_t\overline\pi^+$
as the policy that follows \(\overline\pi^-\) for the first \(t\)
decision times and then follows \(\overline\pi^+\). Formally, for every
augmented history \(\overline h_k\) of length \(k\),
\[
(\overline\pi^-\star_t\overline\pi^+)(\cdot\mid \overline h_k)
=
\begin{cases}
\overline\pi^-(\cdot\mid \overline h_k), & k<t,\\[1mm]
\overline\pi^+(\cdot\mid \overline h_k), & k\ge t.
\end{cases}
\]

Similarly, if \(\overline\nature^-\) and \(\overline\nature^+\) are
adversaries of \(\rmdp \otimes_\alpha\mathcal D\), we write $
\overline\nature^-\star_t\overline\nature^+$
for the adversary defined by
\[
(\overline\nature^-\star_t\overline\nature^+)_k
(\overline h_k,\overline a)
=
\begin{cases}
\overline\nature^-_k(\overline h_k,\overline a), & k<t,\\
\overline\nature^+_k(\overline h_k,\overline a), & k\ge t.
\end{cases}
\]
\begin{definition}[Realizable Shield, extension of~\cref{def:shield_reliability}]
A shield
\(\mathfrak S=(\mathcal D,\Gamma)\) acting on \((\rmdp ,\alpha)\) is
\emph{realizable} over \((\rmdp ,\alpha)\) if, for
every \(t\in\Nat\), every adversary \(\overline\nature^-\) of
\(\rmdp \otimes_\alpha\mathcal D\), and every policy
\(\overline\pi^-\) compliant with \(\mathfrak S\) up to time
\(t\) under \(\overline\nature^-\), there exists a continuation policy
\(\overline\pi^+\) such that, for every adversary
\(\overline\nature^+\), the concatenated policy $
\overline\pi^-\star_t \overline\pi^+$
is compliant with \(\mathfrak S\) under the concatenated adversary $\overline\nature^-\star_t \overline\nature^+$.
\end{definition}

\begin{lemma}[Non-blocking shields are realizable]
Every nonblocking shield is realizable.
\end{lemma}
\begin{proof}
Let \(\mathfrak S=(\mathcal D,\Gamma)\) be nonblocking, and let
\[
D_\Gamma=\{(\xi,m):\Gamma(\xi,m)\neq\emptyset\}.
\]
Since \(\Gamma\) is closed-valued and measurable, \(D_\Gamma\) is
measurable. By the Kuratowski--Ryll-Nardzewski selection theorem, there is
a measurable selector
\(
g\colon D_\Gamma\to\distr{\mathsf A\times U}
\)
such that \(g(\xi,m)\in\Gamma(\xi,m)\) for every
\((\xi,m)\in D_\Gamma\).
Fix \(t\in\Nat\), an adversary \(\overline\nature^-\), and a policy
\(\overline\pi^-\) compliant up to time \(t\) under
\(\overline\nature^-\). We first show that $(\alpha(s_t),m_t)\in D_\Gamma$ almost surely.
This follows by induction on \(k\le t\). At \(k=0\), since the shield is non-blocking,
\((\alpha(\sinit),\minit)\in D_\Gamma\). For the induction step, if
\((\alpha(s_k),m_k)\in D_\Gamma\) almost surely and \(k<t\), compliance
gives that the action distribution used at time \(k\) belongs to
\(\Gamma(\alpha(s_k),m_k)\) almost surely. Since the shield is non-blocking, any such
distribution keeps the next abstract-monitor state in \(D_\Gamma\) with
conditional probability one, for every admissible adversarial transition.
Thus \((\alpha(s_{k+1}),m_{k+1})\in D_\Gamma\) almost surely.

Now define a memoryless randomized continuation policy
\(\overline\pi^+\) by
\[
\overline\pi^+(s,m)
=
\begin{cases}
g(\alpha(s),m), & \text{if }(\alpha(s),m)\in D_\Gamma,\\
\rho_0(s,m), & \text{otherwise,}
\end{cases}
\]
where \(\rho_0(s,m)\) is any measurable valid distribution supported on
\(A(s)\times U\). The value outside \(D_\Gamma\) is irrelevant after a
compliant prefix.

Let \(\overline\nature^+\) be any adversary continuation. Under
\(\overline\nature^-\star_t\overline\nature^+\), the process is in
\(D_\Gamma\) at time \(t\) almost surely. From then on, whenever the
process is in \(D_\Gamma\), the continuation chooses
\(g(\alpha(s),m)\in\Gamma(\alpha(s),m)\). Since the shield is non-blocking, the next
abstract-monitor state remains in \(D_\Gamma\) almost surely, independently
of the future adversary. Inducting over all future times, the concatenated
policy $
\overline\pi^-\star_t\overline\pi^+$
is compliant under $
\overline\nature^-\star_t\overline\nature^+$.
Since \(\overline\nature^+\) was arbitrary, the shield is
realizable.
\end{proof}

Given a history $
\overline h_t=((s_0,m_0),(a_0,u_0),\dots,(a_{t-1},u_{t-1}),(s_t,m_t))$ of \(\rmdp \otimes_{\alpha} \mathcal D\), its \emph{projection} onto \(\rmdp\) is the history $
\overline h_t^{\downarrow}=(s_0,a_0,s_1,a_1,\dots,s_{t-1},a_{t-1},s_t)$.

\begin{definition}[Projected policy]
    \label{def:projected_policy}
    Let $\rmdp$ be an RMDP, let $\alpha\colon S\to\Xi$ be a monitor observation map, let
    $\mathcal D$ be a controllable monitor over $\Xi$, and let
    $\overline\pi$ be a policy on $\rmdp\otimes_\alpha\mathcal D$.
    For every adversary $\nature$ of the transition uncertainty of $\rmdp$, we
    denote by $\overline\pi^{\downarrow,\nature}$
    the \emph{projection of $\overline\pi$ onto $\rmdp$ under $\nature$}, defined as follows:
    
    For every projected history $h$, let $
    \nu_t^\theta(\cdot\mid h)$ be the conditional distribution of the augmented histories
    $\overline H$ given $\overline H^\downarrow=h$, as induced by
    the policy $\overline\pi$ in the
    $(\rmdp\otimes_\alpha \mathcal D)[\nature^\otimes]$.
    The \emph{projection of \(\overline\pi\) onto \(\rmdp\) under \(\nature\)} is the
    policy \(\overline\pi^{\downarrow,\nature}\) defined by
    \[
    \overline\pi^{\downarrow,\nature}(E\mid h)
    =
    \int
    \overline\pi(E\times U\mid \overline h)\,
    \nu_t^\nature(d\overline h\mid h)
    \]
    for every measurable \(E\subseteq\mathsf A\).
\end{definition}

For two policies \(\pi\) and \(\pi'\) on an RMDP \(\rmdp\), and for an
adversary \(\nature\) of \(\rmdp\), we write $\pi\equiv_{\nature}\pi'$ if they induce the same probability measure on infinite state-action paths
under \(\nature\), \ie $
\prob_{\rmdp[\nature],\pi}
=
\prob_{\rmdp[\nature],\pi'}$.

\begin{lemma}[Adversary-independent projection]
Let \(\rmdp\) be an RMDP, let \(\alpha\colon S\to\Xi\) be a monitor observation
map, let \(\mathcal D\) be a controllable monitor over \(\Xi\), and let
\(\overline\pi\) be an HR policy on \(\rmdp\otimes_\alpha\mathcal D\).
Then there exists an HR policy \(\overline\pi^\downarrow\) on \(\rmdp\)
such that, for every adversary \(\nature\) of \(\rmdp\), for every
\(t\in\Nat\),
\[
\overline\pi^{\downarrow,\nature}\equiv_\nature\overline\pi^\downarrow\]
\end{lemma}
\begin{proof}
We construct a common version of the projection. For each \(t\), define a
kernel
\[
\nu_t(\cdot\mid h)\colon H_t\to\distr{\overline H_t}
\]
which will serve as a conditional law of augmented histories given their
projection, independently of the adversary.
This kernel is defined as follows:

\begin{itemize}
    \item
    At time \(0\), set
    \(
    \nu_0(\cdot\mid s_0)=\Dirac{(s_0,\minit)}.
    \)
    Assume that \(\nu_t\) has been constructed. Fix a projected history
    \(h_t=s_0a_0\cdots a_{t-1}s_t\). Draw
    \(\overline h_t\sim\nu_t(\cdot\mid h_t)\), and then draw a product action
    \((a,u)\sim\overline\pi(\cdot\mid\overline h_t)\). Since the spaces are
    standard Borel, there exists a regular conditional distribution $
    R_t(d\overline h_t,du\mid h_t,a)$
    of \((\overline h_t,u)\) given the observed base action \(a\).
    
    \item
    For \(h_{t+1}=h_ta_ts_{t+1}\), define
    \[
    \Phi_{a_t,s_{t+1}}(\overline h_t,u)
    =
    \overline h_t\,(a_t,u)\,
    \bigl(s_{t+1},\zeta(m_t,\alpha(s_{t+1}),u)\bigr),
    \]
    where \(m_t\) is the monitor component of the last state of
    \(\overline h_t\). Then set
    \[
    \nu_{t+1}(\cdot\mid h_ta_ts_{t+1})
    =
    R_t(\cdot\mid h_t,a_t)\circ \Phi_{a_t,s_{t+1}}^{-1}.
    \]
    On histories where the conditional distribution is not uniquely
    determined, choose an arbitrary version.
\end{itemize}

We claim that, for every adversary \(\nature\) of \(\rmdp\), the kernel
\(\nu_t\) is a version of the conditional law of \(\overline H_t\) given
\(\overline H_t^\downarrow\) under
\((\rmdp\otimes_\alpha\mathcal D)[\nature^\otimes]\) and
\(\overline\pi\). We prove this by induction on \(t\). The case \(t=0\) is
immediate. For the induction step, under the canonical product adversary,
after an augmented history \(\overline h_t\) with projection \(h_t\), and
after a product action \((a_t,u_t)\), the next concrete state is sampled
from $
\nature_t(h_t,a_t)$.
This distribution depends only on the projected history \(h_t\) and the
base action \(a_t\), not on the hidden augmented history
\(\overline h_t\) nor on the auxiliary action \(u_t\). Therefore, after
conditioning on the projected next history \(h_ta_ts_{t+1}\), the hidden
part is obtained by conditioning \((\overline h_t,u_t)\) on the observed
base action \(a_t\), and then applying the deterministic monitor update.
This is exactly the recursive definition of \(\nu_{t+1}\).

Now define the HR policy \(\overline\pi^\downarrow\) on \(\rmdp\) by
\[
\overline\pi^\downarrow(E\mid h)
=
\int
\overline\pi(E\times U\mid\overline h)\,
\nu_t(d\overline h\mid h),
\qquad h\in H_t.
\]
For every adversary \(\nature\), the kernel \(\nu_t\) is a version of the
conditional law used in the definition of
\(\overline\pi^{\downarrow,\nature}\). Hence, for every \(t\),
\[
\overline\pi^{\downarrow,\nature}(\cdot\mid H_t)
=
\overline\pi^\downarrow(\cdot\mid H_t)
\qquad
\prob_{\rmdp[\nature],\overline\pi^\downarrow}\text{-a.s.}
\]
Therefore, the two policies induce the same state-action path measure under
\(\nature\), \ie
\[
\overline\pi^{\downarrow,\nature}\equiv_\nature\overline\pi^\downarrow,
\]
which concludes the proof. 
\end{proof}

\begin{definition}[Shield soundness, extension of~\cref{def:shield_soundness}]
    Let \(\spec\) be a specification, let \(\rmdp\) be an RMDP equipped with
    a monitor observation map \(\alpha\). A shield
    \(\mathfrak S=(\mathcal D,\Gamma)\) acting on \((\rmdp,\alpha)\) is
    \emph{sound} over \((\rmdp,\alpha)\) for \(\spec\) if, for
    every policy $
    \overline\pi$
    on \(\rmdp\otimes_\alpha\mathcal D\) compliant with \(\mathfrak S\), we have $
    \rmdp,\,\overline\pi^{\downarrow}\models \spec$.
\end{definition}

For two policies \(\pi\) and \(\pi'\) on an RMDP \(\rmdp\), we write $
\pi\equiv\pi'$
if, for every adversary \(\nature\) of \(\rmdp\), $
\pi\equiv_{\nature}\pi'$.

\begin{definition}[Shield completeness]
    Let \(\spec\) be a specification and let \(\rmdp\) be an RMDP equipped with
    a monitor observation map \(\alpha\). Let
    \(\mathcal C\subseteq \mathrm{HR}\) be a class of policies. A shield
    \(\mathfrak S=(\mathcal D,\Gamma)\) acting on \((\rmdp ,\alpha)\) is
    \(\mathcal C\)-\emph{complete} over \((\rmdp ,\alpha)\) for \(\spec\) if,
    for every policy \(\pi\in\mathcal C\) of \(\rmdp \) such that $
    \rmdp ,\pi\models\spec$, there exists an HR policy $
    \overline\pi$
    of \(\rmdp \otimes_\alpha\mathcal D\) compliant with \(\mathfrak S\) such
    that $
    \pi\equiv\overline\pi^{\downarrow}$.
\end{definition}

For any two measures $\mu,\nu$ over a measurable space $(X,\mathcal F)$, we denote by $\mathrm{TV}(\mu,\nu)$ their total variation distance, \ie $\mathrm{TV}(\mu,\nu)
=
\sup_{A\in\mathcal F} |\mu(A)-\nu(A)|$.

For any standard Borel spaces of states $S$ and actions $A$, for any $\gamma\in(0,1)$, we let $\Omega=(SA)^{\omega}$ equipped
with the $\sigma$-algebra $\mathcal F_\Omega$ generated by cylinder sets. The \emph{discounted Hamming
distance} on $\Omega$ is the distance $d_{\gamma}$ such that, for any two words $h=s_0a_0s_1a_1\cdots$ and $h'=s'_0a'_0s'_1a'_1\cdots$, we have
\[
d_\gamma(h,h')
=
(1-\gamma)\sum_{t=0}^{\infty}\gamma^t
\mathbf 1\!\left\{(s_t,a_t)\neq (s_t',a_t')\right\}.
\]

The discounted Hamming
distance is a metric on $\Omega$, and it satisfies $0\le d_\gamma(h,h')\le 1$. Given two probability measures $\mu,\nu$ on $\Omega$, a \emph{coupling} of
$\mu$ and $\nu$ is a probability measure $\xi$ on $\Omega\times\Omega$
whose first marginal is $\mu$ and second marginal is $\nu$. We denote by
$\Gamma(\mu,\nu)$ the set of all such couplings. We define the
\emph{$1$-Wasserstein distance} associated to $d_{\gamma}$ as the distance $W_{1,\gamma}$ on $\Delta(\Omega,\mathcal F_\Omega)$ such that
\[
W_{1,\gamma}(\mu,\nu)
=
\inf_{\xi\in\Gamma(\mu,\nu)}
\int_{\Omega\times\Omega}
d_\gamma(h,h')\,\xi(dh,dh').
\]

\begin{lemma}\label{lemma:dominating distance}
For all probability measures $\mu,\nu$ on $(\Omega,\mathcal F_{\Omega})$, it holds that
\(
W_{1,\gamma}(\mu,\nu)\le \mathrm{TV}(\mu,\nu).
\)
\end{lemma}

\begin{proof}
Recall that the total variation distance admits the coupling
characterization
\[
\mathrm{TV}(\mu,\nu)
=
\inf_{\xi\in\Gamma(\mu,\nu)} \xi(\omega\neq\omega').
\]
A coupling $\xi^\star\in\Gamma(\mu,\nu)$ is called \emph{maximal} if it
attains this infimum, \ie\ if
\(
\xi^\star(\omega\neq\omega')=\mathrm{TV}(\mu,\nu).
\)
Let $\xi\in\Gamma(\mu,\nu)$ be a maximal coupling, so that
\(
\xi\bigl(\{(\omega,\omega')\in\Omega\times\Omega:\omega\neq\omega'\}\bigr)
=
\mathrm{TV}(\mu,\nu).
\)
Then, since $d_\gamma$ is bounded by $1$, we have
\[
\int d_\gamma(\omega,\omega')\,\xi(d\omega,d\omega')
\le
\xi(\omega\neq\omega')
=
\mathrm{TV}(\mu,\nu).
\]
Taking the infimum over all couplings yields
\[
W_{1,\gamma}(\mu,\nu)\le \mathrm{TV}(\mu,\nu).
\]
\end{proof}

For any RMDP \(\rmdp\), and any two policies \(\pi,\pi'\) over \(\rmdp\), we
define
\[
\mathrm{TV}_{\rmdp}(\pi,\pi')
=
\sup_{\nature}
\mathrm{TV}\!\left(
\prob_{\rmdp[\nature],\pi},
\prob_{\rmdp[\nature],\pi'}
\right),
\]
and, for \(\gamma\in(0,1)\),
\[
W_{1,\gamma}^{\rmdp}(\pi,\pi')
=
\sup_{\nature}
W_{1,\gamma}\!\left(
\prob_{\rmdp[\nature],\pi},
\prob_{\rmdp[\nature],\pi'}
\right),
\]
where the suprema range over all adversaries \(\nature\) of \(\rmdp\).

\begin{definition}[Shield approximate completeness]
Let $\spec$ be a specification, let $\rmdp$ be an RMDP equipped with a
monitor observation map $\alpha$, and let $\epsilon>0$. Let
\(\mathcal C\subseteq\mathrm{HR}\) be a class of policies.
A shield $\mathfrak S=(\mathcal D,\Gamma)$ acting on $(\rmdp,\alpha)$ is
$\mathcal C$-$\epsilon$-$\mathrm{TV}$-\emph{complete}
(resp. $\mathcal C$-$\epsilon$-$W_{1,\gamma}$-\emph{complete})
over $(\rmdp,\alpha)$ for $\spec$ if, for every policy $\pi\in\mathcal C$ of $\rmdp$ such
that $\rmdp,\pi\models\spec$, there exists an HR policy $\overline\pi$ of
$\rmdp\otimes_\alpha\mathcal D$ compliant with $\mathfrak S$ such that
\[
\mathrm{TV}_{\rmdp}\!\left(\pi,\overline\pi^{\downarrow}
\right)
\le \epsilon \text{ (resp. }
W^{\rmdp}_{1,\gamma}\!\left(\pi,\overline\pi^{\downarrow}
\right)
\le \epsilon).
\]
\end{definition}

The notion of \(\mathcal C\)-completeness is equivalent to
\(\mathcal C\)-\(0\)-\(\mathrm{TV}\)-completeness, and also to
\(\mathcal C\)-\(0\)-\(W_{1,\gamma}\)-completeness.

\begin{definition}[Shield approximate optimality, extension of~\cref{def:shield_optimality}]
Let $\spec$ be a specification, let $\rmdp$ be an RMDP equipped with a
monitor observation map $\alpha$, let $\epsilon>0$, and let $\gamma\in(0,1]$. Let
\(\mathcal C\subseteq \mathrm{HR}\) be a class of policies.
A shield $\mathfrak S=(\mathcal D,\Gamma)$ acting on $(\rmdp,\alpha)$ is
$\mathcal C$-$(\epsilon,\gamma)$-\emph{optimal} for $\spec$ if, for every
policy $\pi\in\mathcal C$ of $\rmdp$ such that $\rmdp,\pi\models\spec$, there
exists a policy $\overline\pi$ of $\rmdp\otimes_\alpha\mathcal D$
compliant with $\mathfrak S$ such that, for any adversary $\nature$ of $\mathcal M$,
\[
J^\gamma_{\rmdp[\nature]}(\overline\pi^{\downarrow})
\ge
J^\gamma_{\rmdp[\nature]}(\pi)-\epsilon.
\]
\end{definition}

\begin{lemma}
Let $S$ be a set of states, $A$ be a set of actions, $\gamma\in(0,1)$, and let
$\Omega=(SA)^{\omega}$ be equipped with the $\sigma$-algebra
$\mathcal F_\Omega$ generated by cylinder sets. Suppose that
\[
\sup_{(s,a)\in S\times A} |\rewardfunc(s,a)|\le Z.
\]
Then, for all probability measures $\mu,\nu$ on $\Omega$,
\[
\left|
\mathbb E_{s_0a_0\cdots \sim \mu}\left[\sum_{t=0}^{\infty}\gamma^t \rewardfunc(s_t,a_t)\right]
-
\mathbb E_{s_0a_0\cdots \sim \nu}\left[\sum_{t=0}^{\infty}\gamma^t \rewardfunc(s_t,a_t)\right]
\right|
\le
\frac{2Z}{1-\gamma}\,
W_{1,\gamma}(\mu,\nu).
\]

Moreover, suppose that the absolute value of the total undiscounted return is
uniformly bounded by $B$ on $\Omega$, \ie
\[
\left|\sum_{t=0}^{\infty}\rewardfunc(s_t,a_t)\right|\le B
\qquad\text{for all } s_0a_0s_1a_1\cdots\in\Omega.
\]
Then, for all probability measures $\mu,\nu$ on $\Omega$,
\[
\left|
\mathbb E_{s_0a_0\cdots \sim \mu}\left[\sum_{t=0}^{\infty}\rewardfunc(s_t,a_t)\right]
-
\mathbb E_{s_0a_0\cdots \sim \nu}\left[\sum_{t=0}^{\infty}\rewardfunc(s_t,a_t)\right]
\right|
\le
2B\,\mathrm{TV}(\mu,\nu).
\]
\end{lemma}

\begin{proof}
For any $h=s_0a_0s_1a_1\cdots$, let
\[
G_\gamma(h)\coloneqq\sum_{t=0}^{\infty}\gamma^t \rewardfunc(s_t,a_t),
\]
and let
\[
\operatorname{span}(\rewardfunc)\coloneqq
\sup_{(s,a),(s',a')\in S\times A}
|\rewardfunc(s,a)-\rewardfunc(s',a')|.
\]
Since $|\rewardfunc(s,a)|\le Z$ for all $(s,a)$, we have
\[
\operatorname{span}(\rewardfunc)\le 2Z.
\]

Let $\xi\in\Gamma(\mu,\nu)$ be any coupling. Then
\begin{align*}
\left|
\int G_\gamma\,d\mu - \int G_\gamma\,d\nu
\right|
&=
\left|
\int_{\Omega\times\Omega}
\bigl(G_\gamma(h)-G_\gamma(h')\bigr)\,\xi(dh,dh')
\right| \\
&\le
\int_{\Omega\times\Omega}
|G_\gamma(h)-G_\gamma(h')|\,\xi(dh,dh').
\end{align*}
Now, for $h=s_0a_0\cdots$ and $h'=s'_0a'_0\cdots$,
\begin{align*}
|G_\gamma(h)-G_\gamma(h')|
&\le
\sum_{t=0}^{\infty}
\gamma^t |\rewardfunc(s_t,a_t)-\rewardfunc(s_t',a_t')| \\
&\le
\operatorname{span}(\rewardfunc)
\sum_{t=0}^{\infty}
\gamma^t
\mathbf 1\!\left\{(s_t,a_t)\neq (s_t',a_t')\right\} \\
&=
\frac{\operatorname{span}(\rewardfunc)}{1-\gamma}\,
d_\gamma(h,h').
\end{align*}
Hence
\[
\left|
\int G_\gamma\,d\mu - \int G_\gamma\,d\nu
\right|
\le
\frac{\operatorname{span}(\rewardfunc)}{1-\gamma}
\int d_\gamma(h,h')\,\xi(dh,dh').
\]
Since this holds for every coupling $\xi\in\Gamma(\mu,\nu)$, taking the
infimum over $\xi$ yields
\[
\left|
\int G_\gamma\,d\mu - \int G_\gamma\,d\nu
\right|
\le
\frac{\operatorname{span}(\rewardfunc)}{1-\gamma}\,
W_{1,\gamma}(\mu,\nu)
\le
\frac{2Z}{1-\gamma}\,
W_{1,\gamma}(\mu,\nu).
\]

For the undiscounted case, define
\[
G_1(h)\coloneqq\sum_{t=0}^{\infty}\rewardfunc(s_t,a_t).
\]
By assumption, $|G_1(h)|\le B$ for all $h\in\Omega$, hence
\[
\operatorname{span}(G_1)\coloneqq\sup_{h,h'\in\Omega}|G_1(h)-G_1(h')|\le 2B.
\]
Using the standard inequality
\[
\left|\int f\,d\mu-\int f\,d\nu\right|
\le
\operatorname{span}(f)\,\mathrm{TV}(\mu,\nu)
\]
for bounded measurable functions $f$, applied to $f=G_1$, we obtain
\[
\left|
\int G_1\,d\mu-\int G_1\,d\nu
\right|
\le
2B\,\mathrm{TV}(\mu,\nu).
\qedhere
\]
\end{proof}

\begin{lemma}[Approximate completeness implies approximate optimality]\label{lemma:completeness-to-optimality}
Let \(\spec\) be a specification, let \(\rmdp\) be an RMDP equipped with a
monitor observation map \(\alpha\), let
\(\mathfrak S\) be a shield acting on \((\rmdp,\alpha)\), and let
\(\mathcal C\subseteq\mathrm{HR}\).

\begin{enumerate}
    \item Suppose that $
    \sup_{(s,a)\in S\times\mathsf A}|R(s,a)|\le Z$, that $\gamma<1$,
    and that \(\mathfrak S\) is
    \(\mathcal C\)-\(\epsilon\)-\(W_{1,\gamma}\)-complete
    over \((\rmdp,\alpha)\) for \(\spec\). Then \(\mathfrak S\) is
    \[
    \mathcal C\text{-}
    \left(\frac{2Z\epsilon}{1-\gamma},\gamma\right)\text{-optimal}
    \]
    over \((\rmdp,\alpha)\) for \(\spec\).

    \item Suppose that the absolute value of the total undiscounted return
    is uniformly bounded by \(B\), and that \(\mathfrak S\) is
    \(\mathcal C\)-\(\epsilon\)-\(\mathrm{TV}\)-complete
    over \((\rmdp,\alpha)\) for \(\spec\). Then \(\mathfrak S\) is
    \[
    \mathcal C\text{-}(2B\epsilon,1)\text{-optimal}
    \]
    over \((\rmdp,\alpha)\) for \(\spec\).
\end{enumerate}
\end{lemma}

\begin{proof}
We prove the first item; the second is analogous. Let \(\pi\in\mathcal C\)
be such that \(\rmdp,\pi\models\spec\). By
\(\mathcal C\)-\(\epsilon\)-\(W_{1,\gamma}\)-completeness,
there exists a compliant policy \(\overline\pi\) such that,
\[
W_{1,\gamma}\!\left(
\prob_{\rmdp[\nature],\pi},
\prob_{\rmdp[\nature],\overline\pi^{\downarrow}}
\right)
\le \epsilon.
\]
Fix an adversary \(\nature\). Applying the previous lemma to the two path
measures above yields
\[
\left|
J^\gamma_{\rmdp[\nature]}(\pi)
-
J^\gamma_{\rmdp[\nature]}(\overline\pi^{\downarrow})
\right|
\le
\frac{2Z}{1-\gamma}\epsilon.
\]
Hence
\[
J^\gamma_{\rmdp[\nature]}(\overline\pi^{\downarrow})
\ge
J^\gamma_{\rmdp[\nature]}(\pi)
-
\frac{2Z}{1-\gamma}\epsilon.
\]
Since \(\nature\) was arbitrary, \(\mathfrak S\) is
\(\mathcal C\)-\(\left(\frac{2Z}{1-\gamma}\epsilon,\gamma\right)\)-optimal.
\end{proof}

In the rest of the section, we let $\rmdp$ be an RMDP equipped with a \textbf{finite} safety abstraction $\alpha\colon S\to\Xi$, we let $
\tuple{\Xi,\mathsf A,A_\alpha,\mathcal P_\alpha,\xiinit,AP,\labelfunc_\alpha}=\rmdp/\alpha$, we let \(\spec=\mathbb P_{\ge 1-p}(\formula)\), we let
\(\mathcal A=\tuple{Q,2^{AP},\qinit,\delta,F}\) recognize the bad
prefixes of \(\formula\), and $\beta\colon \Xi\times Q\to[0,1]$ be an $((\rmdp/\alpha)\otimes_{\labelfunc_\alpha}\mathcal A,\Xi\times F)$-
inductive value function
such that $
\beta(\xiinit,\qinit)\le p$.

\begin{theorem}[Realizability and soundness, restatement of~\cref{thm:RMDP_soundness}]
    \label{th:soundness-RMDP}
    The shield \(\mathfrak S(\rmdp/\alpha,\mathcal A,\beta)\) acts on
    \((\rmdp,\alpha)\), is realizable over \((\rmdp,\alpha)\), and is
    sound over \((\rmdp,\alpha)\) for \(\spec\).
\end{theorem}

\begin{proof}
Let \(\mathfrak S(\rmdp/\alpha,\mathcal A,\beta)=(\mathcal D,\Gamma)\).
Since \(\alpha\) is an abstraction of \(\rmdp\), we have
\(A(s)=A_\alpha(\alpha(s))\) for all \(s\in S\). Hence, the shield acts on
\((\rmdp,\alpha)\).

We first prove that the shield is nonblocking. Let
\[
D_\Gamma=\{(\xi,q,y):\Gamma(\xi,q,y)\neq\emptyset\}.
\]
We claim that
\[
B_\beta\coloneqq\{(\xi,q,y):\beta(\xi,q)\le y\}\subseteq D_\Gamma.
\]
Indeed, fix \((\xi,q,y)\in B_\beta\). Since \(\Xi\) and \(A_\alpha\) are
finite, the infimum in the Bellman operator is attained. Hence, by
inductiveness of \(\beta\), there exists \(a\in A_\alpha(\xi)\) such that
\[
\sup_{\mu\in\mathcal P_\alpha(\xi,a)}
\mathbb E_{\xi'\sim\mu}
\left[
\beta\bigl(\xi',\delta(q,\labelfunc_\alpha(\xi'))\bigr)
\right]
\le
\beta(\xi,q).
\]
If \(q\in F\), then \(\beta(\xi,q)=1\), and any enabled action satisfies
the required bound since \(\beta\le 1\). Thus
\[
\Dirac{a}\otimes\Dirac{\beta}\in\Gamma(\xi,q,y),
\]
because \(y\ge\beta(\xi,q)\). Therefore \((\xi,q,y)\in D_\Gamma\).

Since \(\beta(\xiinit,\qinit)\le p\), we have
\((\xiinit,\qinit,p)\in D_\Gamma\). Now let
\((\xi,q,y)\in D_\Gamma\) and \(\rho\in\Gamma(\xi,q,y)\). By definition of
the shield, \(\rho\) is supported on
\(A_\alpha(\xi)\times\mathcal V_\beta\). Hence, for \(\rho\)-almost every
\((a,v)\), every \(\mu\in\mathcal P_\alpha(\xi,a)\), and every next
abstract state \(\xi'\), if
\[
q'=\delta(q,\labelfunc_\alpha(\xi')),
\qquad
y'=v(\xi',q'),
\]
then \(y'\ge\beta(\xi',q')\), since \(v\in\mathcal V_\beta\). Thus
\((\xi',q',y')\in B_\beta\subseteq D_\Gamma\). Therefore, every compliant
one-step action distribution keeps the next abstract-monitor state in
\(D_\Gamma\) with probability one under every admissible transition
measure. Hence, the shield is nonblocking. By the previous lemma, it is
realizable.

It remains to prove soundness. Let \(\overline\pi\) be an HR policy on
\(\rmdp\otimes_\alpha\mathcal D\) compliant with the shield, and let
\(\nature\) be an adversary of \(\rmdp\). Consider the product process
under \(\overline\pi\) and the canonical product adversary
\(\nature^\otimes\). Write
\(
\overline S_t=(S_t,Q_t,Y_t),
\)
and let \(\overline{\mathcal F}_t\) be the \(\sigma\)-algebra generated by
the augmented history \(\overline H_t\).

By compliance, almost surely,
\[
\overline\pi(\cdot\mid \overline H_t)
\in
\Gamma(\alpha(S_t),Q_t,Y_t).
\]
For a realized product action \((a,v)\), the adversary chooses a measure
in \(\mathcal P(S_t,a)\), whose pushforward by \(\alpha\) belongs to
\(\mathcal P_\alpha(\alpha(S_t),a)\). Hence, the defining inequality of
\(\Gamma\) gives
\[
\mathbb E[Y_{t+1}\mid\overline{\mathcal F}_t]\le Y_t.
\]
Thus \((Y_t)_{t\in\Nat}\) is a bounded supermartingale.
Let
\(
\tau_F=\inf\{t\ge 0:Q_t\in F\}.
\)
Since \(\beta=1\) on \(\Xi\times F\), and every
\(v\in\mathcal V_\beta\) satisfies \(v(\xi,q)\ge\beta(\xi,q)\), every
reachable accepting monitor state satisfies \(Q_t\in F\Rightarrow Y_t=1\).
Therefore, for every \(n\in\Nat\),
\[
\mathbf 1_{\{\tau_F\le n\}}\le Y_{\tau_F\wedge n}.
\]
By optional stopping for bounded supermartingales,
\[
\Pr(\tau_F\le n)
\le
\mathbb E[Y_{\tau_F\wedge n}]
\le
Y_0
=
p.
\]
Letting \(n\to\infty\), we obtain
\[
\Pr(\tau_F<\infty)\le p.
\]

The DFA \(\mathcal A\) recognizes the bad prefixes of \(\formula\), so
\(\tau_F<\infty\) is exactly the event that \(\formula\) is violated.
Moreover, by construction of the projected policy, the law of paths in
\(\rmdp[\nature]\) under \(\overline\pi^\downarrow\) is the projection of
the product path law. Since the product preserves labels, the satisfaction
probability is the same under projection. Hence
\[
\rmdp[\nature],\overline\pi^\downarrow
\models
\mathbb P_{\ge 1-p}(\formula)
=
\spec.
\]
Since \(\nature\) was arbitrary, we have
\[
\rmdp,\overline\pi^\downarrow\models\spec.
\]
Thus, the shield is sound.
\end{proof}

We proceed with proving the optimality results.
In what follows, we let $\beta^{\infty}$ be the lowest fixed point of
$\mathcal B_{(\rmdp/\alpha)\otimes_{\labelfunc_{\alpha}}\mathcal A}^{\Xi\times F}$.
Furthermore, we say that an RMDP $\rmdp$ is \emph{conditionally deterministic} w.r.t an abstraction $\alpha$ if every reachable abstract
history determines a unique concrete history. More precisely, for every abstract history $
    \widehat h_t=\xi_0a_0\xi_1a_1\cdots a_{t-1}\xi_t$
with \(\xi_0=\alpha(\sinit)\), there is at most one concrete history $h_t=s_0a_0s_1a_1\cdots a_{t-1}s_t$
such that \(s_0=\sinit\), \(\alpha(s_i)=\xi_i\) for all \(i\le t\), and
\(h_t\) has positive probability under some adversary of \(\rmdp\).
When such a concrete history exists, we denote its last state by $\chi(\widehat h_t)=s_t$.

\begin{theorem}
Suppose that $\beta(\xiinit,\qinit)<p$, and that $\rmdp$ is conditionally deterministic w.r.t. $\alpha$. Then the shield
$\mathfrak S(\rmdp/\alpha,\mathcal A,\beta)$ acts on
$(\rmdp,\alpha)$, and is $\mathrm{HR}$-$\epsilon$-TV-complete over $(\rmdp,\alpha)$
for $\spec$, where
\[
\epsilon=
\frac{\infnorm{\beta-\beta^{\infty}}}
{\infnorm{\beta-\beta^{\infty}}+p-\beta(\xiinit,\qinit)}.
\]
In particular, the shield
$\mathfrak S(\rmdp/\alpha,\mathcal A,\beta^{\infty})$
acts on $(\rmdp,\alpha)$ and is $\mathrm{HR}$-complete over $(\rmdp,\alpha)$ for~$\spec$.
\end{theorem}\label{th:RMDP-TV-completeness}
\begin{proof}
Since \(\alpha\) is an abstraction of \(\rmdp\), we have
\(A(s)=A_\alpha(\alpha(s))\) for all \(s\in S\), so the shield acts on
\((\rmdp,\alpha)\).

Fix an HR policy \(\pi\) of \(\rmdp\) such that
\(\rmdp,\pi\models\spec\). For a history \(\widetilde h\) of
\(\rmdp\otimes_{\labelfunc}\mathcal A\), let $
V_\pi(\widetilde h)$
denote the supremal probability, over all adversaries, of eventually
reaching \(F\) when continuing from \(\widetilde h\) under \(\pi\). Since
\(\pi\) satisfies \(\spec=\mathbb P_{\ge 1-p}(\formula)\), we have
\(V_\pi((\sinit,\qinit))\le p\).

Define $
\delta=\|\beta-\beta^\infty\|_\infty$ and $
V_\pi^\delta(\widetilde h)
=
\min\{1,V_\pi(\widetilde h)+\delta\}$. 
For \(a\in A(s)\) and
\(\xi'\in\Xi\), we write $
q_{\xi'}=\delta(q,\labelfunc_\alpha(\xi'))$.
We define \(z_{\widetilde h,a}\in[0,1]^{\Xi\times Q}\) as follows. If, for any history $\widetilde h=(s_0,q_0) a_0\cdots (s_n,q_n)$, denoting $\widehat h=\alpha(s_0)a_0\cdots \alpha(s_n)$, the
history \(\widehat{h}a\xi'\) is reachable in \(\rmdp/\alpha\), then we set
\[
z_{\widetilde h,a}(\xi',q_{\xi'})
=
V_\pi^\delta\!\left(
    \widetilde h a\left(\chi\left({\widehat h}a\xi'\right),q_{\xi'}\right)
\right).
\]
On entries \((\xi',q')\) with \(q'\neq q_{\xi'}\), and on unreachable
abstract successors, we set
\[
z_{\widetilde h,a}(\xi',q')
=
\min\{1,\beta^\infty(\xi',q')+\delta\}.
\]

This vector belongs to \(\mathcal V_\beta\). Indeed, for unreachable
entries and for entries not corresponding to the automaton successor, this
follows immediately from
\[
\beta(\xi',q')\le \beta^\infty(\xi',q')+\delta
\le \min\{1,\beta^\infty(\xi',q')+\delta\}.
\]
For the other entries, we have
\[
\beta^\infty(\xi',q_{\xi'})
\le
V_\pi\!\left(
    \widetilde h a\left(\chi\left({\widehat h} a\xi'\right),q_{\xi'}\right)
\right),
\]
and therefore
\[
\beta(\xi',q_{\xi'})
\le
\min\{1,\beta^\infty(\xi',q_{\xi'})+\delta\}
\le
z_{\widetilde h,a}(\xi',q_{\xi'}).
\]
Hence \(z_{\widetilde h,a}\in\mathcal V_\beta\).

Define an HR policy \(\overline\pi^0\) on
\(\rmdp\otimes_\alpha\mathcal D\) as follows. For an augmented history
\(\overline h\), let \(\widetilde h\) be its associated history in
\(\rmdp\otimes\mathcal A\). Then \(\overline\pi^0(\cdot\mid\overline h)\)
is the pushforward of \(\pi(\cdot\mid\overline h^\downarrow)\) by $
a\mapsto (a,z_{\widetilde h,a})$.
Thus, if \(\overline\pi^0\) is followed at all histories, then
\[
(\overline\pi^0)^\downarrow\equiv\pi.
\]

We also define a safe HR policy \(\overline\pi_{\safe}\). For every
\((\xi,q)\), choose a distribution \(\mu_{\xi,q}\) over
\(A_\alpha(\xi)\) witnessing the inductiveness of \(\beta\), namely
\[
\mathbb E_{a\sim\mu_{\xi,q}}
\left[
\sup_{\lambda\in\mathcal P_\alpha(\xi,a)}
\mathbb E_{\xi'\sim\lambda}
\left[
\beta\bigl(\xi',\delta(q,\labelfunc_\alpha(\xi'))\bigr)
\right]
\right]
\le
\beta(\xi,q).
\]
For an augmented history ending in \((s,q,y)\), set
\[
\overline\pi_{\safe}(\cdot\mid\overline h)
=
\mu_{\alpha(s),q}\otimes\Dirac{\beta}.
\]
Whenever \(y\ge\beta(\alpha(s),q)\), this distribution belongs to
\(\Gamma(\alpha(s),q,y)\).
Let
\(
\epsilon=
\frac{\delta}{\delta+p-\beta(\xiinit,\qinit)}.
\)
and define the shielded HR policy \(\overline\pi^\epsilon\) as follows:
\begin{itemize}
    \item At the initial
    history, set
    \(
    \overline\pi^\epsilon
    =
    (1-\epsilon)\overline\pi^0+\epsilon\overline\pi_{\safe}.
    \)
    \item At any non-initial augmented history \(\overline h\) ending in
    \((s,q,y)\), with projection \(\widetilde h\), set
    \[
    \overline\pi^\epsilon(\cdot\mid\overline h)
    =
    \begin{cases}
    \overline\pi^0(\cdot\mid\overline h),
    &\text{if } y\ge V_\pi^\delta(\widetilde h),\\[1mm]
    \overline\pi_{\safe}(\cdot\mid\overline h),
    &\text{otherwise.}
    \end{cases}
    \]
\end{itemize}
We show that \(\overline\pi^\epsilon\) is compliant. At the initial state,
the expected next \(y\)-value under \(\overline\pi^0\) is at most
\(p+\delta\), while under \(\overline\pi_{\safe}\) it is at most
\(\beta(\xiinit,\qinit)\). Hence
\(
(1-\epsilon)(p+\delta)
+\epsilon\beta(\xiinit,\qinit)
=
p,
\)
so the initial action distribution is compliant.

Moreover, along every reachable history we have
\(y\ge\beta(\alpha(s),q)\), because both auxiliary vectors
\(z_{\widetilde h,a}\) and \(\beta\) belong to \(\mathcal V_\beta\). Thus,
if the policy chooses \(\overline\pi_{\safe}\), it is compliant. 
If the policy chooses \(\overline\pi^0\), then by definition
\(y\ge V_\pi^\delta(\widetilde h)\). We show that the shield constraint is
satisfied. Let \(\widehat h\) be the abstract history of $\rmdp/\alpha$ associated to
\(\widetilde h\), and let \(\xi=\alpha(s)\). For a chosen base action
\(a\), let \(\Xi_{t+1}\) denote the next abstract state. Under the
augmented action \((a,z_{\widetilde h,a})\), the next budget coordinate is
\[
Y_{t+1}
=
z_{\widetilde h,a}
\bigl(
\Xi_{t+1},
\delta(q,\labelfunc_\alpha(\Xi_{t+1}))
\bigr).
\]
Fix \(a\in A(s)\) and
\(\lambda\in\mathcal P_\alpha(\xi,a)\). By the definition of the quotient
RMDP, there exists \(\mu\in\mathcal P(s,a)\) such that
\(\lambda=\mu\circ\alpha^{-1}\). Moreover, since \(\rmdp\) is
conditionally deterministic with respect to \(\alpha\), for every abstract
successor \(\xi'\) in the support of \(\lambda\), the corresponding
concrete successor is uniquely given by
\(\chi(\widehat h a\xi')\). Therefore,
\[
\mathbb E_{\xi'\sim\lambda}
\left[
z_{\widetilde h,a}\bigl(\xi',
\delta(q,\labelfunc_\alpha(\xi'))\bigr)
\right]
\le
\delta
+
\mathbb E_{s'\sim\mu}
\left[
V_\pi\bigl(\widetilde h a(s',
\delta(q,\labelfunc(s')))\bigr)
\right].
\]
Taking the supremum over
\(\lambda\in\mathcal P_\alpha(\xi,a)\), we obtain
\[
\sup_{\lambda\in\mathcal P_\alpha(\xi,a)}
\mathbb E_{\xi'\sim\lambda}
\left[
z_{\widetilde h,a}\bigl(\xi',
\delta(q,\labelfunc_\alpha(\xi'))\bigr)
\right]
\le
\delta
+
\sup_{\mu\in\mathcal P(s,a)}
\mathbb E_{s'\sim\mu}
\left[
V_\pi\bigl(\widetilde h a(s',
\delta(q,\labelfunc(s')))\bigr)
\right].
\]
Averaging over
\(a\sim\pi(\cdot\mid\widetilde h^\downarrow)\) and using the dynamic
programming inequality for \(V_\pi\), we get
\[
\mathbb E[Y_{t+1}\mid\overline h,\overline\pi^0]
\le
V_\pi(\widetilde h)+\delta.
\]
Since \(Y_{t+1}\in[0,1]\), this also implies
\[
\mathbb E[Y_{t+1}\mid\overline h,\overline\pi^0]
\le
\min\{1,V_\pi(\widetilde h)+\delta\}
=
V_\pi^\delta(\widetilde h)
\le y.
\]
Hence the shield constraint is satisfied.
Hence \(\overline\pi^\epsilon(\cdot\mid\overline h)\in
\Gamma(\alpha(s),q,y)\). Therefore \(\overline\pi^\epsilon\) is compliant.

It remains to bound the TV distance. Fix an adversary \(\nature\) of
\(\rmdp\). Couple the run of \(\pi\) and the projected run of
\(\overline\pi^\epsilon\) as follows. At the initial decision, with
probability \(1-\epsilon\), use the \(\overline\pi^0\)-component and
couple the base action with the action sampled by \(\pi\); with probability
\(\epsilon\), use the safe component. 
On the first event, the projected history agrees with the run of $\pi$,
the invariant $
y\ge V_\pi^\delta(\widetilde h)$
ensures that \(\overline\pi^\epsilon\) keeps choosing
\(\overline\pi^0\) forever. Indeed, whenever \(\overline\pi^0\) is used
and the next abstract successor is \(\xi'\), the next budget is
\[
y'
=
z_{\widetilde h,a}\bigl(\xi',
\delta(q,\labelfunc_\alpha(\xi'))\bigr)
=
V_\pi^\delta\!\left(
\widetilde h a(\chi(\widehat h a\xi'),
\delta(q,\labelfunc_\alpha(\xi')))
\right),
\]
so the invariant is preserved along the coupled run.
Thus
the two projected paths agree on an event of probability at least
\(1-\epsilon\). Consequently,
\[
\mathrm{TV}\!\left(
\prob_{\rmdp[\nature],\pi},
\prob_{\rmdp[\nature],(\overline\pi^\epsilon)^\downarrow}
\right)
\le \epsilon .
\]
Since \(\nature\) was arbitrary, we obtain
\[
\mathrm{TV}_{\rmdp}\!\left(
\pi,(\overline\pi^\epsilon)^\downarrow
\right)
\le \epsilon.
\]
Therefore, the shield is
\(\mathrm{HR}\)-\(\epsilon\)-TV-complete.

If \(\beta=\beta^\infty\), then \(\delta=0\), hence \(\epsilon=0\). The
same construction gives a compliant policy \(\overline\pi^0\) whose
projection is equivalent to \(\pi\). Therefore, the shield is
\(\mathrm{HR}\)-complete.
\end{proof}

Notice that by Lemma \ref{lemma:dominating distance}, the above theorem implies that the shield $\mathfrak S(\rmdp/\alpha,\mathcal A,\beta)$ is \(\mathrm{HR}\)-$\epsilon$-$W_{1,\gamma}$-complete over $(\rmdp,\alpha)$ for $\spec$ as well.

By Lemma \ref{lemma:completeness-to-optimality}, we have the following.
\begin{theorem}[Approximate optimality of the shield, extension of~\cref{thm:RMDP_optimality}]
    Suppose that \(\beta(\xiinit,\qinit)<p\), that $\rmdp$ is conditionally deterministic w.r.t. $\alpha$, and let
    \[
    \epsilon=
    \frac{\|\beta-\beta^{\infty}\|_\infty}
    {\|\beta-\beta^{\infty}\|_\infty+p-\beta(\xiinit,\qinit)}.
    \]
    Then the two following properties hold.
    \begin{enumerate}
        \item If the absolute value of the total undiscounted return is
        uniformly bounded by \(B\), then
        \(\mathfrak S(\rmdp/\alpha,\mathcal A,\beta)\) is
        \[
        \mathrm{HR}\text{-}(2B\epsilon,1)\text{-optimal}
        \]
        over \((\rmdp,\alpha)\) for \(\spec\).
    
        \item If \(\gamma\in(0,1)\) and $
        \sup_{(s,a)\in S\times\mathsf A}|R(s,a)|\le Z$,
        then \(\mathfrak S(\rmdp/\alpha,\mathcal A,\beta)\) is
        \[
        \mathrm{HR}\text{-}
        \left(\frac{2Z\epsilon}{1-\gamma},\gamma\right)\text{-optimal}
        \]
        over \((\rmdp,\alpha)\) for \(\spec\).
    \end{enumerate}
    In particular, the shield
    \(\mathfrak S(\rmdp/\alpha,\mathcal A,\beta^\infty)\) is
    \(\mathrm{HR}\)-\((0,\gamma)\)-optimal for every \(\gamma\in(0,1]\)
    for which the corresponding return is well-defined.
\end{theorem}

\section{Probabilistic Shielding for Unknown MDPs}

In this appendix, we present all extended definitions, results, and proofs for our shielding approach for unknown MDPs, learned as RMDP from~\cref{sec:unknown_MDPs}.

Throughout, we suppose that $\mdp=\tuple{S,A,P,\sinit,AP,\labelfunc,\rewardfunc}$ is an MDP equipped with a \textbf{finite} abstraction $\alpha\colon S\to\Xi$, we let $\mdp/\alpha = \tuple{\Xi,A_\alpha,P_\alpha,\xiinit,AP,\labelfunc_\alpha}$ by its quotient, we suppose that $\spec=\mathbb P_{\geq 1-p}(\formula)$ is a probabilistic safety LTL formula, that $\mathcal A=\tuple{Q,2^{AP},\qinit,\delta,F}$ is a DFA recognizing the bad prefixes of $\formula$, and that $\widehat{\mdp}=(\Xi,A_\alpha,\widehat{\mathcal P},\xiinit,AP,\labelfunc_\alpha,\rewardfunc_\alpha)$ is an RMDP with $P_\alpha\in \widehat{\mathcal P}$. Furthermore, we suppose that $\widehat \beta$ is an $(\widehat{\mdp}\otimes_{\labelfunc_\alpha}\mathcal A,\Xi\times F)$-inductive value function such that $\widehat\beta(\xiinit,\qinit)\leq p$.

\subsection{Soundness of the Shield for Unknown MDPs}
We first prove the soundness result from \cref{thm:unknown_MDP_soundness}.

\begin{definition}[Conditional concrete lift of the learned RMDP]
For every \(\xi\in\Xi\), let \(C_\xi=\alpha^{-1}(\xi)\). Since \(\Xi\)
is finite and \(\alpha\) is surjective, fix one representative
\(r_\xi\in C_\xi\).

For every \((s,a)\in\operatorname{Gr}(A)\) and every \(\xi'\in\Xi\), define
a probability measure \(K_{s,a}^{\xi'}\in\distr{S}\) by
\[
K_{s,a}^{\xi'}(B)
=
\begin{cases}
\displaystyle
\frac{P(s,a)(B\cap C_{\xi'})}
{P_\alpha(\alpha(s),a)(\xi')},
&
\text{if }P_\alpha(\alpha(s),a)(\xi')>0,\\[3mm]
\Dirac{r_{\xi'}}(B),
&
\text{if }P_\alpha(\alpha(s),a)(\xi')=0,
\end{cases}
\]
for every measurable \(B\subseteq S\). Thus \(K_{s,a}^{\xi'}\) is the
conditional distribution of the next concrete state inside the abstraction
cell \(C_{\xi'}\), with an arbitrary choice when that cell has zero
probability under the true transition.

For \(\widehat\mu\in\distr{\Xi}\), define
\[
\Lambda_{s,a}(\widehat\mu)
=
\sum_{\xi'\in\Xi}
\widehat\mu(\xi')K_{s,a}^{\xi'}.
\]
The \emph{conditional concrete lift} of \(\widehat{\mdp}\) is the RMDP
\(
\widetilde{\mdp}
=
\tuple{S,A,\widetilde{\mathcal P},\sinit,AP,\labelfunc,\rewardfunc},
\)
where
\[
\widetilde{\mathcal P}(s,a)
=
\left\{
\Lambda_{s,a}(\widehat\mu):
\widehat\mu\in\widehat{\mathcal P}(\alpha(s),a)
\right\}.
\]
\end{definition}

\begin{lemma}\label{lemma:lifted-RMDP}
    The map \(\alpha:S\to\Xi\) is an abstraction of
    \(\widetilde{\mdp}\), and
    \(
    \widetilde{\mdp}/\alpha=\widehat{\mdp}.
    \)
    Moreover, there exists an adversary \(\nature^P\) of
    \(\widetilde{\mdp}\) such that
    \(
    \mdp=\widetilde{\mdp}[\nature^P].
    \)
\end{lemma}

\begin{proof}
The action and label conditions follow from the fact that \(\alpha\) is an
abstraction of the true MDP.

We first identify the quotient transition correspondence. Fix
\((s,a)\in\operatorname{Gr}(A)\), and write \(\xi=\alpha(s)\). For every
\(\widehat\mu\in\widehat{\mathcal P}(\xi,a)\), we have
\[
\Lambda_{s,a}(\widehat\mu)\circ\alpha^{-1}
=
\widehat\mu.
\]
Indeed, each \(K_{s,a}^{\xi'}\) is supported on \(C_{\xi'}\), and therefore
\[
K_{s,a}^{\xi'}(\alpha^{-1}(\xi''))
=
\mathbf 1_{\{\xi'=\xi''\}},
\]
which implies that
\[
\{\mu\circ\alpha^{-1}:\mu\in\widetilde{\mathcal P}(s,a)\}
=
\widehat{\mathcal P}(\alpha(s),a).
\]
This set depends only on \(\alpha(s)\) and \(a\). Therefore \(\alpha\) is
an abstraction of \(\widetilde{\mdp}\), and its quotient transition
correspondence is exactly \(\widehat{\mathcal P}\). Thus
\[
\widetilde{\mdp}/\alpha=\widehat{\mdp}.
\]

It remains to show that the true MDP is induced in
\(\widetilde{\mdp}\) by an adversary. For every \((s,a)\), we have
\[
P(s,a)
=
\Lambda_{s,a}\bigl(P_\alpha(\alpha(s),a)\bigr).
\]
Indeed, for every measurable \(B\subseteq S\),
\[
\Lambda_{s,a}\bigl(P_\alpha(\alpha(s),a)\bigr)(B)
=
\sum_{\xi'\in\Xi}
P_\alpha(\alpha(s),a)(\xi')K_{s,a}^{\xi'}(B).
\]
The terms with \(P_\alpha(\alpha(s),a)(\xi')=0\) vanish, and the remaining
terms give
\[
\sum_{\xi'\in\Xi}P(s,a)(B\cap C_{\xi'})
=
P(s,a)(B).
\]
Since
\(
P_\alpha(\alpha(s),a)\in\widehat{\mathcal P}(\alpha(s),a),
\)
we obtain
\(
P(s,a)\in\widetilde{\mathcal P}(s,a).
\)
Thus \(\mdp\) is induced by the adversary \(\nature^P\) that always selects
the transition \(P(s,a)\).

Finally, the measurability of the lifted uncertainty correspondence follows from
the measurability of \(\widehat{\mathcal P}\) and of the map
\((s,a,\widehat\mu)\mapsto\Lambda_{s,a}(\widehat\mu)\). Since
\(\Lambda_{s,a}(\widehat\mu)\circ\alpha^{-1}=\widehat\mu\), the relevant
graph map is injective, so the image of the measurable graph of
\(\widehat{\mathcal P}\) is measurable.
\end{proof}

\begin{theorem}[Restatement of~\cref{thm:unknown_MDP_soundness}]
The shield \(\mathfrak S(\widehat{\mdp},\mathcal A,\widehat\beta)\)
acts on \((\mdp,\alpha)\), is nonblocking over \((\mdp,\alpha)\), and is
sound over \((\mdp,\alpha)\) for \(\spec\).
\end{theorem}

\begin{proof}
Let \(\widetilde{\mdp}\) be the concrete lift defined above. By the
Lemma \ref{lemma:lifted-RMDP}, $\widetilde{\mdp}/\alpha=\widehat{\mdp}$.
Therefore
\[
\mathfrak S(\widehat{\mdp},\mathcal A,\widehat\beta)
=
\mathfrak S(\widetilde{\mdp}/\alpha,\mathcal A,\widehat\beta).
\]
Applying Theorem~\ref{th:soundness-RMDP} to
\(\widetilde{\mdp}\), the shield acts on
\((\widetilde{\mdp},\alpha)\), is nonblocking, and is sound over
\((\widetilde{\mdp},\alpha)\).

Since the true MDP \(\mdp\) is induced by an adversary in \(\widetilde{\mdp}\), the
non-blocking condition for all transitions in
\(\widetilde{\mathcal P}\) implies the non-blocking condition for the true
transition \(P\). Hence the shield is non-blocking over \((\mdp,\alpha)\).

For soundness, let \(\overline\pi\) be a compliant policy on
\(\mdp\otimes_\alpha\mathcal D\). We redefine \(\overline\pi\) arbitrarily on histories that are null
under the true MDP, choosing a measurable compliant selector on
\(D_\Gamma\). This modification does not change the induced path law under
the true MDP, and it is compliant for the lifted RMDP. Applying soundness
over \(\widetilde{\mdp}\) to this modified policy yields the desired
soundness for \(\mdp\).
\end{proof}

\subsection{$\mathrm{TV}$-Completeness of the Shield for Unknown and Finite MDPs}
In the following, we use policies with randomized memory. Formally, a
\emph{memory-based randomized policy} (MBR policy) on an RMDP \(\rmdp\)
is specified by a standard Borel memory space \(\mathsf Z\), an initial
memory distribution \(\lambda\in\distr{\mathsf Z}\), a measurable
action-selection kernel
\[
\sigma\colon S\times\mathsf Z\to\distr{\mathsf A},
\]
and a measurable memory-update kernel
\[
\upsilon\colon \mathsf Z\times S\times\mathsf A\times S\to\distr{\mathsf Z},
\]
such that \(\sigma(s,z)(A(s))=1\) for all \((s,z)\in S\times\mathsf Z\).
Such a policy is executed as follows. First, an initial memory state
\(z_0\sim\lambda\) is sampled. At time \(t\), after observing the current
state \(s_t\) and holding memory \(z_t\), the controller samples
\[
a_t\sim\sigma(s_t,z_t).
\]
After the next state \(s_{t+1}\) is generated by the measure selected by
the adversary, the memory is updated according to
\[
z_{t+1}\sim \upsilon(\cdot\mid z_t,s_t,a_t,s_{t+1}).
\]

Although MBR policies are useful as compact representations, in our
setting, they do not induce more state-action path laws than HR policies.
Indeed, every MBR policy admits an adversary-uniform HR realization.

\begin{lemma}[HR realization of MBR policies]\label{lemma:MBR-to-HR}
Let \(\rmdp\) be an RMDP whose adversaries observe only the public
state-action history. For every MBR policy \(\Pi\) on \(\rmdp\), there
exists an HR policy \(\Pi^\flat\) such that, for every adversary
\(\nature\) of \(\rmdp\),
\[
\prob_{\rmdp[\nature],\Pi}
=
\prob_{\rmdp[\nature],\Pi^\flat}.
\]
\end{lemma}

\begin{proof}
Let \(\Pi=(\mathsf Z,\lambda,\sigma,\upsilon)\). We construct an HR policy
\(\Pi^\flat\) that keeps, as a function of the public history, the belief
over the private memory of \(\Pi\).

For a history \(h_t=s_0a_0\cdots a_{t-1}s_t\), let
\(b_t(\cdot\mid h_t)\) be a probability measure on \(\mathsf Z\). At time
\(0\), set
\[
b_0(\cdot\mid s_0)=\lambda.
\]
Given \(b_t(\cdot\mid h_t)\), define
\[
\Pi^\flat(E\mid h_t)
=
\int_{\mathsf Z}\sigma(s_t,z)(E)\,b_t(dz\mid h_t),
\qquad E\subseteq\mathsf A.
\]

It remains to define \(b_{t+1}\). Fix \(h_t\) and a possible next state
\(s_{t+1}\). Define a probability measure \(K_{h_t,s_{t+1}}\) on
\(\mathsf A\times\mathsf Z\) by
\[
K_{h_t,s_{t+1}}(E\times B)
=
\int_{\mathsf Z}
\int_E
\upsilon(B\mid z,s_t,a,s_{t+1})\,
\sigma(da\mid s_t,z)\,
b_t(dz\mid h_t),
\]
for measurable \(E\subseteq\mathsf A\) and \(B\subseteq\mathsf Z\).
Since the spaces are standard Borel, this measure admits a regular
conditional distribution of the next memory \(z_{t+1}\) given the observed
action \(a_t\). Choose one such version and denote it by
\[
b_{t+1}(\cdot\mid h_ta_ts_{t+1}).
\]
On histories where this conditional distribution is not uniquely
determined, any version may be chosen.

We claim that, for every adversary \(\nature\), \(b_t(\cdot\mid h_t)\) is
a version of the conditional law of the private memory \(Z_t\) given the
public history \(H_t=h_t\) under the process induced by \((\Pi,\nature)\).
The proof is by induction on \(t\). The claim is immediate at \(t=0\).
Assume it holds at time \(t\). Given \(H_t=h_t\), the MBR policy samples
\(a_t\) according to
\[
\int_{\mathsf Z}\sigma(s_t,z)(\cdot)\,b_t(dz\mid h_t).
\]
The adversary then samples \(s_{t+1}\) according to
\(\nature_t(h_t,a_t)\), which depends only on the public history and the
chosen action, not on the private memory. Therefore, after conditioning on
the public next history \(h_ta_ts_{t+1}\), the conditional law of
\(Z_{t+1}\) is exactly the regular conditional distribution used above to
define \(b_{t+1}(\cdot\mid h_ta_ts_{t+1})\). This proves the induction.

Consequently, under every adversary \(\nature\), after every public
history \(h_t\), the action distribution of \(\Pi^\flat\) equals the
conditional action distribution of \(\Pi\) after marginalizing its private
memory. Since both policies then face the same adversarial transition
kernel, they induce the same measures on all finite cylinders, and hence
on infinite state-action paths:
\[
\prob_{\rmdp[\nature],\Pi}
=
\prob_{\rmdp[\nature],\Pi^\flat}.
\]
\end{proof}

Bernoulli mixtures of policies are a special case of MBR policies. Let
\(\Pi^0\) and \(\Pi^1\) be two policies on the same RMDP, and let
\(\varepsilon\in[0,1]\). We write
\[
(1-\varepsilon)\Pi^0\oplus\varepsilon\Pi^1
\]
for the policy that first samples \(Z\in\{0,1\}\) with $
\Pr(Z=1)=\varepsilon$,
and then follows \(\Pi^Z\) for the whole execution. We call
\(\oplus\) an \emph{external mixture}. It should not be confused with the combination of policies that, at each history, independently
randomizes between the action distributions of \(\Pi^0\) and \(\Pi^1\).

We now define the policy class used in the following completeness theorem.
Throughout this paragraph, \(\mdp\) is finite. Let
\(\mathfrak D_{\mathcal A}(\mdp)\) be the finite set of maps $
d\colon S\times Q\to\mathsf A$
such that \(d(s,q)\in A(s)\) for every \((s,q)\in S\times Q\).

For every \(\lambda\in\distr{\mathfrak D_{\mathcal A}(\mdp)}\), let
\(\Pi_\lambda\) be the MBR policy:
\begin{itemize}
    \item memory space
    \(
    \mathsf Z=\mathfrak D_{\mathcal A}(\mdp)\times Q,
    \)
    \item initial memory distribution
    \(
    \lambda\otimes\Dirac{\qinit},
    \)
    \item action-selection kernel
    \(
    \sigma(s,(d,q))=\Dirac{d(s,q)},
    \) and
    \item memory-update kernel
    \(
    \upsilon(\cdot\mid(d,q),s,a,s')
    =
    \Dirac{(d,\delta(q,\labelfunc(s')))}.
    \)
\end{itemize}
Thus \(\Pi_\lambda\) first samples a deterministic map \(d\), keeps it
fixed for the whole execution, and uses the DFA state as finite memory.
Furthermore, we define
\[
\mathcal C_{\mathcal A}(\mdp)
=
\left\{
\Pi_\lambda^\flat:
\lambda\in\distr{\mathfrak D_{\mathcal A}(\mdp)}
\right\},
\]
where \(\Pi_\lambda^\flat\) denotes the HR realization of
\(\Pi_\lambda\) from Lemma~\ref{lemma:MBR-to-HR}.

\textbf{In the rest of this section, we assume that the MDP $\mdp$ is finite, and that \(\widehat{\mdp}\) is graph-preserving}. We define the minimum transition probability parameter $\underline p_{\min}$ as 
\[
\underline p_{\min}
\coloneqq
\min\left\{
\inf_{\mu\in\widetilde{\mathcal P}(s,a)}\mu(s')
\;\middle|\;
s\in S,\ a\in A(s),\ P(s,a,s')>0
\right\}.
\]
Furthermore, we define $
N=|S| \cdot |Q|$, $
H_{\max}=\frac{N}{\underline p_{\min}^{\,N}}$, and 
$$\eta=\sup_{(\xi,a)\in \Xi\times A_\alpha(\xi)} \sup_{\widehat P(\xi,a)\in \widehat{\mathcal P}(\xi,a)} \mathrm{TV}(P_\alpha(\xi,a),\widehat{P}(\xi,a)).$$

\begin{lemma}[Concrete one-step perturbation]
For every \((s,a)\in\operatorname{Gr}(A)\) and every
\(\mu\in\widetilde{\mathcal P}(s,a)\),
\[
\mathrm{TV}\bigl(P(s,a),\mu\bigr)\le \eta.
\]
\end{lemma}

\begin{proof}
By definition of the conditional lift, there exists
\(\widehat\mu\in\widehat{\mathcal P}(\alpha(s),a)\) such that
\[
\mu=\Lambda_{s,a}(\widehat\mu)
=
\sum_{\xi'\in\Xi}
\widehat\mu(\xi')K_{s,a}^{\xi'}.
\]
Moreover,
\[
P(s,a)
=
\Lambda_{s,a}\bigl(P_\alpha(\alpha(s),a)\bigr)
=
\sum_{\xi'\in\Xi}
P_\alpha(\alpha(s),a)(\xi')K_{s,a}^{\xi'}.
\]
The kernels \(K_{s,a}^{\xi'}\) are supported on the disjoint cells
\(C_{\xi'}=\alpha^{-1}(\xi')\). Hence, the total variation distance between
the two concrete mixtures is exactly the total variation distance between
their cell-mass distributions:
\[
\mathrm{TV}\bigl(P(s,a),\mu\bigr)
=
\mathrm{TV}\bigl(P_\alpha(\alpha(s),a),\widehat\mu\bigr).
\]
By definition of \(\eta\), the right-hand side is at most \(\eta\).
\end{proof}

\begin{lemma}[Concrete perturbation bound]
Let \(\pi\in\mathcal C_{\mathcal A}(\mdp)\). Then
\[
\inf_{\widetilde\nature}
\prob_{\widetilde{\mdp}[\widetilde\nature],\pi}(\formula)
\ge
\prob_{\mdp,\pi}(\formula)-\eta H_{\max},
\]
where the infimum ranges over adversaries of \(\widetilde{\mdp}\).
\end{lemma}
\begin{proof}
First fix a deterministic \(\mathcal A\)-memory policy
\(d\colon S\times Q\to\mathsf A\), and let \(C_d\) be the union of the bottom
strongly connected components of the support graph of the concrete product
chain induced by \(d\) on \(S\times Q\) that do not intersect \(S\times F\).
Let
\[
\tau_d
=
\inf\{t\ge 0:(S_t,Q_t)\in C_d\cup(S\times F)\}.
\]

Because the conditional lift is graph-preserving, every adversary of
\(\widetilde{\mdp}\) induces the same support graph on \(S\times Q\) under
\(d\). From any state outside \(C_d\cup(S\times F)\), there is a path of
length at most \(N=|S|\cdot |Q|\) to \(C_d\cup(S\times F)\). Every edge on such
a path has probability at least \(\underline p_{\min}\), under every
adversary and after every history. Hence, conditionally on any history
before hitting \(C_d\cup(S\times F)\), the probability of hitting this set
within the next \(N\) steps is at least
\(\underline p_{\min}^{N}\). Therefore
\[
\sup_{\widetilde\nature}
\mathbb E_{\widetilde{\mdp}[\widetilde\nature],d}[\tau_d]
\le
\frac{N}{\underline p_{\min}^{N}}
=
H_{\max}.
\]

Now fix an adversary \(\widetilde\nature\) of \(\widetilde{\mdp}\). Couple
the run of \((\mdp,d)\) and the run of
\((\widetilde{\mdp}[\widetilde\nature],d)\) as follows. As long as the two
concrete product histories coincide, the policy chooses the same action.
At each step, use a maximal coupling between the true transition
\(P(s,a)\) and the transition selected by \(\widetilde\nature\). By the
previous lemma, the one-step mismatch probability is at most \(\eta\).

Let \(\sigma\) be the first mismatch time. If \(\sigma>\tau_d\), then the
two product runs agree up to \(\tau_d\). At time \(\tau_d\), either both
have reached \(S\times F\), in which case both have observed a bad prefix,
or both have entered a bottom component \(C_d\) avoiding \(S\times F\), in
which case neither run can later observe a bad prefix. Thus, the two
satisfaction events can differ only on \(\{\sigma\le\tau_d\}\). Hence
\[
\left|
\prob_{\widetilde{\mdp}[\widetilde\nature],d}(\formula)
-
\prob_{\mdp,d}(\formula)
\right|
\le
\Pr(\sigma\le\tau_d).
\]
The stepwise coupling gives
\[
\Pr(\sigma\le\tau_d)
\le
\eta\,
\mathbb E_{\widetilde{\mdp}[\widetilde\nature],d}[\tau_d]
\le
\eta H_{\max}.
\]
Therefore,
\[
\prob_{\widetilde{\mdp}[\widetilde\nature],d}(\formula)
\ge
\prob_{\mdp,d}(\formula)-\eta H_{\max}.
\]
Finally, let \(\pi\in\mathcal C_{\mathcal A}(\mdp)\) and write
\(
\pi=\sum_d \lambda_d d
\)
as an initial mixture over deterministic \(\mathcal A\)-memory policies.
For fixed \(\widetilde\nature\), the path law under \(\pi\) is the
corresponding mixture of the path laws under the \(d\)'s. Averaging the
previous inequality gives
\[
\prob_{\widetilde{\mdp}[\widetilde\nature],\pi}(\formula)
\ge
\prob_{\mdp,\pi}(\formula)-\eta H_{\max}.
\]
Taking the infimum over \(\widetilde\nature\) proves the claim.
\end{proof}

We assume in the following that there exists a policy \(\pi_{\mathrm{sl}}\) on
\(\widehat{\mdp}\) and a constant \(\kappa>0\) such that
\[
\inf_{\widehat\nature}
\prob_{\widehat{\mdp}[\widehat\nature],\pi_{\mathrm{sl}}}(\formula)
\ge
1-p+\kappa.
\]
We identify \(\pi_{\mathrm{sl}}\) with its lift to the concrete state space
via \(\alpha\).

\begin{lemma}[Concrete Slater repair]
    Let \(\pi\in\mathcal C_{\mathcal A}(\mdp)\) such that $\mdp,\pi\models\spec$.
    Set
    \[
    r
    =
    \frac{\eta H_{\max}}{\kappa+\eta H_{\max}},
    \qquad
    \widehat\pi
    =
    (1-r)\pi\oplus r\pi_{\mathrm{sl}}.
    \]
    Then, it holds that 
    \(
    \widetilde{\mdp},\widehat\pi\models\spec,
    \)
    and 
    \(
    \mathrm{TV}_{\mdp}(\pi,\widehat\pi)\le r.
    \)
    \end{lemma}

\begin{proof}
Since \(\mdp,\pi\models\spec\), we have
\[
\prob_{\mdp,\pi}(\formula)\ge 1-p.
\]
By the concrete perturbation bound,
\[
\inf_{\widetilde\nature}
\prob_{\widetilde{\mdp}[\widetilde\nature],\pi}(\formula)
\ge
1-p-\eta H_{\max}.
\]
Since \(\widetilde{\mdp}/\alpha=\widehat{\mdp}\), the lifted Slater policy
satisfies
\[
\inf_{\widetilde\nature}
\prob_{\widetilde{\mdp}[\widetilde\nature],\pi_{\mathrm{sl}}}(\formula)
\ge
1-p+\kappa,
\]
so for every adversary \(\widetilde\nature\),
\[
\prob_{\widetilde{\mdp}[\widetilde\nature],\widehat\pi}(\formula)
\ge
(1-r)(1-p-\eta H_{\max})
+
r(1-p+\kappa).
\]
The choice
\[
r=
\frac{\eta H_{\max}}{\kappa+\eta H_{\max}}
\]
makes the right-hand side equal to \(1-p\). Hence
\[
\widetilde{\mdp},\widehat\pi\models\spec.
\]
Finally, the TV bound follows by coupling the initial mixture: with probability
\(1-r\), \(\widehat\pi\) follows \(\pi\) for the whole run, and with
probability \(r\), it follows \(\pi_{\mathrm{sl}}\). Thus, under the true
MDP,
\[
\mathrm{TV}\!\left(
\prob_{\mdp,\pi},
\prob_{\mdp,\widehat\pi}
\right)
\le r,
\]
which is equivalent to the claim
\(
\mathrm{TV}_{\mdp}(\pi,\widehat\pi)\le r
\),
which concludes the proof.
\end{proof}

\begin{theorem}[TV completeness for concrete automaton-memory policies]\label{th:TV-completeness-concrete-automaton}
Suppose that \(\mathcal M\) is finite, that $\mathcal M$ is conditionally deterministic w.r.t. $\alpha$, that \(\widehat{\mdp}\) is graph-preserving,
that \(\underline p_{\min}>0\), and that there exists a Slater policy
\(\pi_{\mathrm{sl}}\) with margin \(\kappa>0\). Suppose also that $
\widehat\beta(\xiinit,\qinit)<p$.
Then the shield $
\mathfrak S(\widehat{\mdp},\mathcal A,\widehat\beta)$
acts on \((\mdp,\alpha)\) and is $
\mathcal C_{\mathcal A}(\mdp)\text{-}\epsilon\text{-}\mathrm{TV}
\text{-complete}$
over \((\mdp,\alpha)\) for \(\spec\), where
\[
\epsilon
=
\frac{
\|\widehat\beta-\widehat\beta^\infty\|_\infty
}{
\|\widehat\beta-\widehat\beta^\infty\|_\infty
+
p-\widehat\beta(\xiinit,\qinit)
}
+
\frac{\eta H_{\max}}{\kappa+\eta H_{\max}},
\]
with
\[
H_{\max}
=
\frac{|S|\cdot |Q|}{(\underline p_{\min})^{\left(\,|S|\cdot |Q|\right)}}.
\]
\end{theorem}
\begin{proof}
Let \(\pi\in\mathcal C_{\mathcal A}(\mdp)\) be such that
\(\mdp,\pi\models\spec\). By the Slater repair lemma, the policy
\[
\widehat\pi
=
(1-r)\pi\oplus r\pi_{\mathrm{sl}},
\qquad
r=
\frac{\eta H_{\max}}{\kappa+\eta H_{\max}},
\]
is safe for the conditional lift \(\widetilde{\mdp}\), and $
\mathrm{TV}_{\mdp}(\pi,\widehat\pi)\le r$.

Since $
\widetilde{\mdp}/\alpha=\widehat{\mdp}$,
the shield $
\mathfrak S(\widehat{\mdp},\mathcal A,\widehat\beta)$
is exactly $
\mathfrak S(\widetilde{\mdp}/\alpha,\mathcal A,\widehat\beta)$. Since \(\mdp\) is conditionally deterministic with respect to \(\alpha\)
and \(\widehat{\mdp}\) is graph-preserving, the conditional concrete lift
\(\widetilde{\mdp}\) is conditionally deterministic with respect to
\(\alpha\). Hence Theorem \ref{th:RMDP-TV-completeness} applies to
\(\widetilde{\mdp}\). As a consequence, there exists a compliant HR policy
\(\overline\pi\) on \(\mdp\otimes_\alpha\mathcal D\) such that
\[
\mathrm{TV}_{\widetilde{\mdp}}
\bigl(\widehat\pi,\overline\pi^\downarrow\bigr)
\le
\epsilon_\beta,
\]
where
\[
\epsilon_\beta
=
\frac{
\|\widehat\beta-\widehat\beta^\infty\|_\infty
}{
\|\widehat\beta-\widehat\beta^\infty\|_\infty
+
p-\widehat\beta(\xiinit,\qinit)
}.
\]
This, in particular, implies that
\[
\mathrm{TV}_{\mdp}
\bigl(\widehat\pi,\overline\pi^\downarrow\bigr)
\le
\epsilon_\beta,
\]
which, by the triangle inequality, yields
\[
\mathrm{TV}_{\mdp}
\bigl(\pi,\overline\pi^\downarrow\bigr)
\le
r+\epsilon_\beta.
\]
Thus, the shield is
\(\mathcal C_{\mathcal A}(\mdp)\)-\(\epsilon\)-TV-complete with
\(
\epsilon=\epsilon_\beta+
\frac{\eta H_{\max}}{\kappa+\eta H_{\max}},
\)
which concludes the proof.
\end{proof}

\subsection{\texorpdfstring{\(W_{1,\gamma}\)}{W}-Completeness of the Shield for Unknown MDPs}

We now prove a \(W_{1,\gamma}\)-completeness result. In contrast to the
\(\mathrm{TV}\)-completeness theorem above, this result does not require a
uniform hitting-time bound. Instead, we only compare the true MDP and the
conditional lift over a finite horizon \(T\).

Let \(\widehat{\beta}^{\infty}\) be the least fixed point of $
\mathcal B_{\widehat{\mdp}\otimes_{\labelfunc_\alpha}\mathcal A}^{\Xi\times F}$.
Let \(\beta^\infty\) be the least fixed point of $
\mathcal B_{(\mdp/\alpha)\otimes_{\labelfunc_\alpha}\mathcal A}^{\Xi\times F}$.
We define $
\Delta_\infty
=
\|\widehat{\beta}^{\infty}-\beta^\infty\|_\infty$.
Since \(\alpha\) is an exact safety abstraction, the least bad-reachability
value on the concrete product \(\mdp\otimes_{\labelfunc}\mathcal A\) is the
lift of \(\beta^\infty\), namely $
(s,q)\mapsto \beta^\infty(\alpha(s),q)$.

Since the learned product RMDP is finite, choose a policy
\(\widehat\pi^\star\) on
\(\widehat{\mdp}\otimes_{\labelfunc_\alpha}\mathcal A\) such that, for
every \((\xi,q)\in\Xi\times Q\),
\[
\sup_{\widehat\nature}
\prob_{\widehat{\mdp}[\widehat\nature],\widehat\pi^\star}^{(\xi,q)}
\left[\left\{ h\mid
h\models \mathbf F(\Xi\times F)\right\}
\right]
\le
\widehat{\beta}^{\infty}(\xi,q).
\]
We identify \(\widehat\pi^\star\) with its lift to the concrete state
space through \(\alpha\).

For an HR policy \(\pi\) on \(\mdp\) and \(T\in\Nat\), let
\(\pi^{(T)}\) be the policy that follows \(\pi\) for the first \(T\)
decision times and then follows the lifted policy \(\widehat\pi^\star\)
forever.

\begin{lemma}[Hard-switch perturbation bound]
\label{lem:hard-switch-perturbation-concrete}
For every HR policy \(\pi\) on \(\mdp\) and every \(T\in\Nat\),
\[
\inf_{\widetilde\nature}
\prob_{\widetilde{\mdp}[\widetilde\nature],\pi^{(T)}}(\formula)
\ge
\prob_{\mdp,\pi}(\formula)-T\eta-\Delta_\infty,
\]
where the infimum ranges over adversaries of \(\widetilde{\mdp}\).
\end{lemma}

\begin{proof}
Fix an adversary \(\widetilde\nature\) of \(\widetilde{\mdp}\). We couple
the run of \((\mdp,\pi)\) with the run of
\((\widetilde{\mdp}[\widetilde\nature],\pi^{(T)})\) during the first
\(T\) decision times. As long as the two concrete histories coincide, the
two policies choose the same action distribution, because \(\pi^{(T)}\)
coincides with \(\pi\) before time \(T\). We couple the actions
identically, and then use a maximal coupling between the true transition
\(P(s,a)\) and the transition selected by \(\widetilde\nature\). By the
concrete one-step perturbation lemma, the one-step mismatch probability is
at most \(\eta\).

Let
\(
\sigma=\inf\{t\ge 0:S_t\neq \widetilde S_t\}
\)
be the first time at which the two concrete state sequences differ. Then
\[
\Pr(\sigma\le T)\le T\eta.
\]
Let \(B\) be the event that the true run violates \(\formula\), and let
\(\widetilde B\) be the event that the lifted run violates \(\formula\).
Let \(A_T\) be the event that no accepting state of the DFA has been
reached up to time \(T\) along the true run. On the event
\(\{\sigma>T\}\cap A_T\), the two product runs have the same concrete
product state \((S_T,Q_T)\) at time \(T\). From that time onward, the
lifted run follows \(\widehat\pi^\star\). Hence the conditional probability
of eventually reaching \(S\times F\) is at most
\[
\widehat{\beta}^{\infty}(\alpha(S_T),Q_T).
\]
Therefore, we obtain
\[
\Pr(\widetilde B)
\le
\Pr(\sigma\le T)
+
\Pr(B\cap A_T^c)
+
\mathbb E\!\left[
\mathbf 1_{A_T}
\widehat{\beta}^{\infty}(\alpha(S_T),Q_T)
\right].
\]

Let \(v_\pi(H_T)\) be the conditional probability, under the true MDP and
policy \(\pi\), of eventually violating \(\formula\) after the concrete
history \(H_T\).
Using this notation, we write
\[
\Pr(B)
=
\Pr(B\cap A_T^c)
+
\mathbb E\!\left[
\mathbf 1_{A_T}v_\pi(H_T)
\right].
\]
Since \(\alpha\) is a safety abstraction, the least bad-reachability value
on the concrete product is the lift of \(\beta^\infty\). Hence, on
\(A_T\),
\[
\beta^\infty(\alpha(S_T),Q_T)
\le
v_\pi(H_T),
\]
from which we find that
\[
\widehat{\beta}^{\infty}(\alpha(S_T),Q_T)-v_\pi(H_T)
\le
\widehat{\beta}^{\infty}(\alpha(S_T),Q_T)
-
\beta^\infty(\alpha(S_T),Q_T)
\le
\Delta_\infty.
\]
Combining the previous inequalities gives
\[
\Pr(\widetilde B)-\Pr(B)
\le
T\eta+\Delta_\infty,
\]
which is equivalent to
\[
\prob_{\widetilde{\mdp}[\widetilde\nature],\pi^{(T)}}(\formula)
\ge
\prob_{\mdp,\pi}(\formula)-T\eta-\Delta_\infty.
\]
Taking the infimum over \(\widetilde\nature\) proves the claim.
\end{proof}
\begin{lemma}[Hard-switch Wasserstein cost]
\label{lem:hard-switch-wasserstein-concrete}
For every HR policy \(\pi\) on \(\mdp\) and every \(T\in\Nat\),
\[
W_{1,\gamma}^{\mdp}\!\left(\pi,\pi^{(T)}\right)
\le
\gamma^T.
\]
\end{lemma}

\begin{proof}
Couple the runs of \(\pi\) and \(\pi^{(T)}\) on the true MDP so that they
use the same actions and transition randomness during the first \(T\)
decision times. The two induced state-action trajectories then coincide up
to time \(T\). Therefore, their discounted Hamming distance is at most
\[
(1-\gamma)\sum_{t=T}^{\infty}\gamma^t
=
\gamma^T.
\]
Taking the infimum over all couplings yields the claim.
\end{proof}

\begin{lemma}[Hard-switch Slater repair]
\label{lem:hard-switch-slater-concrete}
Assume that there exists a policy \(\pi_{\mathrm{sl}}\) on
\(\widehat{\mdp}\) and a constant \(\kappa>0\) such that
\[
\inf_{\widehat\nature}
\prob_{\widehat{\mdp}[\widehat\nature],\pi_{\mathrm{sl}}}(\formula)
\ge
1-p+\kappa.
\]
Let \(\pi\) be an HR policy on \(\mdp\) such that
\(
\mdp,\pi\models\spec.
\)
For \(T\in\Nat\), set
\[
\delta_T\coloneqq T\eta+\Delta_\infty,
\qquad
\lambda_T\coloneqq
\frac{\delta_T}{\kappa+\delta_T},
\]
and define
\[
\widehat\pi^{(T)}
\coloneqq
(1-\lambda_T)\pi^{(T)}
\oplus
\lambda_T\pi_{\mathrm{sl}},
\]
where \(\pi_{\mathrm{sl}}\) is lifted to the concrete state space through
\(\alpha\). Then
\(
\widetilde{\mdp},\widehat\pi^{(T)}\models\spec.
\)
Moreover,
\[
W_{1,\gamma}^{\mdp}\!\left(\pi,\widehat\pi^{(T)}\right)
\le
\gamma^T+\lambda_T.
\]
\end{lemma}

\begin{proof}
Since \(\mdp,\pi\models\spec\), we have
\(
\prob_{\mdp,\pi}(\formula)\ge 1-p.
\)
By Lemma~\ref{lem:hard-switch-perturbation-concrete},
\[
\inf_{\widetilde\nature}
\prob_{\widetilde{\mdp}[\widetilde\nature],\pi^{(T)}}(\formula)
\ge
1-p-\delta_T.
\]
Since \(\widetilde{\mdp}/\alpha=\widehat{\mdp}\), the lifted Slater policy
satisfies
\[
\inf_{\widetilde\nature}
\prob_{\widetilde{\mdp}[\widetilde\nature],\pi_{\mathrm{sl}}}(\formula)
\ge
1-p+\kappa.
\]
Therefore, for every adversary \(\widetilde\nature\),
\[
\prob_{\widetilde{\mdp}[\widetilde\nature],\widehat\pi^{(T)}}(\formula)
\ge
(1-\lambda_T)(1-p-\delta_T)
+
\lambda_T(1-p+\kappa).
\]
By the definition of \(\lambda_T\), the right-hand side is equal to
\(1-p\), so the first claim
\(
\widetilde{\mdp},\widehat\pi^{(T)}\models\spec
\)
follows.

Second, for the Wasserstein bound, by convexity of \(W_{1,\gamma}\) in its second
argument,
\[
W_{1,\gamma}^{\mdp}\!\left(\pi,\widehat\pi^{(T)}\right)
\le
(1-\lambda_T)
W_{1,\gamma}^{\mdp}\!\left(\pi,\pi^{(T)}\right)
+
\lambda_T
W_{1,\gamma}^{\mdp}\!\left(\pi,\pi_{\mathrm{sl}}\right).
\]
Since \(W_{1,\gamma}\le 1\), and by
Lemma~\ref{lem:hard-switch-wasserstein-concrete},
\[
W_{1,\gamma}^{\mdp}\!\left(\pi,\widehat\pi^{(T)}\right)
\le
(1-\lambda_T)\gamma^T+\lambda_T
\le
\gamma^T+\lambda_T.
\]
\end{proof}

\begin{theorem}[Hard-switch \(W_{1,\gamma}\)-completeness]
\label{th:hard-switch-wasserstein}
Suppose that $
\widehat{\beta}(\xiinit,\qinit)<p$, suppose that $\mathcal M$ is conditionally deterministic w.r.t. $\alpha$,
and suppose that there exists a Slater policy
\(\pi_{\mathrm{sl}}\) with margin \(\kappa>0\), \ie
\[
\inf_{\widehat\nature}
\prob_{\widehat{\mdp}[\widehat\nature],\pi_{\mathrm{sl}}}(\formula)
\ge
1-p+\kappa.
\]
Then, for every \(T\in\Nat\), the shield $
\mathfrak S(\widehat{\mdp},\mathcal A,\widehat{\beta})$
acts on \((\mdp,\alpha)\) and is $
\mathrm{HR}\text{-}\epsilon_T\text{-}W_{1,\gamma}\text{-complete}$
over \((\mdp,\alpha)\) for \(\spec\), where
\[
\epsilon_T
=
\frac{
\|\widehat{\beta}-\widehat{\beta}^{\infty}\|_\infty
}{
\|\widehat{\beta}-\widehat{\beta}^{\infty}\|_\infty
+
p-\widehat{\beta}(\xiinit,\qinit)
}
+
\gamma^T
+
\frac{T\eta+\Delta_\infty}{\kappa+T\eta+\Delta_\infty}.
\]
\end{theorem}

\begin{proof}
Let \(\pi\) be an HR policy on \(\mdp\) such that
\[
\mdp,\pi\models\spec.
\]
By Lemma~\ref{lem:hard-switch-slater-concrete}, the policy
\(\widehat\pi^{(T)}\) is safe for the conditional lift
\(\widetilde{\mdp}\), and
\[
W_{1,\gamma}^{\mdp}\!\left(\pi,\widehat\pi^{(T)}\right)
\le
\gamma^T+
\frac{T\eta+\Delta_\infty}{\kappa+T\eta+\Delta_\infty}.
\]

For brevity, define
\[
\epsilon_\beta
=
\frac{
\|\widehat{\beta}-\widehat{\beta}^{\infty}\|_\infty
}{
\|\widehat{\beta}-\widehat{\beta}^{\infty}\|_\infty
+
p-\widehat{\beta}(\xiinit,\qinit)
}.
\]
Since
\(
\widetilde{\mdp}/\alpha=\widehat{\mdp},
\)
the shield
\(
\mathfrak S(\widehat{\mdp},\mathcal A,\widehat{\beta})
\)
is exactly
\(
\mathfrak S(\widetilde{\mdp}/\alpha,\mathcal A,\widehat{\beta}).
\)
Applying the RMDP \(\mathrm{TV}\)-completeness theorem to
\(\widetilde{\mdp}\), there exists a compliant HR policy
\(\overline\pi\) on \(\mdp\otimes_\alpha\mathcal D\) such that
\[
\mathrm{TV}_{\widetilde{\mdp}}
\left(
\widehat\pi^{(T)},\overline\pi^\downarrow
\right)
\le
\epsilon_\beta.
\]
Since the true MDP \(\mdp\) is a resolution of \(\widetilde{\mdp}\), this
implies
\[
\mathrm{TV}_{\mdp}
\left(
\widehat\pi^{(T)},\overline\pi^\downarrow
\right)
\le
\epsilon_\beta.
\]
By Lemma~\ref{lemma:dominating distance}, we obtain
\[
W_{1,\gamma}^{\mdp}
\left(
\widehat\pi^{(T)},\overline\pi^\downarrow
\right)
\le
\epsilon_\beta,
\]
which, via the triangle inequality gives
\[
W_{1,\gamma}^{\mdp}
\left(
\pi,\overline\pi^\downarrow
\right)
\le
\epsilon_\beta
+
\gamma^T
+
\frac{T\eta+\Delta_\infty}{\kappa+T\eta+\Delta_\infty}.
\]
Thus, the shield is
\(\mathrm{HR}\)-\(\epsilon_T\)-\(W_{1,\gamma}\)-complete.
\end{proof}

\subsection{Optimality Results}

We now derive optimality guarantees presented in \cref{thm:unknown_MDP_optimality} from the completeness results above.
The \(W_{1,\gamma}\)-result already applies to all HR policies. For the
\(\mathrm{TV}\)-result, which was proved only for the class
\(\mathcal C_{\mathcal A}(\mdp)\), we first recall that this class is
sufficient to attain the optimal value of the finite undiscounted safe
policy optimization problem under standard total-reward assumptions.

\begin{lemma}[Initial mixtures with automaton memory suffice]
\label{lem:CA-suffices-for-undiscounted-optimality}
Assume that \(\mdp\) is finite. Consider the undiscounted safe policy
optimization problem
\[
\sup_{\pi\in\mathrm{HR}} J^1_{\mdp}(\pi)
\qquad
\text{subject to}
\qquad
\mdp,\pi\models\mathbb P_{\ge 1-p}(\formula).
\]
Suppose that:
\begin{enumerate}
    \item for every HR policy \(\pi\), the total reward
    \(
    \sum_{t=0}^{\infty}\rewardfunc(s_t,a_t)
    \)
    is almost surely convergent and has finite expectation;
    \item an optimal feasible HR policy exists.
\end{enumerate}
Then there exists an optimal feasible policy
\(
\pi^\star\in\mathcal C_{\mathcal A}(\mdp).
\)
Moreover, \(\pi^\star\) can be chosen as an initial mixture of at most two
deterministic \(\mathcal A\)-memory policies.
\end{lemma}

\begin{proof}
Consider the product MDP
\(
\mdp\otimes_{\labelfunc}\mathcal A
\)
with state space \(S\times Q\), and let \(G=S\times F\) be the set of bad
product states. We make \(G\) absorbing and set the reward to zero after
\(G\) is reached. This does not change either the event of violating
\(\formula\) or the value of the total reward up to violation.

The safety constraint can be written as a total-cost constraint. Define the
one-step cost
\[
c((s,q),a)
=
\mathbb P_{s'\sim P(s,a)}
\left[
\delta(q,\labelfunc(s'))\in F
\right]
\]
for \(q\notin F\), and \(c((s,q),a)=0\) for \(q\in F\). Then, for every
policy \(\pi\),
\[
\mathbb E^\pi
\left[
\sum_{t=0}^{\infty} c((s_t,q_t),a_t)
\right]
\]
is the probability to reach $G$ following $\pi$, so the constraint
\[
\mdp,\pi\models\mathbb P_{\ge 1-p}(\formula)
\]
is equivalent to
\[
\mathbb E^\pi
\left[
\sum_{t=0}^{\infty} c((s_t,q_t),a_t)
\right]
\le p.
\]

Thus, the problem is a finite total-reward CMDP on the product MDP with one
total-cost constraint. By the standard occupation-measure formulation for
finite total-reward CMDPs, the achievable reward-cost region is the convex
hull of the performance vectors of deterministic stationary policies on
the product MDP. Hence, an optimal feasible point can be achieved by an
initial mixture of deterministic stationary product policies. Since there
is one cost constraint, Carathéodory's theorem, or equivalently, the basic
feasible solution structure of the corresponding linear program, implies
that at most two deterministic stationary product policies are needed.

Finally, deterministic stationary policies on the product \(S\times Q\)
are exactly deterministic \(\mathcal A\)-memory policies on \(\mdp\), and
their initial mixtures are precisely the policies in
\(\mathcal C_{\mathcal A}(\mdp)\). Therefore, an optimal feasible policy
exists in \(\mathcal C_{\mathcal A}(\mdp)\).
\end{proof}

We are now ready to state our two final optimality results, as presented in \cref{sec:unknown_MDPs}.
The first follows directly from combining Theorem \ref{th:TV-completeness-concrete-automaton}, Lemma \ref{lem:CA-suffices-for-undiscounted-optimality}, and Lemma \ref{lemma:completeness-to-optimality}.

\begin{theorem}[Undiscounted HR-optimality from TV-completeness; extension of~\cref{thm:unknown_MDP_optimality},~case~1]
\label{th:tv-hr-optimality-unknown-mdp}
Assume the hypotheses of Theorem~\ref{th:TV-completeness-concrete-automaton}
and Lemma~\ref{lem:CA-suffices-for-undiscounted-optimality}. Suppose
moreover that the absolute value of the total undiscounted return is
uniformly bounded by \(B\). Let
\[
\epsilon_{\mathrm{TV}}
=
\frac{
\|\widehat\beta-\widehat\beta^\infty\|_\infty
}{
\|\widehat\beta-\widehat\beta^\infty\|_\infty
+
p-\widehat\beta(\xiinit,\qinit)
}
+
\frac{\eta H_{\max}}{\kappa+\eta H_{\max}}.
\]
Then the shield $
\mathfrak S(\widehat{\mdp},\mathcal A,\widehat\beta)$
is
\[
\mathrm{HR}\text{-}(2B\epsilon_{\mathrm{TV}},1)\text{-optimal}
\]
over \((\mdp,\alpha)\) for \(\spec\).
\end{theorem}

The following result follows directly from combining Theorem \ref{th:hard-switch-wasserstein} and Lemma \ref{lemma:completeness-to-optimality}.

\begin{theorem}[Discounted HR-optimality from hard-switch completeness; extension of~\cref{thm:unknown_MDP_optimality},~case~2]
\label{th:wasserstein-hr-optimality-unknown-mdp}
Assume the hypotheses of Theorem~\ref{th:hard-switch-wasserstein}. Suppose
that $
\sup_{(s,a)\in S\times\mathsf A}|\rewardfunc(s,a)|\le Z$.
Then, the shield $
\mathfrak S(\widehat{\mdp},\mathcal A,\widehat\beta)$
is
\[
\mathrm{HR}\text{-}
\left(
\frac{2Z\epsilon}{1-\gamma},
\gamma
\right)\text{-optimal}
\]
over \((\mdp,\alpha)\) for \(\spec\), where
\[
\epsilon
=
\min_{T\in\mathbb N}\left(\frac{
\|\widehat\beta-\widehat\beta^\infty\|_\infty
}{
\|\widehat\beta-\widehat\beta^\infty\|_\infty
+
p-\widehat\beta(\xiinit,\qinit)
}
+
\gamma^T
+
\frac{T\eta+\Delta_\infty}{\kappa+T\eta+\Delta_\infty}
\right).\]
\end{theorem}

\section{Details for Experiments}

\subsection{Implementation Details}
\label{appendix:implementation}

In this appendix, we further detail the implementation of our shielding pipeline, in particular, for step (3) of running policy optimization with the shield active.
Recall that the transition function $P$ of the true MDP $\mdp$ is unknown, but we instead have access to a simulator.
Thus, enabling the shield means that we obtain a shielded simulator MDP $\overline{\mdp}^{\mathfrak S} \left(\overline S,A^{\mathfrak S},P^{\mathfrak S},\overline\sinit,AP,\overline{\labelfunc}\right)$, 
obtained from the restriction of policy distributions allowed by the shield. More precisely, $\overline{\mdp}^{\mathfrak S}$ is the MDP with same set of states, labels, labelling, and rewards as $\mathcal M\otimes_{\alpha} \mathcal D$, but with action space $A^{\mathfrak S}(s,m)=\Gamma(\alpha(s),m)$, and probability transition function \[
\overline P^{\mathfrak S}((s,m),\rho)(E)
=
\int_{\mathsf A\times U}
\overline P((s,m),(a,u))(E)\,
\rho(d(a,u)).
\]
Given the definition of $\Gamma$, the space $A^{\mathfrak S}(s,m)$ is convex, and, for optimization purposes, $A^{\mathfrak S}(s,m)$ can be taken as its extremal points. 
Furthermore, since $\Gamma(\alpha(s),m)$ is defined over $\overline\mu\in\distr{\mathsf A\times\mathcal V_\beta}$ by a single scalar inequality, an extremal point of $A^{\mathfrak S}(s,m)$ is always a mixture of at most two Dirac distributions over $A\times U$. In addition, for a fixed action distribution \(\mu\) over two actions $a$ and $a'$, the set of all auxiliary actions $(v,v')$ that satisfy compliance with the shield is a convex polytope itself. More precisely, for any action distribution $\mu$, the set of all possible auxiliary actions for every base action $a$ that satisfy compliance with the shield is the set $\mathcal V_{\lambda}(\xi,q,y)$ of all $(v_a)_{a\in A_\alpha(\xi)}
\in
\mathcal V_\beta^{A_\alpha(\xi)}$ such that
\[
\sum_{a\in A_\alpha(\xi)}
\mu(a)\sup_{\mu\in\mathcal P_\alpha(\xi,a)}
\mathbb E_{\xi'\sim\mu}
\left[
v_a\bigl(\xi',\delta(q,\labelfunc_\alpha(\xi'))\bigr)
\right]
\le y.
\]

As a consequence, since we still assume the the quotient $\mathcal M/\alpha$ is finite, the $\mathcal V_\beta^{A_\alpha(\xi)}$ has finitely many extremal points. Thus, we can restrict the action space of $\overline{\mdp}^{\mathfrak S}$ to 
\begin{enumerate}
    \item the choice of two actions, of a trade-off $\lambda\in [0;1]$ between those two action, inducing a distribution $\mu$ over actions, and
    \item the choice of an extremal point of $\mathcal V_{\lambda}(\xi,q,y)$.
\end{enumerate}
In practice, we optimize with PPO~\cite{DBLP:journals/corr/SchulmanWDRK17} over $\overline{\mdp}^{\mathfrak S}$, we leave the choice of the trade-off $\lambda$ to the actor, and implement the following heuristic for the choice of the extremal point of $\mathcal V_{\lambda}(\xi,q,y)$.

\paragraph{Auxiliary action heuristic.}
We describe the heuristic used to choose the auxiliary actions after
the actor has selected a mixed base action. Fix a shield state
\((\xi,q,y)\), and let \(\lambda\in\distr{A_\alpha(\xi)}\) be the action
distribution selected by the actor. In principle, a compliant shielded
action may assign an action-dependent continuation certificate
\((v_a)_{a\in A_\alpha(\xi)}\in\mathcal V_{\widehat\beta}^{A_\alpha(\xi)}\)
satisfying
\[
\sum_{a\in A_\alpha(\xi)}
\lambda(a)
\sup_{\widehat\mu\in\widehat{\mathcal P}(\xi,a)}
\mathbb E_{\xi'\sim\widehat\mu}
\left[
v_a\bigl(\xi',\delta(q,\labelfunc_\alpha(\xi'))\bigr)
\right]
\le y .
\]
Rather than optimizing over this full polytope, our implementation uses a
single shared continuation certificate \(v\in\mathcal V_{\widehat\beta}\)
for all actions in the support of \(\lambda\). This is conservative but
simplifies the shielded action parametrization. Starting from the baseline
certificate \(\widehat\beta\), we distribute the available residual budget
uniformly by considering the one-dimensional family
\[
v_m(\xi',q')
=
\min\{\widehat\beta(\xi',q')+m,1\},
\qquad m\ge 0.
\]
We then choose the largest margin \(m\) preserving one-step robust
feasibility:
\[
m^\star
=
\sup
\left\{
m\ge 0:
\sum_{a\in A_\alpha(\xi)}
\lambda(a)
\sup_{\widehat\mu\in\widehat{\mathcal P}(\xi,a)}
\mathbb E_{\xi'\sim\widehat\mu}
\left[
v_m\bigl(\xi',\delta(q,\labelfunc_\alpha(\xi'))\bigr)
\right]
\le y
\right\}.
\]
The auxiliary action used by the shield is \(v_{m^\star}\), i.e. we take
\(v_a=v_{m^\star}\) for every action \(a\) in the support of
\(\lambda\). If \(m=0\) is already infeasible, then the proposed action
distribution \(\lambda\) is not compatible with the shield; in that case
the implementation falls back to a precomputed safe distribution witnessing
the inductiveness of \(\widehat\beta\).

For finite RMDPs, the map
\[
m\mapsto
\sum_{a\in A_\alpha(\xi)}
\lambda(a)
\sup_{\widehat\mu\in\widehat{\mathcal P}(\xi,a)}
\mathbb E_{\xi'\sim\widehat\mu}
\left[
v_m\bigl(\xi',\delta(q,\labelfunc_\alpha(\xi'))\bigr)
\right]
\]
is monotone and piecewise affine. Its breakpoints occur only when a
coordinate \(\widehat\beta(\xi',q')+m\) reaches~\(1\). We therefore compute
\(m^\star\) by sorting these saturation breakpoints and scanning the
corresponding affine pieces. On the first segment where the robust
expectation reaches the budget \(y\), \(m^\star\) is obtained by solving a
one-dimensional affine equation. If the robust expectation remains below
\(y\) after all coordinates have saturated, then all continuation
thresholds can be set to~\(1\).

\subsection{Additional Experimental Results}
\label{app:additional_results}

This section reports the full experimental results for each environment considered in the main paper.

\paragraph{Environments.}

We now provide a description of the environments used to evaluate our approach. 
These environments have been used in prior work, and their behavior remains unchanged unless stated otherwise.
In particular, we consider two variants of the \textbf{Media Streaming} environment, two variants of the \textbf{Colour Bomb Gridworld}, as well as Pacman and its variant, \textbf{Pacman Slippery}, in which the agent may execute slippery actions. 

Overall, this results in a total of six environments. 
For brevity, we provide a detailed description only for the general version of each environment family (\ie three environments in total), while the corresponding variants differ only in the aspects mentioned above.

\begin{itemize}
    \item \textbf{Media streaming.}
    The agent manages a data buffer of size $20$. Packets leave the buffer according to a Bernoulli process with rate $\mu_{\textit{out}}=0.7$. At each timestep, the agent chooses between two actions $A=\{\textit{slow},\textit{fast}\}$, which add packets according to Bernoulli rates $\mu_{\textit{slow}}=0.1$ and $\mu_{\textit{fast}}=0.9$, respectively. The reward objective is to minimise outage time: the agent receives reward $-1$ whenever the buffer is empty and $0$ otherwise. Safety is governed by a counter $c_t$ recording the number of times the fast action has been used, capped at $C+1$, with $C=\lfloor T/2\rfloor$ and $T=100$. We use the probabilistic safety constraint
    \[
        \mathbb P_{\geq 0.5}
        \left(
            \mathbf G(c_t \leq C)
        \right).
    \]
    The environment has $20 \times 52 = 1040$ states, corresponding to the buffer level and the capped fast-action counter. A related benchmark has been considered in \cite{DBLP:conf/nips/BuraHKSC22}.

    \item \textbf{Media streaming - Alternative.}
    We also consider a stochastic-safety variant of the media-streaming environment. The buffer dynamics and reward function are the same as above, but safety is governed by a danger level $d_t$, initialised at $0$ and capped at $D_{\max}+1$, where $D_{\max}=20$. We use the safety constraint
    \[
        \mathbb P_{\geq 0.5}
        \left(
            \mathbf G(d_t \leq D_{\max})
        \right).
    \]
    Under the slow action, the danger level decreases by one with probability $0.5$, remains unchanged with probability $0.1$, and increases by one with probability $0.4$. Under the fast action, the danger level increases by one with probability $0.8$, remains unchanged with probability $0.1$, and decreases by one with probability $0.1$. Thus, the fast action improves the buffer more reliably but tends to increase the probability of entering unsafe states, while the slow action is safer on average but does not guarantee that danger decreases. The full observation is $(d_t,b_t,t)$, giving $22 \times 20 \times 101 = 44440$ possible observations. Since the buffer affects reward but not safety, the shield uses the abstraction $(d_t,t)$, which has $22 \times 101 = 2222$ states.

    \item \textbf{Colour bomb gridworld.}
    The agent operates in a $15\times 15$ gridworld, with discrete actions
    $A=\{\textit{left},\textit{right},\textit{down},\textit{up},\textit{stay}\}$.
    Outside the medic states, the intended action is executed with probability $0.9$, while with probability $0.1$ the agent slips uniformly to one of the other actions. In medic states, the dynamics are deterministic. The environment contains wall cells, bomb cells, medic cells, and coloured goal regions. At the beginning of each episode, the initial state is sampled uniformly from a fixed set of start states. The agent receives reward $+1$ upon entering any coloured goal region and reward $0$ otherwise; after reaching such a goal region, the process is reset to a uniformly sampled start state. The base gridworld has $15^2=225$ states. We consider the probabilistic safety constraint
    \[
        \mathbb{P}_{\geq 0.5}
        \left(
            \mathbf{G}\,\neg\mathsf{bomb}
        \right),
    \]
    and an LTL-safety variant
    \[
        \mathbb P_{\geq 0.99}
        \left(
        \mathbf{G}
        \left(
            \mathsf{bomb}
            \Rightarrow
            \mathbf{F}_{\leq 10}
            \left(
                \mathsf{medic} \wedge \mathbf{X}\mathsf{medic}
            \right)
        \right)
        \right).
    \]
    The LTL property is represented by a $22$-state safety automaton, so the product MDP used by the shield has $225\times 22 = 4950$ states. A related gridworld has been used in~\cite{ABENTShielding}.

    \item \textbf{Pacman and slippery Pacman.}
    We consider a Pacman-with-coins environment on a $7\times 10$ map. The environment contains one Pacman agent, one ghost, walls, and collectible coins, with reward $+1$ obtained when Pacman visits a cell containing an uncollected coin. The full coin-augmented state is combinatorial: the map has $28$ free cells, giving up to $2^{28}$ possible coin configurations for each agent--ghost configuration. For shielding, we use a safety abstraction that discards the coins and retains only the positions and directions of Pacman and the ghost. This abstraction has $4624$ states. The ghost dynamics are stochastic: the ghost chooses an available action directed toward Pacman with probability $0.4$, and otherwise randomises uniformly over the other available actions with total probability $0.6$. We use the safety specification
    \[
        \mathbb P_{\geq 1-p}
        \left(
            \mathbf G(\mathrm{loc}_{\mathrm{Pacman}} \neq \mathrm{loc}_{\mathrm{Ghost}})
        \right),
    \]
    where $p=0.05$ for Pacman and $0.5$ for Pacman slippery. We also evaluate a slippery variant on the same map and with the same $4624$-state safety abstraction. In this variant, the requested Pacman action is executed with probability $0.9$, while with probability $0.1$ one of the other actions is executed uniformly at random. The task is to collect as many coins as possible while satisfying the probabilistic safety specification of avoiding the ghost throughout the episode. Similar environments were considered in \cite{ABENTShielding,DBLP:conf/aaai/CourtBG25}.
\end{itemize}

We use the shorthand labels \texttt{MDP}, \texttt{MLE}, \texttt{Kno}, and \texttt{Unk} in all plots to indicate the four different cases described in \cref{sec:experiments}.

\paragraph{Evaluation metrics.}
We report two evaluation metrics:
\begin{itemize}
    \item \textbf{Expected discounted reward} (\texttt{rew}): the average episodic return, corresponding to an empirical estimate of $J^\gamma_M(\pi)$.
    \item \textbf{Specification satisfaction probability} (\texttt{sat}): the empirical probability that a rollout satisfies the specification.
\end{itemize}

When learning curves are shown, we additionally report the corresponding training-time quantities
\texttt{rew} and \texttt{sat}.
For each environment, we provide two groups of plots:
\begin{enumerate}
    \item learning curves over training steps, shown for each sample size (the specific number that changes across the various environments represents the minimum number of samples needed for \texttt{Unk});
    \item averaged performance as a function of sample size.
\end{enumerate}

\input{figures/results/reformatted_plots}

%% file: figures/results/reformatted_plots.tex
\captionsetup{hypcap=false}
\FloatBarrier
\paragraph{Colour bomb gridworld}

\begin{center}
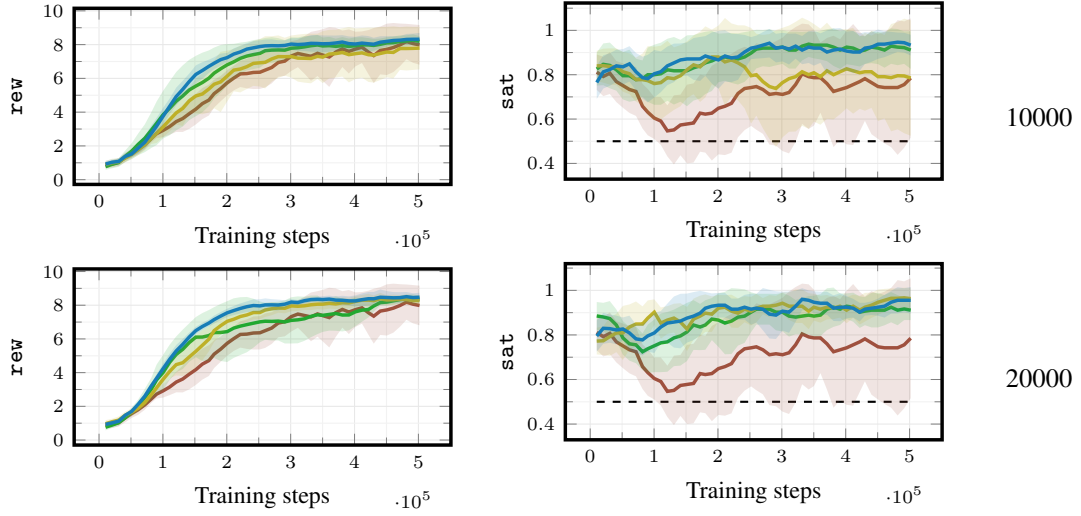
  %
    \begin{longtable}{@{}ccr@{}}
        \multicolumn{2}{c}{\input{figures/results/legend}} & Samples\\
        \input{figures/results/GW/GW_SA_10_learning_rew} & 
        \input{figures/results/GW/GW_SA_10_learning_sat} &
        10 \\
        \input{figures/results/GW/GW_SA_100_learning_rew} &
        \input{figures/results/GW/GW_SA_100_learning_sat} &
        100 \\
    \input{figures/results/GW/GW_SA_980_learning_rew} &
    \input{figures/results/GW/GW_SA_980_learning_sat} &
    980 \\

    \input{figures/results/GW/GW_SA_1000_learning_rew} &
    \input{figures/results/GW/GW_SA_1000_learning_sat} &
    1000 \\

    \input{figures/results/GW/GW_SA_5000_learning_rew} &
    \input{figures/results/GW/GW_SA_5000_learning_sat} &
    5000 \\

    \input{figures/results/GW/GW_SA_10000_learning_rew} &
    \input{figures/results/GW/GW_SA_10000_learning_sat} &
    10000 \\
    
        \input{figures/results/GW/GW_SA_20000_learning_rew} &
        \input{figures/results/GW/GW_SA_20000_learning_sat} &
        20000
    \end{longtable}
    \captionof{figure}{Learning curves for the \textbf{Colour bomb gridworld} environment. Left: \texttt{rew}. Right: \texttt{sat}. Rows correspond to different sample sizes.}
    \label{fig:gw-learning-curves-4}
\end{center}  %

\begin{center}
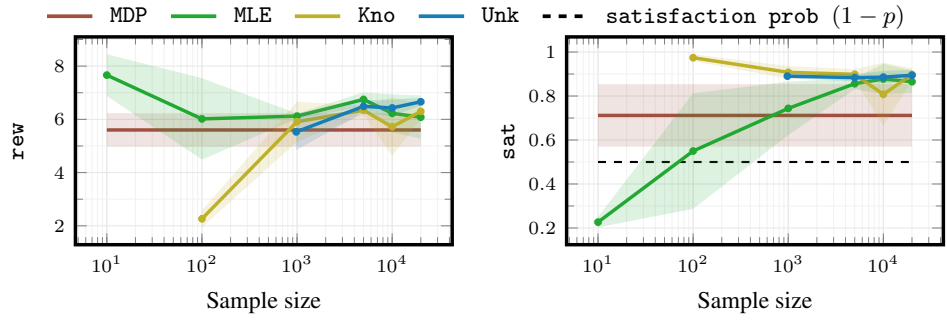
  %
    \begin{tabular}{cc}
        \multicolumn{2}{c}{\input{figures/results/legend}}\\
        \input{figures/results/GW/GW_finalavg_rew} &
        \input{figures/results/GW/GW_finalavg_sat}
    \end{tabular}
    \captionof{figure}{Average performance on \textbf{Colour bomb gridworld} as a function of sample size. Left: \texttt{rew}. Right: \texttt{sat}.}
    \label{fig:gw-sample-size}
\end{center}  %

\FloatBarrier
\paragraph{Colour bomb gridworld -- LTL goal}
\begin{center}  %
    \begin{longtable}{@{}ccc@{}}
    \multicolumn{2}{c}{\input{figures/results/legend}} & Samples \\

    \input{figures/results/GWltl/GWltl_SA_10_learning_rew} &
    \input{figures/results/GWltl/GWltl_SA_10_learning_sat} &
    10 \\

    \input{figures/results/GWltl/GWltl_SA_100_learning_rew} &
    \input{figures/results/GWltl/GWltl_SA_100_learning_sat} &
    100 \\

    \input{figures/results/GWltl/GWltl_SA_1000_learning_rew} &
    \input{figures/results/GWltl/GWltl_SA_1000_learning_sat} &
    1000 \\

    \input{figures/results/GWltl/GWltl_SA_1226_learning_rew} &
    \input{figures/results/GWltl/GWltl_SA_1226_learning_sat} &
    1226 \\

    \input{figures/results/GWltl/GWltl_SA_5000_learning_rew} &
    \input{figures/results/GWltl/GWltl_SA_5000_learning_sat} &
    5000 \\

    \input{figures/results/GWltl/GWltl_SA_10000_learning_rew} &
    \input{figures/results/GWltl/GWltl_SA_10000_learning_sat} &
    10000 \\
    
        \input{figures/results/GWltl/GWltl_SA_20000_learning_rew} &
        \input{figures/results/GWltl/GWltl_SA_20000_learning_sat} &
        20000 \\
    \end{longtable}
    \captionof{figure}{Learning curves for the \textbf{Colour bomb gridworld -- LTL} environment. Left: \texttt{rew}. Right: \texttt{sat}. Rows correspond to different sample sizes.}
    \label{fig:gwltl-learning-curves-4}
\end{center}  %

\begin{center}  %
    \begin{tabular}{cc}
        \multicolumn{2}{c}{\input{figures/results/legend}}\\
        \input{figures/results/GWltl/GWltl_finalavg_rew} &
        \input{figures/results/GWltl/GWltl_finalavg_sat}
    \end{tabular}
    \captionof{figure}{Average performance on \textbf{Colour bomb gridworld -- LTL} as a function of sample size. Left: \texttt{rew}. Right: \texttt{sat}.}
    \label{fig:gwltl-sample-size}
\end{center}  %

\FloatBarrier

\paragraph{Pacman}

\begin{center}  %
    \begin{longtable}{@{}ccc@{}}
    \multicolumn{2}{c}{\input{figures/results/legend}} & Samples \\

    \input{figures/results/Mini/Mini_SA_10_learning_rew} &
    \input{figures/results/Mini/Mini_SA_10_learning_sat} &
    10 \\

    \input{figures/results/Mini/Mini_SA_100_learning_rew} &
    \input{figures/results/Mini/Mini_SA_100_learning_sat} &
    100\\

    \input{figures/results/Mini/Mini_SA_317_learning_rew} &
    \input{figures/results/Mini/Mini_SA_317_learning_sat} &
    317 \\

    \input{figures/results/Mini/Mini_SA_1000_learning_rew} &
    \input{figures/results/Mini/Mini_SA_1000_learning_sat} &
    1000\\

    \input{figures/results/Mini/Mini_SA_5000_learning_rew} &
    \input{figures/results/Mini/Mini_SA_5000_learning_sat} &
    5000 \\

    \input{figures/results/Mini/Mini_SA_10000_learning_rew} &
    \input{figures/results/Mini/Mini_SA_10000_learning_sat} &
    10000 \\
        \input{figures/results/Mini/Mini_SA_20000_learning_rew} &
        \input{figures/results/Mini/Mini_SA_20000_learning_sat} & 20000 
    \end{longtable}
    \captionof{figure}{Learning curves for the \textbf{Pacman} environment. Left: \texttt{rew}. Right: \texttt{sat}. Rows correspond to different sample sizes.}
    \label{fig:mini-learning-curves-4}
\end{center}  %

\begin{center}  %
    \begin{tabular}{cc}
        \multicolumn{2}{c}{\input{figures/results/legend}}\\
        \input{figures/results/Mini/Mini_finalavg_rew} &
        \input{figures/results/Mini/Mini_finalavg_sat}
    \end{tabular}
    \captionof{figure}{Average performance on \textbf{Pacman} as a function of sample size. Left: \texttt{rew}. Right: \texttt{sat}.}
    \label{fig:mini-sample-size}
\end{center}  %

\FloatBarrier
\paragraph{Pacman Slippery}
\begin{center}  %
    \begin{longtable}{@{}ccc@{}}
    \multicolumn{2}{c}{\input{figures/results/legend}} & Samples \\

    \input{figures/results/MiniSlip/MiniSlip_SA_10_learning_rew} &
    \input{figures/results/MiniSlip/MiniSlip_SA_10_learning_sat} &
    10 \\

    \input{figures/results/MiniSlip/MiniSlip_SA_100_learning_rew} &
    \input{figures/results/MiniSlip/MiniSlip_SA_100_learning_sat} &
    100\\

    \input{figures/results/MiniSlip/MiniSlip_SA_1000_learning_rew} &
    \input{figures/results/MiniSlip/MiniSlip_SA_1000_learning_sat} &
    1000 \\

    \input{figures/results/MiniSlip/MiniSlip_SA_5000_learning_rew} &
    \input{figures/results/MiniSlip/MiniSlip_SA_5000_learning_sat} &
    5000\\

    \input{figures/results/MiniSlip/MiniSlip_SA_10000_learning_rew} &
    \input{figures/results/MiniSlip/MiniSlip_SA_10000_learning_sat} &
    10000 \\

    \input{figures/results/MiniSlip/MiniSlip_SA_16215_learning_rew} &
    \input{figures/results/MiniSlip/MiniSlip_SA_16215_learning_sat} &
    16215\\
        \input{figures/results/MiniSlip/MiniSlip_SA_20000_learning_rew}&
        \input{figures/results/MiniSlip/MiniSlip_SA_20000_learning_sat} & 20000 \\
    \end{longtable}
    \captionof{figure}{Learning curves for the \textbf{Pacman Slippery} environment. Left: \texttt{rew}. Right: \texttt{sat}. Rows correspond to different sample sizes.}
    \label{fig:minislip-learning-curves-4}
\end{center}  %

\begin{center}
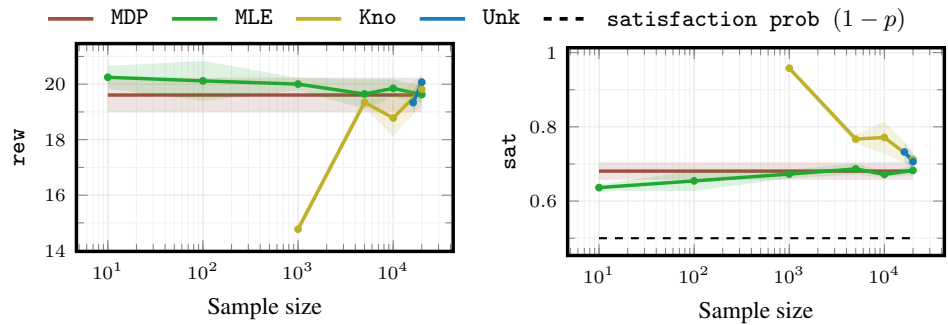
  %
    \begin{tabular}{cc}
        \multicolumn{2}{c}{\input{figures/results/legend}}\\
        \input{figures/results/MiniSlip/MiniSlip_finalavg_rew} &
        \input{figures/results/MiniSlip/MiniSlip_finalavg_sat}
    \end{tabular}
    \captionof{figure}{Average performance on \textbf{Pacman Slippery} as a function of sample size. Left: \texttt{rew}. Right: \texttt{sat}.}
    \label{fig:minislip-sample-size}
\end{center}  %

\FloatBarrier
\paragraph{Media Streaming}

\begin{center}  %
    \begin{longtable}{@{}ccc@{}}
    \multicolumn{2}{c}{\input{figures/results/legend}} & Samples \\

    \input{figures/results/Stream/Stream_SA_10_learning_rew} &
    \input{figures/results/Stream/Stream_SA_10_learning_sat} &
    10 \\

    \input{figures/results/Stream/Stream_SA_100_learning_rew} &
    \input{figures/results/Stream/Stream_SA_100_learning_sat} &
    100\\

    \input{figures/results/Stream/Stream_SA_869_learning_rew} &
    \input{figures/results/Stream/Stream_SA_869_learning_sat} &
    869 \\

    \input{figures/results/Stream/Stream_SA_1000_learning_rew} &
    \input{figures/results/Stream/Stream_SA_1000_learning_sat} &
    1000\\

    \input{figures/results/Stream/Stream_SA_5000_learning_rew} &
    \input{figures/results/Stream/Stream_SA_5000_learning_sat} &
    5000 \\

    \input{figures/results/Stream/Stream_SA_10000_learning_rew} &
    \input{figures/results/Stream/Stream_SA_10000_learning_sat} &
    10000\\
        \input{figures/results/Stream/Stream_SA_20000_learning_rew} &
        \input{figures/results/Stream/Stream_SA_20000_learning_sat} & 20000
    \end{longtable}
    \captionof{figure}{Learning curves for the \textbf{Media Streaming} environment. Left: \texttt{rew}. Right: \texttt{sat}. Rows correspond to different sample sizes.}
    \label{fig:stream-learning-curves-4}
\end{center}  %

\begin{center}  %
    \begin{tabular}{cc}
        \multicolumn{2}{c}{\input{figures/results/legend}}\\
        \input{figures/results/Stream/Stream_finalavg_rew} &
        \input{figures/results/Stream/Stream_finalavg_sat}
    \end{tabular}
    \captionof{figure}{Average performance on \textbf{Media Streaming} as a function of sample size. Left: \texttt{rew}. Right: \texttt{sat}.}
    \label{fig:stream-sample-size}
\end{center}  %

\FloatBarrier
\paragraph{Media Streaming -- Alternative}
\begin{center}
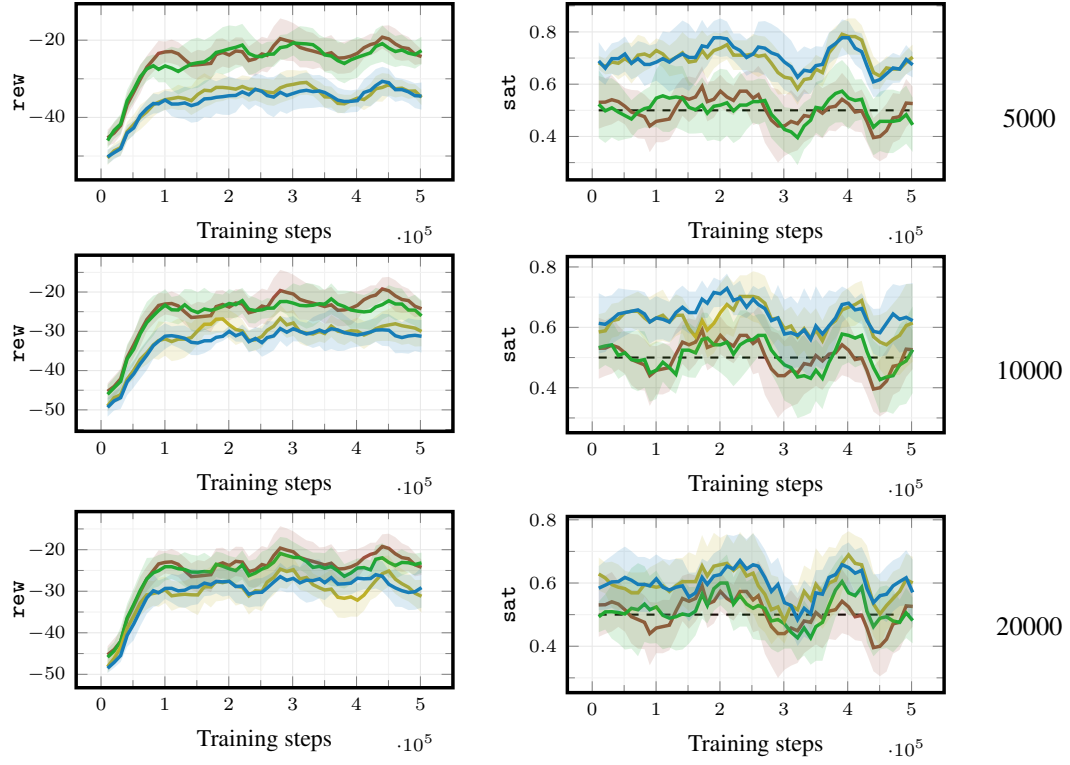
  %
    \begin{longtable}{@{}ccc@{}}
    \multicolumn{2}{c}{\input{figures/results/legend}} & Samples \\

    \input{figures/results/StreamV2/StreamV2_SA_10_learning_rew} &
    \input{figures/results/StreamV2/StreamV2_SA_10_learning_sat} &
    10 \\

    \input{figures/results/StreamV2/StreamV2_SA_100_learning_rew} &
    \input{figures/results/StreamV2/StreamV2_SA_100_learning_sat} &
    100\\

    \input{figures/results/StreamV2/StreamV2_SA_1000_learning_rew} &
    \input{figures/results/StreamV2/StreamV2_SA_1000_learning_sat} &
    1000 \\

    \input{figures/results/StreamV2/StreamV2_SA_5000_learning_rew} &
    \input{figures/results/StreamV2/StreamV2_SA_5000_learning_sat} &
    5000\\

    \input{figures/results/StreamV2/StreamV2_SA_10000_learning_rew} &
    \input{figures/results/StreamV2/StreamV2_SA_10000_learning_sat} &
    10000 \\

    \input{figures/results/StreamV2/StreamV2_SA_20000_learning_rew} &
    \input{figures/results/StreamV2/StreamV2_SA_20000_learning_sat} &
    20000
\end{longtable}
    \captionof{figure}{Learning curves for the \textbf{Media Streaming -- Alternative} environment. Left: \texttt{rew}. Right: \texttt{sat}. Rows correspond to different sample sizes.}
    \label{fig:streamv2-learning-curves-3}
\end{center}  %

\begin{center}
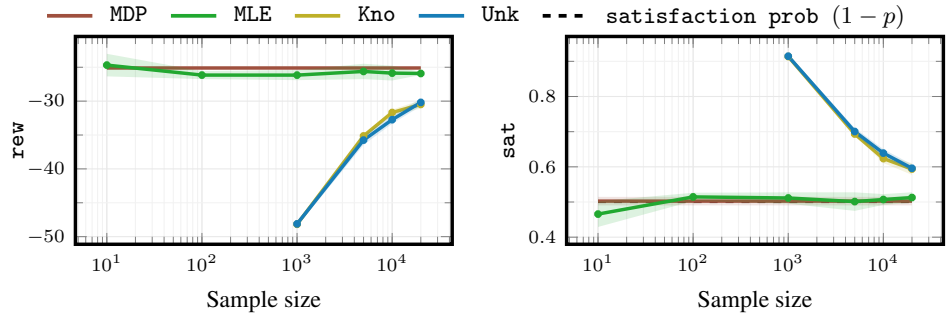
  %
    \begin{tabular}{cc}
        \multicolumn{2}{c}{\input{figures/results/legend}}\\
        \input{figures/results/StreamV2/StreamV2_finalavg_rew} &
        \input{figures/results/StreamV2/StreamV2_finalavg_sat}
    \end{tabular}
    \captionof{figure}{Average performance on \textbf{Media Streaming -- Alternative} as a function of sample size. Left: \texttt{rew}. Right: \texttt{sat}.}
    \label{fig:streamv2-sample-size}
\end{center}  %

%% file: figures/results/legend.tex
\begin{tikzpicture}[baseline=(current bounding box.center)]
\begin{axis}[
    hide axis,
    xmin=0, xmax=1,
    ymin=0, ymax=0,
    legend columns=5,
    legend style={
        font=\small,
        draw=none,
        fill=none,
        column sep=4pt,
        row sep=0pt,
        inner xsep=2pt,
        inner ysep=1pt,
        /tikz/every even column/.append style={column sep=6pt},
        at={(0.5,0.5)},
        anchor=center
    },
    legend cell align=left
]

\addplot[
    color={rgb,255:red,214;green,39;blue,40},
    line width=1.4pt
] coordinates {(0,0)};
\addlegendentry{\texttt{MDP}}

\addplot[
    color={rgb,255:red,44;green,160;blue,44},
    line width=1.4pt
] coordinates {(0,0)};
\addlegendentry{\texttt{MLE}}

\addplot[
    color={rgb,255:red,255;green,127;blue,14},
    line width=1.4pt
] coordinates {(0,0)};
\addlegendentry{\texttt{Kno}}

\addplot[
    color={rgb,255:red,31;green,119;blue,180},
    line width=1.4pt
] coordinates {(0,0)};
\addlegendentry{\texttt{Unk}}

\addplot[
    black, dashed, line width=1.4pt
] coordinates {(0,0)};
\addlegendentry{\texttt{satisfaction prob $(1-p)$}}

\end{axis}
\end{tikzpicture}